\documentclass[10pt,a4paper]{article}
\usepackage[T1]{fontenc}
\usepackage{graphicx}
\usepackage{mathtools}
\usepackage{amssymb}
\usepackage{amsthm}
\usepackage{dsfont}
\usepackage{amsmath}
\usepackage{thmtools}
\usepackage{xcolor}
\usepackage{nameref}
\usepackage{babel}
\usepackage{cleveref}
\usepackage[authoryear,round]{natbib}
\usepackage{dsfont}
\usepackage{ifthen}
\usepackage{pgfplots}
\pgfplotsset{compat=1.18}
\usepackage{subcaption}
\definecolor{rgbblue}{RGB}{68,119,170}
\definecolor{rgbgreen}{RGB}{34,136,51}
\definecolor{rgbviolet}{RGB}{170,51,119}

\newboolean{toggleappendix}
\newcommand{\storeproof}[2]{
  \ifthenelse{\boolean{toggleappendix}}%
    {\expandafter\gdef\csname proof-#1\endcsname{#2}}%
    {\begin{proof}#2\end{proof}}
}

\newcommand{\showproof}[1]{%
  \ifthenelse{\boolean{toggleappendix}}%
    {\begin{proof}\csname proof-#1\endcsname\end{proof}}%
    {}
}

\usepackage[left=2.5cm, right=2.5cm, top=2.5cm, bottom=2.5cm]{geometry}

\newtheorem{theorem}{Theorem}
\numberwithin{theorem}{section}
\newtheorem{lemma}[theorem]{Lemma}

\newtheorem{example}[theorem]{Example}
\theoremstyle{remark}
\newtheorem{remark}[theorem]{Remark}

\newtheorem{assumption}[theorem]{Assumption}

\allowdisplaybreaks

\title{On the minimax optimality of Flow Matching through the connection to kernel density estimation}
\author{Lea Kunkel and Mathias Trabs\footnote{We thank the participants of the Oberwolfach Mini-Workshop 2508b "Statistical Challenges for Deep Generative Models" for inspiring discussions and helpful comments on this project.}}
\date{Ruhr University Bochum and Karlsruhe Institute of Technology}

\begin{document}
	\maketitle
	\begin{center}
	\begin{minipage}{0.8\textwidth}
    Flow Matching has recently gained attention in generative modeling as a simple and flexible alternative to diffusion models. While existing statistical guarantees adapt tools from the analysis of diffusion models, we take a different perspective by connecting Flow Matching to kernel density estimation.
    We first verify that the kernel density estimator matches the optimal rate of convergence in Wasserstein distance up to logarithmic factors, improving existing bounds for the Gaussian kernel. Based on this result, we prove that for sufficiently large networks, Flow Matching achieves the optimal rate up to logarithmic factors. 
    If the target distribution lies on a lower-dimensional single-chart manifold, we show that the kernel density estimator profits from the smaller intrinsic dimension on a small tube around the manifold. The faster rate also applies to Flow Matching, providing a theoretical foundation for its empirical success in high-dimensional settings.

	\end{minipage}
\end{center}
\textbf{Keywords:} Generative models, continuous normalizing flows, rate of convergence, kernel density estimator, Wasserstein distance, distribution estimation \\
\textbf{MSC 2020:} 62E17, 62G07, 68T07

	\section{Introduction}
	The goal of generative models is to learn a map $\psi$ that transforms a random variable $Z$ drawn from a simple and known \textit{latent distribution}  $\mathbb{U}$ into a new variable $\psi(Z)$ whose distribution closely approximates the unknown target distribution $\mathbb{P}^*$. Hence, generative models are naturally linked to the classical nonparametric distribution estimation problem, enhanced by the additional requirement of enabling fast and efficient sampling from the learned distribution. 
    
	Flow Matching, a generative model introduced by \cite{lipman2023}, has recently attracted significant interest due to its considerably more straightforward construction in comparison to diffusion models (building on \cite{sohl2015}), which have been regarded as the state-of-the-art generative method. Instead of learning $\psi$ directly, like GANs \citep{goodfellow}, Flow Matching relies on learning a vector field $v$ that describes the direction of movement to transfer the probability mass from $\mathbb{U}$ to $\mathbb{P}^*$. This leads to the following ODE in time $t \in [0,1]$,
    \begin{equation*}
	\frac{d}{dt} \psi_t(x) = v_t(\psi_t(x)), \quad \psi_0(x) = x.
\end{equation*}
    The solution to this ODE yields a full path of probability densities via the pushforward distributions 
    \begin{equation*}
	p_t = [\psi_t]_{\#}p_0, 
\end{equation*}
    where $p_0$ is the density of $\mathbb{U}$. The vector field $v$ should be chosen such that $p_1$ is close to the density $p^*$ of $\mathbb{P}^*$.
    
   Originally, normalizing flows were used to transform a latent distribution into an approximation of an unknown distribution by an invertible map constructed according to the change-of-variables theorem. This method was defined by \cite{tabak2010} and \cite{tabak2013}, but popularized by \cite{rezende2015} and \cite{dinh2015}. For an overview see \cite{kobyzev2020}. Continuous normalizing flows (CNF), introduced by \cite{chen2018}, use a neural ordinary differential equation (ODE) to construct this mapping via a maximum likelihood method. However, the training requires simulations of the ODE, which results in high computational cost, see \cite{grathwohl2018}. The Flow Matching approach of \cite{lipman2023} circumvents this by using a regression objective to train a CNF, which avoids simulations during training,
   \begin{equation}\label{eq:PSI} 
	\Psi(\tilde{v}) \coloneqq	\mathbb{E}_{\substack{t\sim \mathcal{U}[0,1], \\Y \sim p^*, \\X_t \sim p_t\left(\cdot | Y\right)}}\big[\left|\tilde v_t(X_t)-v_t\left(X_t| Y\right)\right|^2\big],
\end{equation}
where $p_t(\cdot|Y)$ is a time-dependent conditional density path associated to the conditional vector field $v_t(\cdot|Y)$ and both $p_t(\cdot|Y)$ and $v_t(\cdot|Y)$ are known by construction. This regression objective is in certain cases similar to the score matching objective \citep{hyvarinen2005,vincent2011}, which is used to train diffusion models based on stochastic differential equations, introduced by \cite{song2020}.

Flow Matching algorithms have been successfully used in many different applications that benefit from efficient sampling, such as text-to-speech \citep{guo2024} and text-to-image \citep{yang2024,esser2024} settings, the production of novel molecular or protein structures \citep{dunn2024,bose2024} or the construction of surrogate models in high-energy physics \citep{bieringer2024classifier}. It has also been adjusted theoretically to different settings. \cite{atanackovic2024} adapted Flow Matching to the case of interacting particles, \cite{gat2024} explored the discrete setting. \cite{chen2024} generalized the Euclidean setting to the Riemannian setting, allowing for more general geometries. \cite{kerrigan2023} extended Flow Matching to function spaces. However, the statistical properties have only recently been studied by \cite{fukumizu2024} in the Wasserstein $2$ distance. They do not use the exact setting of \cite{lipman2023}, but rather introduce stopping times that depend on the number of samples, enabling the transfer of methods known from the statistical analysis of diffusion models.

In this work, we demonstrate that Flow Matching in the setting of \cite{lipman2023} is closely related to the classical kernel density estimation (KDE). This connection allows us to analyze Flow Matching from a new perspective. First, we show that the motivation of Flow Matching also holds for its empirical counterparts where we replace the unknown target density $p^*$ in \eqref{eq:PSI} with the empirical measure of the observations. For sufficiently large network classes, the resulting generative algorithm coincides exactly with a kernel density estimator, where the kernel is given by the density of the latent distribution. We show convergence rates for the kernel density estimator in Wasserstein $1$ distance in a setting that allows for the Gaussian kernel which is the most popular choice in Flow Matching applications. Our rates coincide with the minimax optimal rates from \cite{niles_weed2022} up to logarithmic factors. We then transfer our analysis to Flow Matching using neural network classes of parameterized vector fields. This particularly shows that even in case of perfect approximation, which corresponds to heavy overparameterization in practice, the Flow Matching estimator does not deteriorate to the empirical measure, but to an estimator that is minimax optimal. Put another way, Flow Matching can be, up to constants, at least as good as a KDE even without the presence of network-specific structures.

Unlike \cite{fukumizu2024}, we do not build on the similarities to diffusion models. Separating the error of the kernel density estimator allows us to analyze empirical risk minimization without the need for Chernoff-type bounds. This avoids the problems pointed out by \cite{yakovlev2025} in the analysis of diffusion models. Overall, the analysis of Flow Matching is more delicate, since Girsanov's theorem does not apply and thus the strategy of \cite{chen2023c} cannot be used. Therefore, our bound depends exponentially on the Lipschitz constant of the vector field. This is one of the reasons for the difficulties in the proof of \cite{fukumizu2024}\footnote{The bounds of the integral in the exponent of the first equation on p.\ 14 (in the arXiv version) are not correct. The separation of the outer integral does not imply the same separation of the integral in the exponent. The correct factor in the integral is $\exp(2 \int_{T_0}^{t_j} L_u d u)$, where $L_u$ is the Lipschitz constant of the approximated vector field, which depends on $T_0$ and is thus not bounded by a universal constant.}. The main theoretical novelty of this work compared to standard approaches in the statistical literature on diffusions and Flow Matching is to use the KDE as a reference model for the empirical risk minimizer.
Therefore, we do not recover the classical trade-off in network size. Since in general, assuming that the Lipschitz constant of the vector field is bounded with a fixed constant requires an assumption on the unknown distribution \citep{kunkel2025distribution}, our approach requires significantly fewer assumptions on the unknown distribution. Our analysis allows us to show convergence rates in Wasserstein $1$ distance in the case where overparameterized neural networks are used for the vector field.

While these rates are minimax optimal up to logarithmic factors, they suffer from the curse of dimensionality. In case $\mathbb{P}^*$ is supported on a lower-dimensional single-chart manifold, we show that, on a small tube around the manifold, the mass of the generated distribution concentrates according to the unknown distribution much faster, depending on the intrinsic dimension and thus circumventing the curse of dimensionality. This result again applies to both the KDE and the Flow Matching estimator. For the former, our results are the first that do not require a structural adaptation of the KDE to the manifold itself. Therefore, they contribute to the statistical understanding of the KDE under the manifold hypothesis previously studied by \cite{divol2022, Berenfeld2021}.

\paragraph{Related work}
    \cite{lipman2023} use a fixed latent distribution. Similar approaches for flows between two possibly unknown distributions $\mathbb{P}$ and $\mathbb{Q}$ are studied by \cite{tong2024}, \cite{liu2022} and \cite{albergo2022}. \cite{tong2024} also generalize the mentioned methods.
	\cite{gao2024} prove a suboptimal rate of convergence for certain forms of Flow Matching not exactly conforming to the construction of \cite{lipman2023} in the Wasserstein $2$ distance. \cite{benton2024} analyzed Flow Matching excluding the approximation error by imposing assumptions on the covariance that lead to global Lipschitz bounds on the vector field. They also focus on different constructions than \cite{lipman2023}. \cite{gong2025} study the properties of ReLU networks to approximate a vector field corresponding to higher order trajectories.

As \cite{lipman2023} point out, Flow Matching is closely related to diffusion models. For an overview of generative diffusion models, see \cite{cao2024}. Indeed, even in cases that are not constructed to be consistent with diffusion models, the approximation of a score function has similar properties to the approximation of the Flow Matching vector field. \cite{fukumizu2024} build on this connection.
The statistical properties of score matching are an area of ongoing research, see for example \cite{chen2023a,chen2023b,chen2023c,oko2023,tang24,azangulov2024,zhang2024,yakovlev2025}. 
\cite{marzouk2023} study the statistical properties of CNFs trained by likelihood maximization.

The kernel density or Parzen-Rosenblatt estimator \citep{Rosenblatt1956,Parzen1962} is the classical method for estimating a smooth density. For an overview in the univariate case, see \cite{tsybakov2009}, Chapter 1.2 or \cite{Devroye2001} and in the multivariate case see \cite{scott1992}. Typically, the kernel density estimator is analyzed using the mean squared error or the $L_1$ distance.  Despite its longstanding use, kernel density estimation is still the subject of ongoing research, for example in density estimation on unknown manifolds \citep{berry2017,Berenfeld2021,divol2022,wu2022}. In particular, \cite{divol2022} uses the Wasserstein metric to evaluate the performance of a kernel density estimator, however he needs to assume properties of the kernel that are not satisfied by standard choices such as the Gaussian kernel.

    \paragraph{Outline} 
In \Cref{sec:empirical_flow_matching} we introduce the problem setting mathematically and recall all the definitions needed to define the conditional Flow Matching objective from \cite{lipman2023}. We introduce the empirical counterparts, validate the empirical conditional Flow Matching objective, and portray the connection to the kernel density estimator. In \Cref{sec:rate_of_convergence} we first separate the error corresponding to the kernel density estimator and the vector field approximation. We then proceed to bound the former error in \Cref{sec:kde_rates}. Incorporating theory of network approximation, we bound the latter error and obtain a rate of convergence for the entire Flow Matching procedure in \Cref{sec:vector_field_rate}. Subsequently, we show in \Cref{manifold} that the rate can be improved if the unknown distribution is supported on a lower-dimensional manifold. Finally, we illustrate our results in \Cref{sec:numerical_illustration}. All proofs are deferred to \Cref{sec:proofs}.

\subsection{Some notations}
	
We write  $X \sim p$ as a synonym for $X \sim \mathbb{P}$ in case $\mathbb{P}$ has a density $p$ with respect to the Lebesgue measure and $\mathbb{P}$ is the distribution of the random variable $X$. 
By $a \lesssim b$ we denote $a \leq c \cdot b ,$ where $a,b = \mathbb{R}$ and $c$ is a constant independent of $n$. We use $ a_n = \mathcal{O}(n^{\ell})$ if $a_n$ grows not faster than $n^{\ell}$, i.e.\ $a_n \lesssim n^{\ell}$ and $ a_n \in \Theta(n^{\ell})$ if $a_n$ grows as fast as $n^{\ell}$.
For $d \in \mathbb{N}$, we use the Euclidean norm $| \cdot |$ on $\mathbb{R}^d$ and use $\|\cdot\|$ for associated operator norms. For a set $A \subset \mathbb{R}^d$, we denote the tubular neighborhood of the set $A$ with radius $r>0$ with $B(A, r)$.
The supremum norm and the $L_1$-norm of a function $f \colon \Omega \rightarrow \mathbb{R}^{d^{\prime \prime}} $ with $d^{\prime}, d^{\prime\prime} \in \mathbb{N}$ and $\Omega \subset \mathbb{R}^{d^{\prime}}$ are defined as 
\begin{equation*}
    \| f\|_{\infty, \Omega} \coloneqq \sup_{x \in \Omega} |f(x)|\quad\text{and}\quad
\| f\|_{1, \Omega} \coloneqq  \int_{\Omega} |f(x)| \; \mathrm{d}x.
\end{equation*}
In case of $\Omega = \mathbb{R}^{d^{\prime}},$ we write $\| f\|_{\infty, \mathbb{R}^{d^{\prime}}} = \| f\|_{\infty}$ and $\| f\|_{1, \mathbb{R}^{d^{\prime}}} = \| f\|_{1}.$
   
For $d^{\prime \prime} =1, \alpha \in (0,1]$ and $\Lambda >0$, the Besov ball $B_{1, \infty}^{\alpha}(\Lambda, \Omega)$ is defined as
\begin{equation}\label{def:Besovball}
	B_{1, \infty}^{\alpha}(\Lambda, \Omega) \coloneqq \big\{ f \in L_1(\Omega) \colon  |f|_{	B_{1, \infty}^{\alpha}(\Omega)}\coloneqq \sup_{t>0} t^{-\alpha} 	\omega_1(f, t)_1 \leq \Lambda\big\}
\end{equation}
where
\begin{equation*}
	\omega_1(f, t)_1:=\sup _{0<|h| \leq t}\int |f(x)-f(x+h)|\; \mathrm{d}x, \quad t>0, x \in \Omega.
\end{equation*} The difference is set to zero if $x +h \notin \Omega.$ Replacing $|f|_{B_{1, \infty}^\alpha(\Omega)}\leq \Lambda$ with $|f|_{B_{1, \infty}^\alpha(\Omega)}<\infty$ in \eqref{def:Besovball}, we obtain the Besov space $B_{1, \infty}^\alpha(\Omega)$. On $B_{1, \infty}^{\alpha}(\Omega)$ we define the norm $\|f\|_{	B_{1, \infty}^{\alpha}(\Omega)} = \|f\|_{L_1(\Omega)} + |f|_{	B_{1, \infty}^{\alpha}(\Omega)}$.
Compared to more commonly used smoothness classes, such as Hölder classes, the Besov class defined above only imposes integrated and not pointwise smoothness. This allows for local irregularities like jumps.

Finally, let $C^{k}(\Omega)$ denote the set of functions $f$ such that the coordinate functions are each $k$-times continuously differentiable and the supremum norms of the derivatives are finite.
For $k \in \mathbb{N}_0^d$ and $|k| = \sum_{i = 1}^{d}k_i$ set
$D^k=\frac{\partial^{|k|}}{\partial x_1^{k_1} \ldots \partial x_d^{k_d}},
$ where $\frac{\partial}{\partial x_i}$ denotes the weak partial derivative. 
 
\section{Empirical Flow Matching}\label{sec:empirical_flow_matching}

Assume we observe an i.i.d. sample $X^*_1,\dots,X^*_n$ from an unknown distribution $\mathbb{P}^*$ on a compact set $\mathcal{X} \subset [-1,1]^d.$ Further assume that $\mathbb{P}^*$ has a density $p^*$ with respect to the Lebesgue measure. The goal of generative modeling is to mimic the distribution $\mathbb{P}^*.$

For a time-dependent vector field $v \colon [0,1] \times \mathbb{R}^d \rightarrow \mathbb{R}^d$ we consider the flow $\psi \colon [0,1] \times \mathbb{R}^d \rightarrow \mathbb{R}^d$ given as the solution to the ODE
\begin{equation}\label{ODE}
	\frac{d}{dt} \psi_t(x) = v_t(\psi_t(x)), \quad \psi_0(x) = x.
\end{equation}
For a fixed latent distribution with Lebesgue density $p_0$, the vector field $v$ generates a probability density path $p \colon [0,1] \times \mathbb{R}^d \rightarrow \mathbb{R}_{\geq0}$ with $\int p_t(x) \mathrm{d}x= 1$ for all $t$ via the pushforward distributions \begin{equation*}
	p_t = [\psi_t]_{\#}p_0,  \quad \text{i.e. } \psi_t(Z) \sim p_t \text{ for } Z\sim p_0.
\end{equation*}
The ODE \eqref{ODE} corresponds to the Lagrangian description (in terms of particle trajectories) of the conservation of mass formula \cite[p.14]{Villani2008}. The change of variables formula links it to the Eulerian description: A necessary and sufficient condition for $v_t$ to generate $p_t$ is 
	\begin{equation}\label{nec_and_suf_cond}
		\frac{d}{d t} p_t+\operatorname{div}(p_t v_t)=0,
	\end{equation}
see \cite[p. 14]{Villani2008}
.\\

Approximating a given vector field $v$ that generates a certain density path $p_t$ using a function $\tilde{v}$ leads to the Flow Matching objective
\begin{equation}\label{eq:fmo}
	\mathbb{E}_{\substack{t \sim \mathcal{U}[0,1]\\X_t \sim p_t}} \big[|v_t(X_t) - \tilde{v}_t(X_t) |^2\big]
\end{equation}
by \cite{lipman2023}.
They consider a probability path of the form
	\begin{equation}\label{marginal_prob_paths}
		p_t(x)=\int p_t(x | y) p^*(y) \; \mathrm{d}y,
	\end{equation}
where $p_t(\cdot | y)\colon \mathbb{R}^d \rightarrow \mathbb{R}$  is a \emph{conditional probability path} generated by some vector field $v_t(\cdot|y)\colon \mathbb{R}^d \rightarrow \mathbb{R}^d$ for $y \in \mathbb{R}^d$. The vector field generating \eqref{marginal_prob_paths} is then given by
\begin{equation}\label{marginal_vector_field}
	v_t(x)=\int v_t(x | y) \frac{p_t(x | y) p^*(y)}{p_t(x)} d y.
\end{equation}
In this setting \cite{lipman2023} show that the Flow Matching objective \eqref{eq:fmo} is equivalent to the conditional Flow Matching objective 
\begin{equation} \label{flow_matching_objective}
	\Psi(\tilde{v}) \coloneqq	\mathbb{E}_{\substack{t\sim \mathcal{U}[0,1], \\Y \sim p^*, \\X_t \sim p_t\left(\cdot | Y\right)}}\big[\left|\tilde{v}_t(X_t)-v_t\left(X_t| Y\right)\right|^2\big].
\end{equation}

As $p^*$ is unknown, in practice, the expectation in $Y\sim p^*$ is replaced by the empirical counterpart based on the observations $X_1^*,\dots,X_n^*$. This leads to the empirical counterparts of \eqref{marginal_prob_paths} and \eqref{marginal_vector_field} given by  
\begin{align}
	p^n_t(x) &= \frac{1}{n} \sum_{i = 1}^{n} p_t(x|X_i^*)\qquad\text{and}\qquad 
	v^n_t(x) = \sum_{i = 1}^{n} v_t(x|X_i^*)\frac{p_t(x|X_i^*)}{ \sum_{j = 1}^{n} p_t(x|X_j^*)	}. \label{eq:def_vt}
\end{align}
With this modification we recover the sufficient condition for $v^n_t$ to generate $p^n_t$ analogously to \cite[Theorem 1]{lipman2023}. 
	
\begin{lemma} \label{thm:vectorfield_generates_ppath}
	If  $v_t(\cdot |X_i^*)$ generates $p_t(\cdot |X_i^*)$ for all $i=1,\dots,n,$ then $v^n_t$ generates $p^n_t.$
\end{lemma}

The motivation for the conditional Flow Matching objective \eqref{flow_matching_objective} is the equivalence to the unconditioned Flow Matching objective with respect to the optimizing arguments \cite[Theorem 2]{lipman2023}. Using the empirical counterparts, this still holds true.
    
\begin{theorem} \label{thm:cond_equiv} Using $p^n_t$ and $v^n_t$ from \eqref{eq:def_vt}, and a class of parameterized vector fields $\mathcal{N},$ which is constructed such that the minimizers exist, we have that
	\begin{equation*}
		\underset{\hat{v} \in \mathcal{N}}{\mathrm{argmin}} \int_{0}^{1} \mathbb{E}_{X_t \sim p^n_t}\big[|\hat{v}_t(X_t)- v^n_t(X_t)|^2\big]\; \mathrm{d}t = 	\underset{\hat{v} \in \mathcal{N}}{\mathrm{argmin}} \int_{0}^{1} \frac{1}{n} \sum_{i = 1}^{n} \mathbb{E}_{\tilde{X}_t \sim p_t(\cdot|X_i^*)}\big[|\hat{v}_t(\tilde{X}_t)- v_t(\tilde{X}_t|X_i^*)|^2\big]\; \mathrm{d}t.
	\end{equation*}	
\end{theorem}

Hence the \textit{empirical conditional Flow Matching} objective
\begin{equation}\label{emp_cond_fm_obj}
	\tilde{\Psi}(\tilde{v}) \coloneqq		 \int_{0}^{1} \frac{1}{n} \sum_{i = 1}^{n} \mathbb{E}_{\tilde{X}_t \sim p_t(\cdot|X_i^*)}\big[|\tilde{v}_t(\tilde{X}_t)- v_t(\tilde{X}_t|X_i^*)|^2\big]\; \mathrm{d}t
\end{equation}
is justified theoretically. Compared to \cite{gao2024}, we use  $v_t(\cdot | X_i^*)$ directly as proposed by \cite{lipman2023} and stick to the entire time interval. Note that for simplicity we do not sample $t.$

From minimizing \eqref{emp_cond_fm_obj} in $\tilde{v}\in\mathcal{N},$ we obtain an optimal argument $\hat{v}.$ Solving the ODE \eqref{ODE} using  $\hat{v}$, we obtain a flow $\hat{\psi}_t$, i.e. $\hat\psi$ is given by
\begin{equation}\label{ODEhat}
	\frac{d}{dt} \hat \psi_t(x) = \hat v_t(\hat\psi_t(x)), \quad \hat\psi_0(x) = x,\qquad\text{for}\quad \hat v\in\operatorname*{argmin}_{\tilde v\in\mathcal N}\tilde\Psi(\tilde v).
\end{equation}
We use this flow to push forward the known, latent distribution to time $t = 1$. In accordance with the goal of generative modeling, the distribution of this pushforward should mimic $\mathbb{P}^*.$\\

In order to apply Flow Matching, we have to construct a class of conditional probability paths. Let $Z \sim \mathbb{U}$ and let $K$ denote the density of $\mathbb{U}$. Consider the functions  $\sigma \colon [0,1] \times \mathbb{R}^d \rightarrow \mathbb{R}_{>0}$ and $ \mu \colon [0,1] \times \mathbb{R}^d \rightarrow  \mathbb{R}^d.$ Set for $t \in [0,1]$ and a given $X_i^*, i \in \{1,...,n \}$,
 \begin{equation}\label{def_flow_map}
 	\psi_t(Z | X_i^*) \coloneqq \sigma_t(X_i^*)Z+\mu_t(X_i^*). 
 \end{equation}
The density of $\psi_t(Z | X_i^*)$ is by the transformation formula
	\begin{equation*}
		p_{t}(x|X_i^*) = \frac{1}{\sigma^d_t(X_i^*)}K\Big(\frac{x-\mu_t(X_i^*)}{\sigma_t(X_i^*)}\Big).
	\end{equation*} 
	We call $p_t(\cdot |X_i^*)$ the \textit{conditional kernel probability path}. 
	Setting
	\begin{equation}\label{flow_map_kernel_flow}
		\frac{d}{d t}\psi_t(x | X_i^*) = v_t(\psi_t(x | X_i^*) |  X_i^*),
	\end{equation}
	we  recover the same result as in \cite[Theorem 3]{lipman2023}:
	\begin{lemma}\label{thm:cond_vt_depends_on_mu_and_sigma}
		Let $p_t\left(x | X_i^*\right)$ be a conditional kernel probability path, and $\psi_t(\cdot| X_i^*)$ its corresponding flow as in \eqref{def_flow_map}. Then, the unique vector field that defines $\psi_t(\cdot | X_i^*)$ via \eqref{flow_map_kernel_flow} has the form:
		\begin{equation*}
		v_t\left(x  | X_i^*\right)=\frac{\frac{\partial\sigma_t }{\partial t}\left(X_i^*\right)}{\sigma_t\left(X_i^*\right)}\big(x-\mu_t\left(X_i^*\right)\big)+\frac{\partial\mu_t }{\partial t}\left(X_i^*\right) .
		\end{equation*}

	\end{lemma}

Set $\sigma_{\min}> 0.$ To flow from $p^n_0 (x) = \frac{1}{n} \sum_{i = 1}^{n} K(x) = K(x)$ to $p^n_1(x) = \frac{1}{n \sigma^d_{\min}} \sum_{i = 1}^{n} K(\frac{x-X_i^*}{\sigma_{\min}}),$ we can choose any differentiable functions $\sigma_t$ and $\mu_t$ such that \begin{equation*}
    \mu_0 = 0, \quad \mu_1(y) = y\quad\text{and} \quad \sigma_0 = 1, \quad \sigma_1 = \sigma_{\min}.
\end{equation*}

At time $t=1$ the distribution $\mathbb P^{\psi_1^n(Z)}$ then coincides with the kernel density estimator
\begin{equation}\label{eq:KDE}
	p^n_1(x) = \frac{1}{n \sigma^d_{\min}} \sum_{i = 1}^{n} K\Big(\frac{x-X_i^*}{\sigma_{\min}}\Big),
\end{equation}
where the kernel is given by the latent distribution $\mathbb U$. Choosing $\mathbb{U} = \mathcal{N}_d(0,I_d)$, i.e., we consider the $d$-dimensional Gaussian kernel $K(x) = (2\pi)^{-d/2}\exp(-|x|^2/2)$, yields the proposed flow from \citet[Section 4]{lipman2023}. Moreover, considering general kernels is in line with methods by \cite{tong2024}, \cite{liu2022} and \cite{albergo2022}, which transform an unknown distribution to another.

\section{Rate of convergence}\label{sec:rate_of_convergence}

The aim of this work is to evaluate how well Flow Matching performs depending on the number of observations $n.$ To this end, we have to control the distance between $\mathbb P^*$ and $\mathbb{P}^{\hat{\psi}_1(Z)}$ with flow $\hat\psi$ from \eqref{ODEhat}. As an evaluation metric, we use the Wasserstein $1$ metric.

On the normed vector space $(\mathbb{R}^d, |\cdot|)$ the Wasserstein $1$ distance between two probability distributions $\mathbb{P}$ and $\mathbb{Q}$ on $\mathbb{R}^d$ is defined as
\begin{equation*} 
	\begin{aligned}
		\mathsf{W}_1(\mathbb{P}, \mathbb{Q}) &  \coloneqq \Big(\inf _{\pi \in \Pi(\mathbb{P}, \mathbb{Q})} \int_{\mathbb{R}^d \times \mathbb{R}^d} |x-y| \; \mathrm{d}  \pi(x, y)\Big) ,
	\end{aligned}
\end{equation*}
where $\Pi$ is the set of all distributions whose marginal distributions correspond to $\mathbb{P}$ and $\mathbb{Q}$ respectively. By \cite[Theorem 5.10(i)]{Villani2008}, the following duality holds: 
\begin{equation}\label{wasserstein_1_def}
	\mathsf{W}_1(\mathbb{P}, \mathbb{Q})= \sup_{W \in \operatorname{Lip}(1)}\mathbb{E}_{\substack{X \sim \mathbb{P}\\ Y \sim \mathbb{Q}}}[W(X)-W(Y)].
\end{equation}
A comprehensive overview of the advantages of this distance over other measures that metrize weak convergence can be found in \cite[p. 98 f.]{Villani2008}. Due to its compatibility with dimension reduction settings such as the manifold hypothesis, the Wasserstein $1$ distance is a popular choice for the evaluation of generative models, see e.g.\ \cite{Schreuder2020,Liang2018,stephanovitch2023,kunkel2025}.\\

Another popular measure of dissimilarity is the total variation distance, 
\begin{equation*}
    \operatorname{TV}(\mathbb{P}, \mathbb{Q}) \coloneqq \sup_{A \in \mathcal{A}}|\mathbb{P}(A)- \mathbb{Q}(A)|,
\end{equation*}
where $\mathcal{A}$ is the Borel-$\sigma$-algebra. This metric has frequently been used to evaluate diffusion models \citep{oko2023, yakovlev2025}. \Cref{ex:tv} illustrates why the total variation distance is not suitable to evaluate Flow Matching models.
\begin{example}\label{ex:tv}
    Let $d = 1, Z \sim \mathcal{N}(0, 1), \varepsilon>0$ and define the two vector fields $v^{(1)}, v^{(2)} \colon   [0,1] \times \mathbb{R} \rightarrow \mathbb{R}$
    \begin{equation*}
        v^{(1)}_t(x)\coloneqq 0, \quad v^{(2)}_t(x)\coloneqq \varepsilon \sin\Big(\frac{x}{\varepsilon}\Big), \quad \text{for all } t \in [0,1].
    \end{equation*}
    Let $\psi^{(1)}, \psi^{(2)}$ be the corresponding solutions to the ODE \eqref{ODE}. Then
    \begin{equation*}
      \|v^{(1)}-v^{(2)}\|_{\infty} \overset{\varepsilon\rightarrow 0}{\rightarrow}0 , \quad \mathsf{W}_1(\mathbb{P}^{\psi_1^{(1)}(Z)}, \mathbb{P}^{\psi_1^{(2)}(Z)}) \leq \varepsilon \overset{\varepsilon\rightarrow 0}{\rightarrow}0 \quad \text{and} \quad \underset{\varepsilon\rightarrow 0}{\liminf}\operatorname{TV}(\mathbb{P}^{\psi_1^{(1)}(Z)}, \mathbb{P}^{\psi_1^{(2)}(Z)})\geq c,
    \end{equation*}
    for a constant $c >0$ .
\end{example}

The proof of \Cref{ex:tv} shows that for convergence in total variation distance, one needs control over the gradient of $\psi$ with regard to $x$. Considering the change-of-variables theorem and the fact that the total variation distance operates on a density level, this is unsurprising. However, the gradient is not controlled by the Flow Matching objective. Assuming gradient control over the vector field in space can lead to different results, see \cite{su2025flow}.

The following theorem provides a first comparison of the performance of $\mathbb{P}^{\hat{\psi}_1(Z)}$ and the empirical flow $\mathbb{P}^{\psi^n_1(Z)}$ in the Wasserstein $1$ distance. 

\begin{theorem}\label{thm:error_decomp} 
    Assume that all functions in $\mathcal{N}$ are Lipschitz continuous for fixed $t$ with Lipschitz constant $\Gamma_t.$  Then for any $\tilde{v}\colon [0,1] \times\mathbb{R}^d \rightarrow \mathbb{R}^d,$ with $\tilde{v} \in \mathcal{N}$ and $v^n$ from \eqref{eq:def_vt}	
		\begin{equation}
			\mathsf{W}_1(\mathbb{P}^*, \mathbb{P}^{\hat{\psi}_1(Z)}) \leq \mathsf{W}_1(\mathbb{P}^*, \mathbb{P}^{\psi^n_1(Z)}) + \sqrt{2 e} e^{\int_{0}^1 \Gamma_t\; \mathrm{d}t} \|v^n - \tilde{v} \|_{\infty}. \label{basic_orakel}
		\end{equation}
	\end{theorem}

 Using the construction from \eqref{def_flow_map} to \eqref{eq:KDE}, $\mathsf{W}_1(\mathbb{P}^*, \mathbb{P}^{\psi^n_1(Z)})$ is the estimation error of a kernel density estimator. The second term $\|v^n - \tilde{v} \|_{\infty}$ depends on the set $\mathcal{N}$ and its ability to approximate $v^n$. For a sufficiently rich class $\mathcal M$ this second term will be negligible in our analysis.  A corresponding error decomposition can be shown for the Wasserstein $2$ distance, too.

The proof of \Cref{thm:error_decomp} uses Grönwall's Lemma, which leads to the factor $\exp(\int_0^1 \Gamma_t \; \mathrm{d}t)$ in the second term in \eqref{basic_orakel}. This is standard in the analysis of ODEs, in the context of flow-based generative models see \cite{fukumizu2024} and \cite{albergo2022}. Compared to the SDE setting of diffusion models, there is no analogous result to Girsanov's theorem, which is typically used to study these models, following \cite{chen2023c}, and circumvents the exponential dependence on the Lipschitz constant. This greatly complicates the analysis of Flow Matching. 
Given this dependence, a Lipschitz regularization in the Flow Matching objective seems to be theoretically beneficial. Imposing the assumption that the vector field defined in \eqref{marginal_vector_field} is Lipschitz continuous implies an extensive restriction on the class of feasible $\mathbb{P}^*$, see \cite{kunkel2025distribution}.
\begin{remark}
    A bound for arbitrary Lipschitz $\tilde{v}_t$ cannot be better than $\mathsf{W}_1(\mathbb{P}^*, \mathbb{P}^{\psi^n_1(Z)})$. If $v^n_t$ is contained in $\mathcal{N},$ then the Picard-Lindelöf theorem yields $\hat{\psi}_t = \psi^n_t$ and hence $\mathsf{W}_1(\mathbb{P}^*, \mathbb{P}^{\hat{\psi}_1(Z)}) = \mathsf{W}_1(\mathbb{P}^*, \mathbb{P}^{\psi^n_1 (Z)}).$
\end{remark}

Due to \Cref{thm:error_decomp}, we can bound both  error terms individually. We start with the error corresponding to the kernel density estimator in \Cref{sec:kde_rates}. Subsequently, we bound the second error term for approximation classes $\mathcal N$ consisting of feedforward neural networks in \Cref{sec:vector_field_rate} to obtain our final result.

\subsection{Rate of convergence of the KDE in Wasserstein distance}\label{sec:kde_rates}

To bound $\mathsf{W}_1(\mathbb{P}^*, \mathbb{P}^{\psi_1^n(Z)})$ in \Cref{thm:error_decomp}, we have to analyze the corresponding kernel density estimator in the Wasserstein $1$ distance. For compact kernels, existing results for the $L_1$-error of the KDE can be exploited. However,
this can only lead to the suboptimal convergence rate $n^{-\frac{\alpha}{2\alpha +d}}$. Indeed, the Lipschitz regularity of the test function in \eqref{wasserstein_1_def} can be exploited to obtain a faster rate. Additionally, to allow for latent distributions with unbounded support, especially the Gaussian distribution used by \citet{lipman2023}, we need to investigate kernels with unbounded support. 
	
\begin{theorem}\label{thm:kde_rate_smooth}
Assume $d \geq 2, n \geq 2,  p^* \in B_{1, \infty}^\alpha\left(\Lambda, \mathbb{R}^d\right), \alpha \in(0,1], \Lambda >0$ and $\operatorname{supp}\left(p^*\right)$ bounded. Assume $K$, used in the Flow Matching setting via \eqref{eq:KDE}, is a nonnegative $d$-dimensional kernel such that 
	\begin{align*}
		\int  yK(y) \; \mathrm{d}y = 0 , \quad \text{and} \quad  \int  |y|^{1+\alpha }   K(y)  \;\mathrm{d}y < \infty. 
	\end{align*}
	Further assume $K \in C^{\lceil \frac{d+2}{2}\rceil }$ with $\|D^{k}K\|_1 \leq C$  for $k \in \mathbb{N}_0^d$ such that $|k|\leq \frac{d+2}{2}$ and some $C>0$ and $\sigma_{\min} \in (0,1)$. Then there are $C_1$ and $C_2$ such that for $v^n$ from \eqref{eq:def_vt},
	\begin{equation*}
		\mathbb{E}\big[	\mathsf{W}_1(\mathbb{P}^*, \mathbb{P}^{\psi^n_1(Z)})  \big] \leq C_1\sigma_{\min}^{1+\alpha} + \frac{C_2}{\sqrt{\sigma_{\min}^{d-2}}} \frac{\log n}{\sqrt n}.
	\end{equation*}
	For $\sigma_{\min} \sim (n/\log^2n)^{-\frac{1}{2\alpha +d}}$ we obtain $\mathbb{E}\big[	\mathsf{W}_1(\mathbb{P}^*, \mathbb{P}^{\psi^n_1(Z)})  \big]= \mathcal O((n/\log^2n)^{-\frac{1+\alpha}{2\alpha +d}}).$
\end{theorem} 
The proof separates the bias induced by smoothing and the stochastic error after smoothing. Using the dual form \eqref{wasserstein_1_def} and the relationship between differentiation and convolution, the bound of the bias exploits both the $\alpha$-smoothness of $\mathbb{P}^*$ and the Lipschitz-smoothness of test functions in \eqref{wasserstein_1_def}. The stochastic error is bounded by considering $(K_{\sigma_{\min} } \ast f)$ as the smoothed test-function and then taking advantage of a result by \cite{Schreuder2020_2}. Noting that the proof relies on the dual form of the Wasserstein $1$ distance, a bound for the Wasserstein $2$ distance would require a completely different approach. A corresponding rate for $d = 1$ can be derived here and in all following theorems with analogous proof strategies, we will omit this case for the sake of brevity.

The convergence rate $(n/\log^2n)^{-\frac{1+\alpha}{2\alpha +d}}$ coincides up to the logarithmic factor with the lower bound by  \citet[Theorem~3]{niles_weed2022} and thus the above rate is minimax optimal up to the logarithm. \cite{divol2022} also considers rates of convergence for kernel density estimators. In case of $d \neq 2,$ he obtains the optimal rate for distributions bounded away from zero. However, the kernels he considers must be smooth radial functions with support bounded in a unit ball \cite[Condition A]{divol2022}. Hence, his result does not apply to the Gaussian kernel. 

\begin{remark} \label{rem:gaussian_kernel_is_smooth}  $ $
\begin{enumerate}
    \item The assumption $\|D^{k}K\|_1 \leq C$  for $k \in \mathbb{N}_0^d$ and $C>0$ is satisfied by the Gaussian kernel. 

 \item For kernels that do not satisfy the differentiability condition in \Cref{thm:kde_rate_smooth}, we show in \Cref{thm:kde_rate_nonsmooth} a rate of convergence of $n^{-\frac{1+\alpha}{2\alpha +2+d}}.$ For nonzero $\alpha$ and $d$ big enough, this rate decays faster than $n^{-\frac{1}{d}},$ the rate of convergence of the empirical measure \citep{Dudley1969,Boissard2014}. This result replaces the assumption that $\operatorname{supp}(p^*)$ is bounded by a moment assumption on $p^*.$
    \item The second term in \Cref{thm:kde_rate_smooth} is an upper bound for the expected value of $\mathsf{W}_1(K_{\sigma_{\min}} \ast \mathbb{P}_n,K_{\sigma_{\min}} \ast \mathbb{P}^* ). $ Results on the convergence of the smooth empirical measure have been discussed in case of the Gaussian kernel. \citet[Proposition 1]{goldfeld2020} obtain a rate in $\mathcal{O}\big( \frac{1}{\sqrt{n\sigma_{\min}^d}}\big)$ for subgaussian distributions $\mathbb{P}^*$, which coincides with the rate obtained in \Cref{thm:kde_rate_nonsmooth}. \citet{weed2018} looks at distributions on $[-1,1]$ and obtains a rate in $\mathcal{O}\big( \frac{1}{\sqrt{n}\sigma_{\min}^d}\big).$ Hence the results above improve these results for distributions in $B^{\alpha}_{1, \infty}$ on bounded support.
\end{enumerate}

\end{remark}

\subsection{Flow Matching with parameterized vector field}\label{sec:vector_field_rate}

If the empirical vector field $v^n$ is contained in $\mathcal{N}$, then the second error term in \Cref{thm:error_decomp} vanishes and the Flow Matching procedure inherits the optimal rate of convergence from \Cref{thm:kde_rate_smooth}. In practice, \eqref{emp_cond_fm_obj} is optimized over a class of neural networks to find a good approximation of $v^n$. We thus have to take this approximation error into account. 

To analyze the ability of a network to approximate $v^n,$ it is necessary to know the properties of $v^n$ for all $t \in [0,1].$ Hence we are going to specify $\sigma$, $\mu$ and a latent distribution for the subsequent analysis. 

\begin{assumption}\label{ass:sigma+mu}$ $
\begin{enumerate} 
    \item We consider the following choices of  $\sigma \colon [0,1] \times \mathbb{R}^d \rightarrow \mathbb{R}_{>0}$ and $ \mu \colon [0,1] \times \mathbb{R}^d \rightarrow  \mathbb{R}^d:$
    \begin{equation*} \label{sigma_mu_setting}
		\sigma_t(X_i^*)= 1-(1-\sigma_{\min})t \quad \text{and} \quad \mu_t(X_i^*) = tX_i^*.
	\end{equation*} 
    \item We choose $\mathbb{U} = \mathcal{N}(0,I_d)$ for the latent distribution.
\end{enumerate}
\end{assumption}

  The choice of $\sigma_t$ and $\mu_t$ in \Cref{ass:sigma+mu} coincides with the setting of \cite{lipman2023} in Example II. This is not a unique choice, for other choices used in the literature in similar settings, see \cite[Table 1]{tong2024}.  To avoid losing information, if $\mathbb{U} = \mathcal{N}(0,I_d)$, then the choice of $\mu_t$ should satisfy $|\mu_t(X_i^*) | \leq |X_i^*|$. 
      $\sigma_t$ should map to $[0,1]$ to avoid unnecessarily increasing the Lipschitz constant $\Gamma_t$ of the vector field.  Generalizing the linear choice of $\sigma_t$ to \begin{equation*} \sigma_t(X_i^*)= \big(1-(1-\sigma_{\min}^{\frac{1}{k}})t\big)^k \end{equation*} for a $k>0$ has only a negligible effect. The curvature cancels out in the construction of $v_t(\cdot|X_i^*),$ see \Cref{thm:cond_vt_depends_on_mu_and_sigma}.

In order to obtain quantitative results for the neural network approximation, we further need to specify the neural networks. We use fully connected feedforward ReLU networks, which enjoy nice approximation properties, see for example \cite{Yarotsky2017,Guehring2020,kohler2021,Schmidt_Hieber2020,suzuki2018}.

Let us fix some notation:
define the Rectified Linear Unit (ReLU) activation function, $\phi\colon \mathbb{R} \rightarrow \mathbb{R}, \phi(x) = \max(0,x)$. For a vector $v=\left(v_1, \ldots, v_p\right) \in \mathbb{R}^p$, we define the shifted ReLU activation function $\phi_v: \mathbb{R}^p \rightarrow \mathbb{R}^p$ as
\begin{equation*}
\phi_v(x)=\left(\phi\left(x_1-v_1\right), \ldots, \phi\left(x_p-v_p\right)\right), \quad x=\left(x_1, \ldots, x_p\right) \in \mathbb{R}^p.
\end{equation*}
For a number $L \in \mathbb{N}$ of hidden layers and a vector $\mathcal{A}=\left(p_0, p_1, \ldots, p_{L+1}\right) \in$ $\mathbb{N}^{L+2}$, a neural network of depth $L+1$ and architecture $\mathcal{A}$ is a function of the form
\begin{equation}\label{def:NN}
f: \mathbb{R}^{p_0} \rightarrow \mathbb{R}^{p_{L+1}}, \quad f(x)=W_L \circ \phi_{v_L} \circ W_{L-1} \circ \phi_{v_{L-1}} \circ \cdots \circ W_1 \circ \phi_{v_1} \circ W_0 \circ x,
\end{equation}
where $W_i \in \mathbb{R}^{p_{i+1} \times p_i}$ are weight matrices and $v_i \in \mathbb{R}^{p_i}$ are shift vectors. \\

Now we want to combine the approximation properties of ReLU networks with our above analysis to evaluate the performance of the Flow Matching mechanism from \eqref{ODEhat} where $\mathcal{N}=\mathsf{NN}(L,S,\Gamma)$ is a set of ReLU networks $v\colon[0,1]\times\mathbb R^d\to\mathbb R^d,(t,x)\mapsto v_t(x)$ with a fixed maximal number of layers $L$, at most $S$ nonzero weights and Lipschitz constant of $v_t$ of at most $\Gamma$ for any $t$. We obtain the following rate:

\begin{theorem}\label{thm:gaussian_kernel_optimal_rate} Assume $d\geq 2$, grant Assumption~\ref{ass:sigma+mu} and suppose $p^* \in B^{\alpha}_{1, \infty}(\Lambda, [-1,1]^d)$ for $\alpha\in(0,1], \Lambda >0$. Set $\sigma_{\min} = n^{- \frac{1}{2\alpha +d}}.$ Then we have for $n$ big enough
    \begin{equation*}
    \mathbb{E}\big[\mathsf{W}_1(\mathbb{P}^*, \mathbb{P}^{\hat{\psi}_1}) \big]\lesssim n^{- \frac{1+ \alpha}{2\alpha +d}}\log^2(n),
    \end{equation*} where $\hat{\psi}$ is given by \eqref{ODEhat} with $\mathcal{N}=\mathsf{NN}(L_n,S_n,\Gamma_n)$ where $L_n = \Theta(n^{\ell})$, $S_n =\Theta(\exp(n^{\ell}))$ and $\Gamma_n = \Theta(n^{\frac{3}{2\alpha +d}})$ for some $\ell >0$.
\end{theorem}
The proof of \Cref{thm:gaussian_kernel_optimal_rate} combines \Cref{thm:error_decomp} with \Cref{thm:kde_rate_smooth}, the tail behavior of the Gaussian distribution and an approximation result of \cite{Guehring2020}. In a nutshell, we approximate the vector field $v^n$ on a growing but bounded set with an accuracy that compensates the exponential dependence on the Lipschitz constant. Given a Wasserstein $2$ version of \Cref{thm:kde_rate_smooth}, the same approach could be used to extend the result to $\mathbb{E}[\mathsf{W}_2(\mathbb{P}^*, \mathbb{P}^{\hat{\psi}_1(Z)})].$

     \ifthenelse{\boolean{toggleappendix}}{}{\showproof{gaussian_kernel_optimal_rate}}
	 \begin{remark}$ $
     \begin{enumerate}
         \item The Lipschitz constraint in \Cref{thm:gaussian_kernel_optimal_rate} is typically not enforced in practice. However, we only require a bound on the Lipschitz constant of order $\Gamma_n=\frac{2d+1}{\sigma_{\min}^3}+\frac{1}{2}$ which is only a mild restriction for $\sigma_{\min}\to 0$ as $n\to\infty$.
         \item The network size in \Cref{thm:gaussian_kernel_optimal_rate} is not bounded from above. The empirical reference measure $\mathbb{P}^{\psi_1^n(Z)}$ in \Cref{thm:error_decomp} circumvents the need for concentration inequalities, which in turn avoids complexity bounds. Thus there is no trade-off in approximation power and network size.
         \item Combining \Cref{thm:kde_rate_nonsmooth} with \Cref{thm:error_decomp} and the proof strategy used in \Cref{thm:gaussian_kernel_optimal_rate} leads to a rate of convergence of $n^{-\frac{1+\alpha}{2\alpha + 2 + d}}$ for $p^*$ with unbounded support that satisfy a moment and tail condition.
         \item Compared to the literature on diffusion models and the first results on Flow Matching, we do not use early stopping times, but rather set $\sigma_{\min} >0$ according to $n.$ In terms of variance, both procedures are interchangeable: calculating $\sigma_{t_*}$ in \citet[Theorem 9]{fukumizu2024} leads to exactly the same variance. However, early stopping leads to a bias induced by the mean function $\mu$, which does not scale proportional to $1$ at a time $t<1.$  This could be fixed using an adapted mean function, which is proportional to $1$ at the time of the early stopping, but this leads to much more complicated structures of $\mu.$
     \end{enumerate}
	 \end{remark}
\Cref{thm:gaussian_kernel_optimal_rate} shows that Flow Matching achieves minimax optimal rates (up to logarithmic factors) for certain classes of unknown densities, giving some justification for its empirical success. There are several limitations and potential avenues for further research in the results of this work. 
 
First of all the networks used for \Cref{thm:gaussian_kernel_optimal_rate} are very large and do not correspond to the networks used in practice. While theoretical worst-case error bounds will always lead to networks that are not exactly small, the work of \cite{yakovlev2025} gives hope for improvement. Up to now it is an open question whether a smaller network would come at the cost of a non-optimal rate. In addition, the size of the network and the overall result could benefit from the theoretical use of more sophisticated networks, like the U-net \citep{ronneberger2015} construction in \cite{lipman2023}. 

We have focused on the statistical analysis of Flow Matching, neglecting the optimization process as a source of additional error as well as the numerical error for solving the ODE \eqref{ODEhat}. Including these errors would be more in line with the real world. The same applies to including the sampling of $t\sim \mathcal{U}[0,1]$ in the Flow Matching objective. 

\subsection{Density estimation on a manifold}\label{manifold}
The rate in \Cref{thm:gaussian_kernel_optimal_rate} depends on the dimension $d$ and therefore exhibits the classical curse of dimensionality. In this subsection, we analyze a case, in which the curse of dimensionality can be circumvented and the rate depends on some intrinsic dimension $d^{\prime} < d$. 

To this end, we consider a single-chart $d^{\prime}$-dimensional manifold $\mathcal{M} $, i.e.\ there exists a smooth map $g \colon \mathbb{R}^{d^{\prime}} \rightarrow \mathbb{R}^d$ such that
\begin{equation*}
    \mathcal{M} = g(\mathbb{R}^{d^{\prime}}),
\end{equation*} where $\mathcal{M}$ is a closed subset of $\mathbb{R}^d$ and $g$ is a diffeomorphism onto its image.
Related settings are commonly assumed in generative modeling, specifically in GANs, see for example \cite{Schreuder2020}.
For our analysis, we impose the following assumption on the smoothness of $g$.
\begin{assumption}\label{ass:manifold}Let $g \in C^{\max \{2,  \lceil d^{\prime}/2 \rceil\} }$ and $\|g\|_{\infty, [-1,1]^{d^{\prime}}} \leq M,$ $ \sup_{y \in \mathbb{R}^{d^{\prime}}}\|Dg(y) \|\leq M_1, \sup_{y \in \mathbb{R}^{d^{\prime}}}\|D(Dg(y)) \|\leq M_2$. 
For $d^{\prime}\geq 4$ we assume that higher partial derivatives up to order $\lceil d^{\prime}/2 \rceil$ are bounded.
    The Jacobian $Dg$ is such that for every $y \in \mathbb{R}^{d^{\prime}}$, $0 < m < \lambda_{\min}(Dg(y))\leq \lambda_{\max}(Dg(y))<M_3$, where $\lambda_{\min}$ is the smallest and $\lambda_{\max}$ is the largest singular value and $m, M_3 >0$ are constants. The inverse $g^{-1}$ of $g$ is twice continuously differentiable. For technical simplicity, we further assume that $0 \in \mathcal{M}.$
\end{assumption}

We further assume that $\mathcal{M} $ has a positive, finite reach $0 < \tau < \infty$. The notion of the reach of a manifold is due to \cite{federer1959curvature}. First, we define the distance function to $\mathcal{M} $ for $x \in \mathbb{R}^d$ as $\operatorname{dist}(x, \mathcal{M})=\inf _{z \in \mathcal{M} }|x-z|$. Then, the medial axis of $\mathcal{M} $ is given by
\begin{equation*}
    \operatorname{Med}(\mathcal{M} )=\{x \in \mathbb{R}^d \mid \exists \ y, z \in \mathcal{M}, y \neq z,|y-x|=|z-x|=\operatorname{dist}(x, \mathcal{M} )\}.
\end{equation*}
The reach of $\mathcal{M} $ is then defined as
\begin{equation*}
    \tau=\inf _{x \in \mathcal{M} } \operatorname{dist}(x, \operatorname{Med}(\mathcal{M} )).
\end{equation*}

For $r < \tau$ and $r < c_A^{-1}, $ where $ c_A = \frac{2 M_2 \tau +3mM_1}{2 m^2 \tau } $, let $\pi \colon \mathbb{R}^{d}\rightarrow \mathcal{M}  \cup \{0 \}$ be a function such that $\pi\big|_{B(\mathcal{M} ,r )}$ is the orthogonal projection onto $\mathcal{M} $ and $\pi(B(\mathcal{M}, r )^{c})=\{0 \}$.
In this subsection, we assume that $\mathbb{P}^*$ is supported on $\mathcal{M} $ and that there is an $\alpha$-Besov-smooth distribution $\mathbb{P}^*_{d^{\prime}}$ on $[-1,1]^{d^{\prime}}$ such that $g(Y) \sim \mathbb{P}^*$ when $Y \sim \mathbb{P}^*_{d^{\prime}}$. The next theorem bounds the estimation error of the KDE in a proximity of the manifold.

\begin{theorem}\label{thm:kde_rate_dimred} 
For $d^{\prime} \geq 2, n \geq 2$, assume $\mathbb{P}^*_{d^{\prime}}$ has a Lebesgue density $p^*_{d^{\prime}}\in B^{\alpha}_{1, \infty}(\Lambda,  [-1,1]^{d^{\prime}})$ for $\alpha \in (0,1], \Lambda >0$. Further assume that the kernel $K$ satisfies the assumptions of \Cref{thm:kde_rate_smooth}. Then there are $C_1$ and $C_2$ such that for $v^n$ from \eqref{eq:def_vt},
\begin{equation*}
    \mathbb{E}\big[\mathsf{W}_1(\mathbb{P}^*, \pi_\#\mathbb{P}^{\psi^n_1(Z)})\big]\leq C_1\sigma_{\min}^{1+\alpha} + \frac{C_2}{\sqrt{\sigma_{\min}^{d^{\prime}-2}}} \frac{\log n}{\sqrt n}.
\end{equation*}
For $\sigma_{\min} \sim (n/\log^2(n))^{-\frac{1}{2\alpha + d^{\prime}}}$ we obtain $\mathbb{E}\big[\mathsf{W}_1(\mathbb{P}^*, \pi_\#\mathbb{P}^{\psi^n_1(Z)})\big] = \mathcal{O}\big((n/\log^2(n))^{-\frac{1+\alpha}{2\alpha + d^{\prime}}}\big)$.
\end{theorem}
The proof of \Cref{thm:kde_rate_dimred} is inspired by and uses techniques from \cite{amari_estimatingreach} and \cite{Genovese_ridge_estimation}. The proof starts with a decomposition of the projection into tangential vectors and a remainder. The remainder is bounded using a Taylor series and the reach assumption. The rest of the proof follows along the lines of \Cref{thm:kde_rate_smooth}, but is carefully adapted to the manifold and projection setting. While the rigorous arguments are very technical, in a nutshell, the tangential noise is retraced to $\mathbb{R}^{d^{\prime}}$ and the space dependence of this retraction is shown to be small. Just as \Cref{thm:kde_rate_smooth}, the proof strongly relies on the dual form of the Wasserstein $1$ metric, thus extensions to Wasserstein $2$ are not straightforward.

The above result is in line with the similar results for kernel density estimators for distributions on manifolds, cf.\ \cite{Berenfeld2021,divol2022}. Note however that in these articles, the kernel is normalized with respect to the lower dimension, i.e. they consider $h^{-d^{\prime}}K(\cdot/h)$ for bandwidth $h>0$, while our generative setting requires us to prove \Cref{thm:kde_rate_dimred} for kernels which are normalized with respect to the ambient space dimension, i.e. $h^{-d}K(\cdot/h)$ (and $h=\sigma_{\min}$ in our notation). Furthermore, in \citet[Theorem 3.1]{divol2022}, the kernel estimator is evaluated on the manifold. Since the Flow Matching estimator results in generated samples and we evaluate the distribution of these samples, we cannot measure the Wasserstein distance along the manifold without the tube-projection, since the density of a distribution restricted to a null-set is not unique. This is different from the density estimation settings, where the density is either constructed directly on the manifold \citep{divol2022} or evaluated pointwise \citep{Berenfeld2021}. \Cref{sec:numerical_illustration} will give another motivation for this evaluation measure.

\begin{theorem}\label{thm:gaussian_kernel_optimal_rate_dimred} Grant Assumption~\ref{ass:sigma+mu} and assume the setting of \Cref{thm:kde_rate_dimred}. Set $\sigma_{\min} = n^{- \frac{1}{2\alpha +d^{\prime}}}.$ Then there are sequences $L_n,S_n,\Gamma_n\in \mathbb{N}$ such that for $n$ big enough
\begin{equation*}
   \mathbb{E}\big[ \mathsf{W}_1(\mathbb{P}^*, \pi_\#\mathbb{P}^{\hat{\psi}_1(Z)})\big]\lesssim n^{-\frac{1+\alpha}{2\alpha + d^{\prime}}}\log(n)^2,
\end{equation*}
where $\hat{\psi}$ is given by \eqref{ODEhat} with $\mathcal N=\mathsf{NN}(L_n,S_n,\Gamma_n)$ where $L_n = \Theta(n^{\ell})$, $S_n = \Theta(\exp(n^{\ell}))$ and $\Gamma_n = \Theta(n^{\frac{3}{2\alpha +d}})$ for some $\ell >0$.   
\end{theorem}

\Cref{thm:gaussian_kernel_optimal_rate_dimred} shows that the Flow Matching estimator can approximate $\mathbb{P}^*$ with a faster rate on a small tube around the manifold, exploiting the smoothness of $\mathbb{P}^*$ on the manifold. It uses exactly the same argument as \Cref{thm:gaussian_kernel_optimal_rate}.
For diffusion models, similar behavior has been observed by \cite{tang24,azangulov2024}.

\section{Numerical illustration}\label{sec:numerical_illustration}
While the empirical success is widely known, the main argument of this work is that in an overparameterized setting the Flow Matching algorithm is at least as good as the kernel density estimator. In a full dimensional setting, we still recover a minimax optimal rate of convergence. Particularly, the Flow Matching algorithm does not mimic the empirical measure, but a smoothed version. This enables leveraging smoothness, even in an overparameterized setting. In this section, we want to illustrate this comparison using a constructed example. 

To this end, we consider the following setting: the distribution $\mathbb{P}^*$ is the uniform distribution on the graph of the sine curve restricted to $[-3,3]$. Thus, $\mathbb{P}^*$ lies on a $1$-dimensional manifold in $\mathbb{R}^2$. We assume that we have $200$ samples from $\mathbb{P}^*$. In this setting, we train a Flow Matching model with three hidden SeLU layers of width $512$ and $\mathbb{U}= \mathcal{N}(0, I_2)$, for a definition of the SeLU activation function see \citet{klambauer2017self}. Our implementation builds on the code provided by \cite{tong2024}, we use the Adam optimizer \citep{kingma2014} with the standard parameters $(\operatorname{lr}=0.001$, $\beta_1=0.9$ and $\beta_2= 0.999)$ and \cite{politorchdyn} to solve the neural ODE with the solver \texttt{dopri5} and set $\mathrm{atol}= \mathrm{rtol} = 10^{-5}$. In order to analyze the model without avoidable additional smoothing, we do not use minibatches. We evaluated the estimator after $10\ 000, 20\ 000, 30\ 000$ and $40\ 000$ updates to avoid committing to a fixed number of training iterations. As a comparison, we use a KDE with a Gaussian kernel. For both methods, we generate $512$ samples. To interpret the KDE as a generative model, we draw an index uniformly, multiply a sample from $\mathcal{N}(0, I_2)$ with $\sigma_{\min}$ and then add the noise to the sample of the index. An example of this setting and generated samples from the generative models for $\sigma_{\min} \in \{0.5, 0.1, 0.05, 0.01\}$ can be found in \Cref{fig:two_rows_four_each}. We repeated the following experiments for networks with widths of $32$ and $8$, the results can be found in \Cref{fig:four_images} in \Cref{appendix}.

\begin{figure}[htbp]
    \centering

    \begin{subfigure}[b]{0.24\textwidth}
        \centering
        \includegraphics[trim= 0cm 0cm 0cm 1.88cm, clip, width=\textwidth]{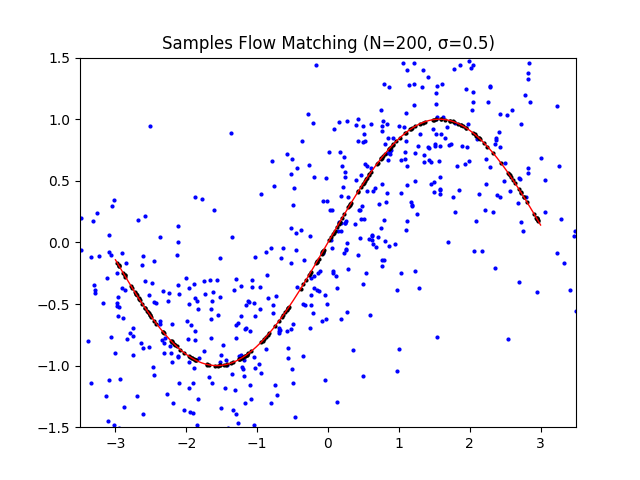}
        \caption{FM, $\sigma_{\min}=0.5$}
    \end{subfigure}
    \hfill
    \begin{subfigure}[b]{0.24\textwidth}
        \centering
        \includegraphics[trim= 0cm 0cm 0cm 1.88cm, clip,width=\textwidth]{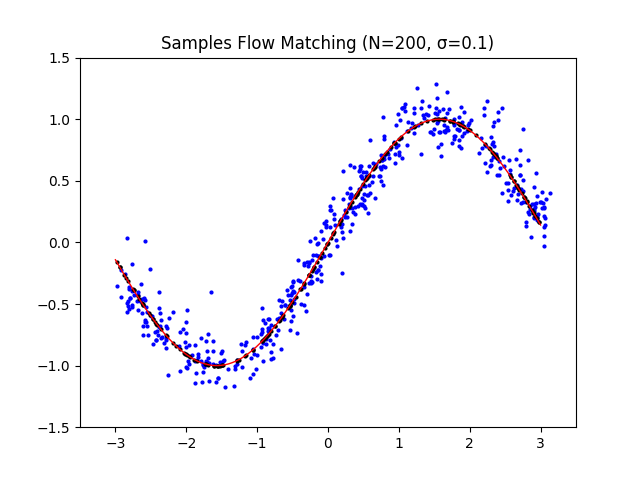}
        \caption{FM, $\sigma_{\min}=0.1$}
    \end{subfigure}
    \hfill
    \begin{subfigure}[b]{0.24\textwidth}
        \centering
        \includegraphics[trim= 0cm 0cm 0cm 1.88cm, clip,width=\textwidth]{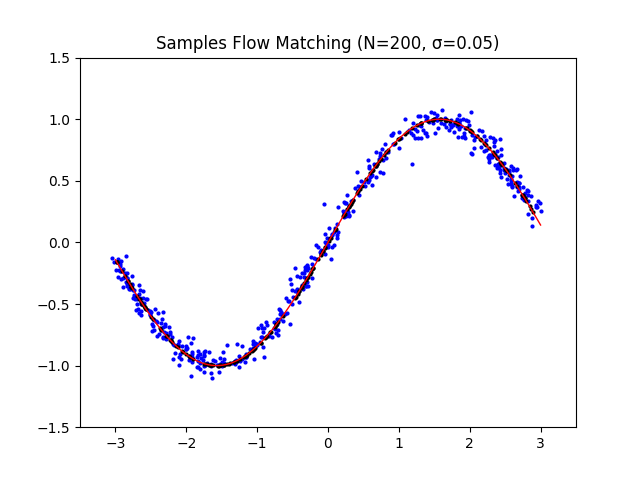}
        \caption{FM, $\sigma_{\min}=0.05$}
    \end{subfigure}
    \hfill
    \begin{subfigure}[b]{0.24\textwidth}
        \centering
        \includegraphics[trim= 0cm 0cm 0cm 1.88cm, clip,width=\textwidth]{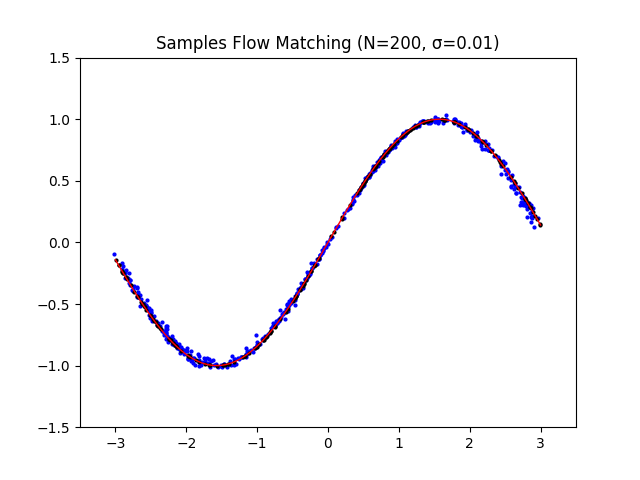}
        \caption{FM, $\sigma_{\min}=0.01$}
    \end{subfigure}

    \vspace{0.5cm}

    \begin{subfigure}[b]{0.24\textwidth}
        \centering
        \includegraphics[trim= 0cm 0cm 0cm 1.88cm, clip,width=\textwidth]{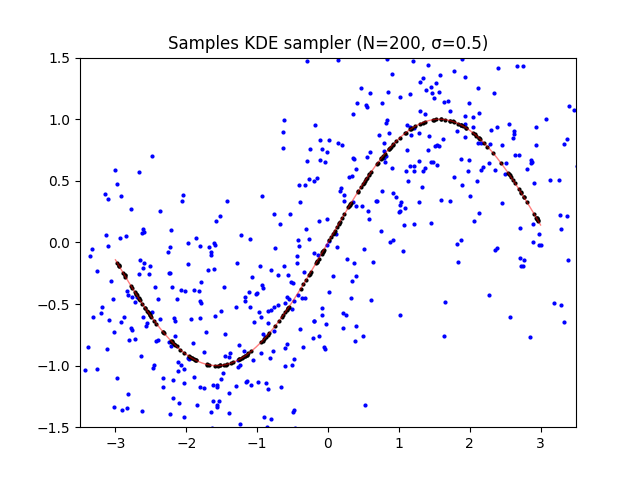}
        \caption{KDE, $\sigma_{\min}=0.5$}
    \end{subfigure}
    \hfill
    \begin{subfigure}[b]{0.24\textwidth}
        \centering
        \includegraphics[trim= 0cm 0cm 0cm 1.88cm, clip,width=\textwidth]{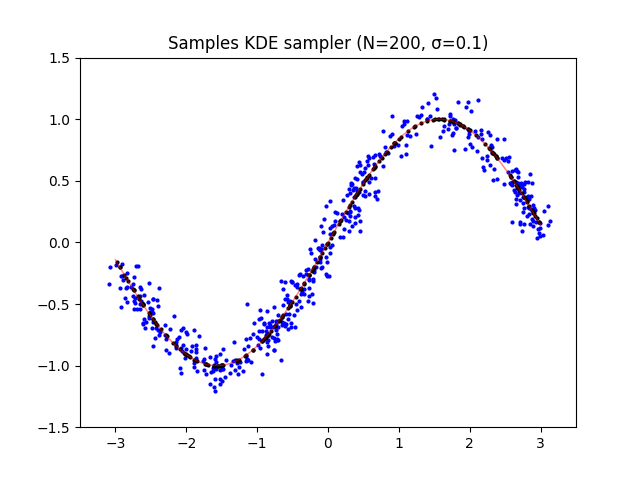}
        \caption{KDE, $\sigma_{\min}=0.1$}
    \end{subfigure}
    \hfill
    \begin{subfigure}[b]{0.24\textwidth}
        \centering
        \includegraphics[trim= 0cm 0cm 0cm 1.88cm, clip,width=\textwidth]{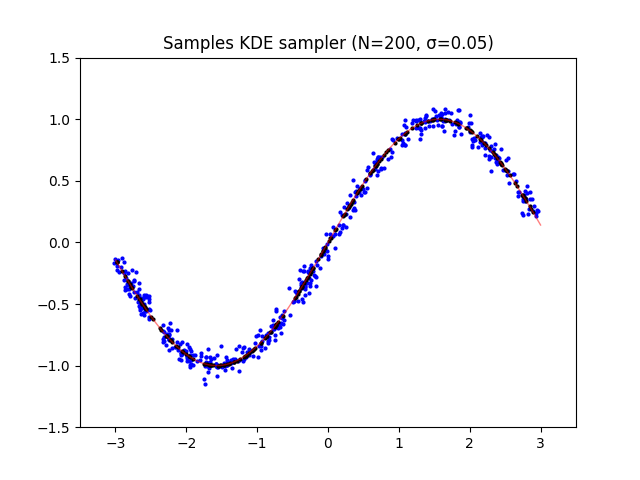}
        \caption{KDE, $\sigma_{\min}=0.05$}
    \end{subfigure}
    \hfill
    \begin{subfigure}[b]{0.24\textwidth}
        \centering
        \includegraphics[trim= 0cm 0cm 0cm 1.88cm, clip,width=\textwidth]{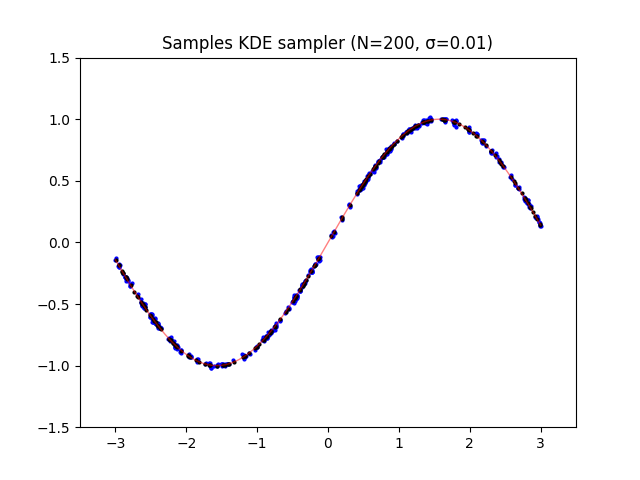}
        \caption{KDE, $\sigma_{\min}=0.01$}
    \end{subfigure}

    \caption{Example results of the Flow Matching (FM) estimator and the KDE. The generated samples are the blue dots, the training sample is represented by black dots on the red sine curve. The top row shows generated samples of the FM estimator after $40000$ runs.}
    \label{fig:two_rows_four_each}
\end{figure}

By construction, we know that the KDE does not detect the manifold, but spreads noise in two dimensions. The question is whether the Flow Matching estimator is capable of detecting the manifold, without imposing structural assumptions on the manifold. To measure this, we calculate the mean distance of the generated samples to the manifold. We repeated the same experiment independently drawing new samples $10$ times.  The results are gathered in \Cref{fig:mean_distance_100}. For all choices of $\sigma_{\min}$ except the largest, the KDE achieves the smallest mean distance, which strongly opposes the claim that the Flow Matching estimator detects the manifold better than the KDE.

Naturally, for very small $\sigma_{\min}$, the KDE is very close to the empirical distribution. In proofs, i.e.\ the proof of \Cref{thm:kde_rate_smooth} only a proper choice of $\sigma_{\min}$ leads to the optimal rate. The Flow Matching estimator does not behave the same, the model seems to smooth intrinsically. Albeit we ran the model for a very large number of training iterations, there is no mechanism to prevent the convergence towards a local minimum. Therefore this effect could be due to the optimization error. More desirable is an intrinsic smoothing.
In order to measure this effect, we calculated the maximal distance between two projections on the sine curves in the $10$ runs. In \Cref{fig:largest_gap} it is visible that for smaller $\sigma_{\min}$, the Flow Matching estimator approximately covers the manifold better than the KDE, whereas this cannot be seen for larger $\sigma_{\min}$. While for the uniform distribution that is considered here, gaps in the covering are a satisfactory indication for the quality of estimation, this measure fails for other distributions. In hindsight, this motivates the projected measure evaluated in \Cref{thm:kde_rate_dimred}.

\begin{figure}[htbp]
    \centering
    
    \begin{subfigure}[b]{0.49\textwidth}
         \centering
    \includegraphics[width=\linewidth]{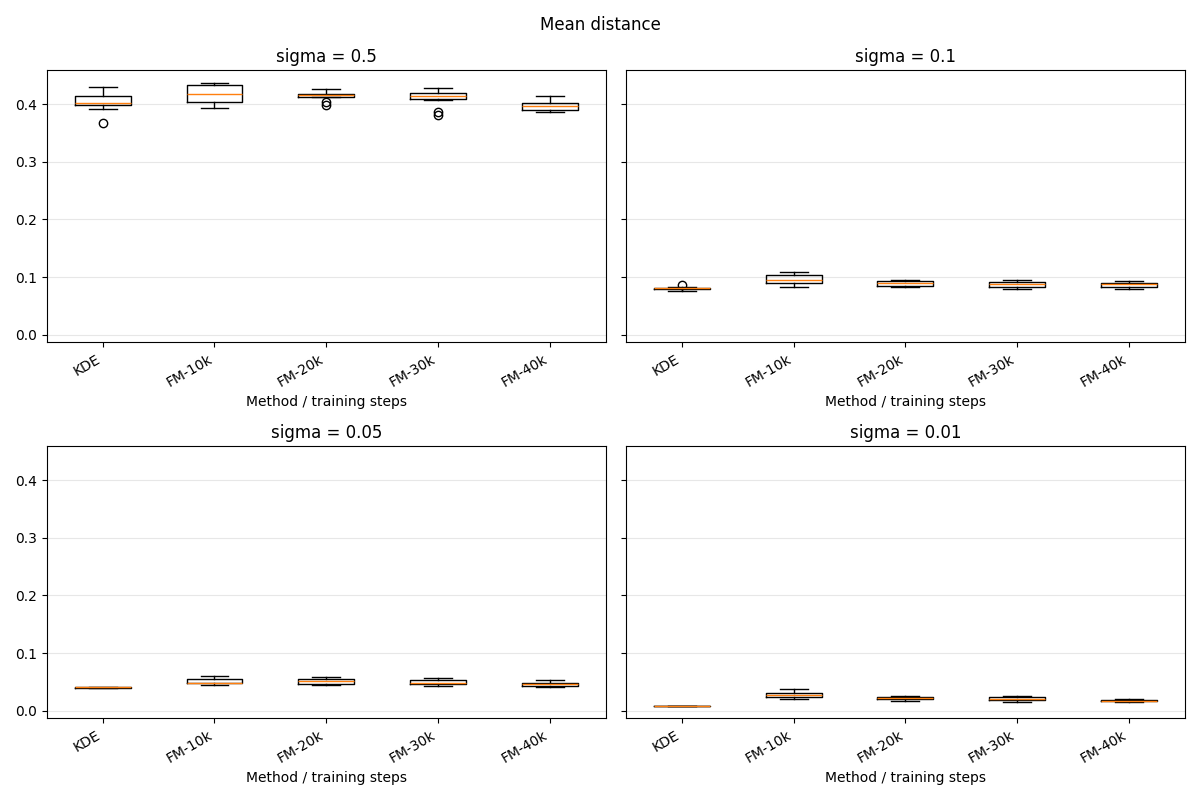}
    \caption{Mean distance of the generated samples to the manifold. Note that the shape of the sine curve leads to a smaller expected distance than $\sigma_{\min}$, this effect is smaller for smaller $\sigma_{\min}$.}
    \label{fig:mean_distance_100}
    \end{subfigure}
    \hfill
    \begin{subfigure}[b]{0.49\textwidth}
        \centering
    \includegraphics[width=\linewidth]{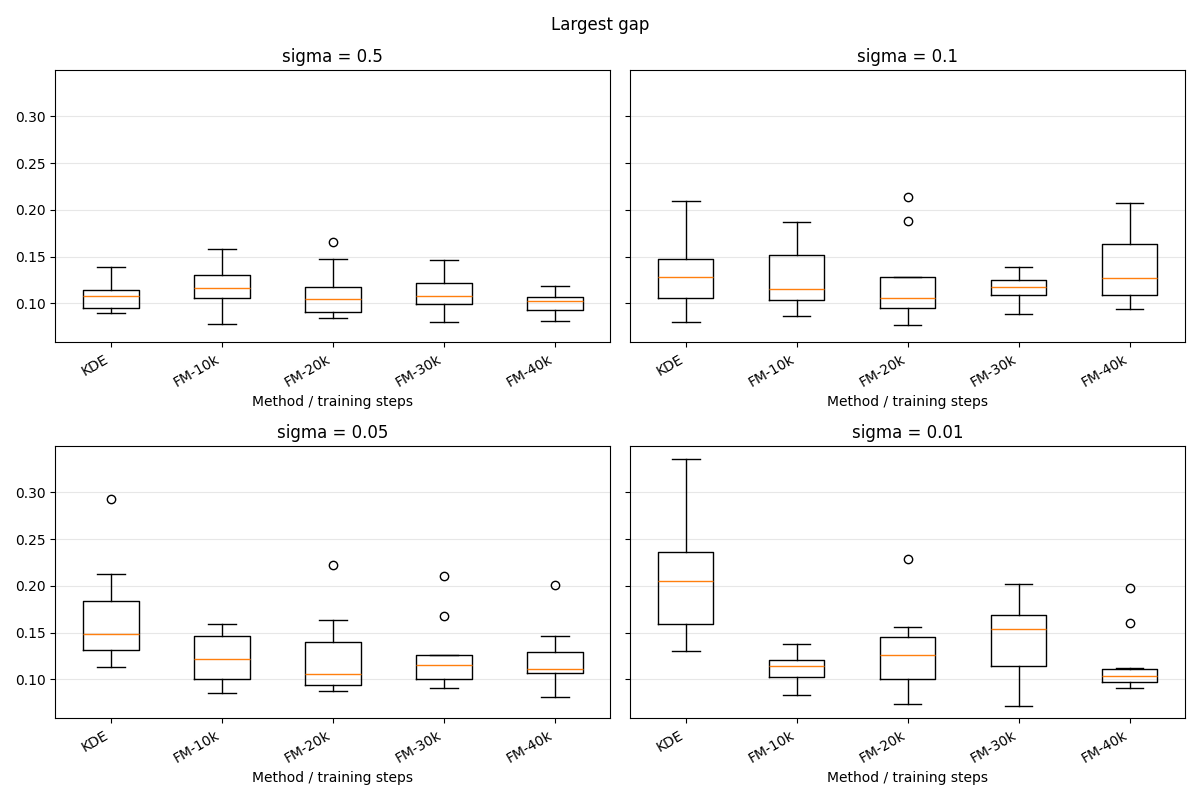}
    \caption{Largest gap on the manifold between the projections of the generated samples. \vspace{0.8cm}}
    \label{fig:largest_gap}
    \end{subfigure}
    
    \caption{Comparison between the Flow Matching estimator and the KDE in terms of distance to the manifold and the largest gap of the projection estimator on the manifold.}
    \label{fig:Estimators}
\end{figure}

Our results indicate that in order to theoretically find situations in which the Flow Matching algorithm outperforms the KDE in very general settings without structural assumptions on the unknown distribution, classical distances between probability measures might not be the right choice. While creating an image that is different from the training images appears as creative and is favorable, the performance of the corresponding distribution in the Wasserstein $1$ distance is not necessarily good. In the context of GANs, an inclusion of the distance between the generated distribution and the empirical measure has recently been investigated by \cite{vardanyan2024statistically}.

\section{Proofs}\label{sec:proofs}

Before we begin with the proofs, let us introduce some additional notation: The spectral norm of $A$ is denoted by $\|A\|_2.$ We define the ceiling of $x \in \mathbb{R}$ as $\lceil x\rceil:=\min \{k \in \mathbb{Z}| k \geq x\}.$
    For $\alpha \in (0,1]$, we define the Hölder space $\mathcal{H}^{\alpha}$ by
    \begin{equation}\label{def:hoelder}
    \mathcal{H}^{\alpha}(\Omega) \coloneqq \big\{ f \colon \Omega \rightarrow \mathbb{R}^{d} \colon \| f \|_{\mathcal{H}^{\alpha}(\Omega)} < \infty \big\}, \quad \text{ where } \quad     \| f \|_{\mathcal{H}^{\alpha}(\Omega)} \coloneqq \sup_{\substack{x,y \in \Omega \\ x \neq y}} \frac{|f(x)-f(y)|}{|x-y|^{\alpha}} +  \| f \|_{\infty, \Omega} .\end{equation}
For $\alpha = 1,$ we obtain the space of Lipschitz functions. We abbreviate
$
\operatorname{Lip}(f) = \| f \|_{\mathcal{H}^{1}(\Omega)}.
$
Again, in case of $\Omega = \mathbb{R}^{d},$ we omit $\Omega.$

   \subsection{Proofs of \Cref{sec:empirical_flow_matching}}

\begin{proof}[Proof of \Cref{thm:vectorfield_generates_ppath}]
	  A necessary and sufficient condition for $v_t$ to generate $p_t$ is given in \eqref{nec_and_suf_cond}. We thus verify
	\begin{align*}
		\frac{d}{d t} p^n_t(x) & = \frac{1}{n} \sum_{i = 1}^{n}\frac{d}{d t} p_t(x|X_i^*) =-\frac{1}{n} \sum_{i = 1}^{n}\operatorname{div}(p_t(x|X_i^*) v_t(x|X_i^*)) \\
		& =- \operatorname{div}\Big(\frac{1}{n}\sum_{i = 1}^{n} p_t(x|X_i^*) \sum_{i = 1}^{n} \frac{p_t(x|X_i^*)v_t(x|X_i^*)}{\sum_{j = 1}^{n} p_t(x|X_j^*)} \Big) = - \operatorname{div}(p^n_t(x)v^n_t(x)). \qedhere
	\end{align*}
\end{proof}

	\begin{proof}[Proof of \Cref{thm:cond_equiv}]
				For fixed $t \in [0,1]$ we have
		\begin{align*}
			|\hat{v}_t(x)-v^n_t(x)|^2 & =|\hat{v}_t(x)|^2-2\langle \hat{v}_t(x), v^n_t(x)\rangle+|v^n_t(x)|^2, \\
			|\hat{v}_t(x)-v_t(x | X_i^*)|^2 & =|\hat{v}_t(x)|^2-2\langle \hat{v}_t(x), v_t(x | X_i^*)\rangle+|v_t(x | X_i^*)|^2.
		\end{align*}
		The last term does not influence the minimal argument in $\hat{v}.$ For the first two we have
		\begin{align*}
			\mathbb{E}_{X_t \sim p^n_t}[|\hat{v}_t(X_t)|^2] =\int |\hat{v}_t(x)|^2p^n_t(x) \; \mathrm{d}x = \frac{1}{n} \sum_{i = 1}^{n} \int |\hat{v}_t(x)|^2 p_t(x|X_i^*) \; \mathrm{d}x = \frac{1}{n} \sum_{i = 1}^{n}  \mathbb{E}_{\tilde{X}_t \sim p_t(\cdot|X_i^*)} [|\hat{v}_t(\tilde{X}_t)|^2],
		\end{align*}
		and
		\begin{align*}
			\mathbb{E}_{X_t \sim p^n_t}[\langle \hat{v}_t(X_t), v^n_t(X_t)\rangle] & = \int \langle \hat{v}_t(x), v^n_t(x)\rangle   p^n_t(x) \; \mathrm{d}x\\
			& = \int \Big\langle \hat{v}_t(x), \sum_{i = 1}^{n} v_t(x|X_i^*)\frac{p_t(x|X_i^*)}{ \sum_{i = 1}^{n} p_t(x|X_i^*)	} \Big\rangle    \frac{1}{n} \sum_{i = 1}^{n} p_t(x|X_i^*) \; \mathrm{d}x\\
			& = \frac{1}{n}\sum_{i = 1}^{n} \int \langle \hat{v}_t(x), v_t(x|X_i^*)  \rangle p_t(x|X_i^*)\; \mathrm{d}x\\
			& =  \frac{1}{n} \sum_{i = 1}^{n}  \mathbb{E}_{\tilde{X}_t \sim p_t(\cdot|X_i^*)} [\langle \hat{v}_t(\tilde{X}_t), v_t(\tilde{X}_t | X_i^*)\rangle].\qedhere
		\end{align*}
	\end{proof}

 \begin{proof}[Proof of \Cref{thm:cond_vt_depends_on_mu_and_sigma}]    

	As $\sigma_t(X_i^*)>0$ for all $t,X_i^*,$ we have that
	\begin{equation*}
		\psi_t^{-1}(x|X_i^*)= \frac{x-\mu_t(X_i^*)}{\sigma_t(X_i^*)}.
	\end{equation*}
	Hence
	\begin{equation*}
			\frac{\partial \psi_t }{\partial t}(z | X_i^*)\Big|_{z = \psi_t^{-1}(x|X_i^*)} = v_t(x |  X_i^*), 
	\end{equation*}
	 and
	\begin{equation*}
			\frac{\partial \psi_t }{\partial t}(x| X_i^*) = \frac{\partial \sigma_t }{\partial t}(X_i^*)x+\frac{\partial \mu_t }{\partial t}(X_i^*).
	\end{equation*} 
	Thus, we get
	\begin{equation*}
			v_t\left(x  | X_i^*\right)=\frac{\frac{\partial\sigma_t }{\partial t}\left(X_i^*\right)}{\sigma_t\left(X_i^*\right)}\big(x-\mu_t\left(X_i^*\right)\big)+\frac{\partial\mu_t }{\partial t}\left(X_i^*\right) . \qedhere
	\end{equation*}
\end{proof}

    \subsection{Proofs and extended results of \Cref{sec:rate_of_convergence}}\label{sec:proofs_sec3}
    
    \begin{proof}[Proof of \Cref{ex:tv}]

    	For the Wasserstein $1$ distance, we can bound
    	\begin{align*}
    		\mathsf{W}_1(\mathbb{P}^{\psi_1^{(1)}(Z)}, \mathbb{P}^{\psi_1^{(2)}(Z)}) &\leq \mathbb{E}_{Z \sim \mathcal{N}(0,1)}[|\psi_1^{(1)}(Z)-\psi_1^{(2)}(Z)|]\\
    		& = \mathbb{E}_{Z \sim \mathcal{N}(0,1)}[|Z-\psi_1^{(2)}(Z)|]\\
    		& = \mathbb{E}_{Z \sim \mathcal{N}(0,1)}\Big[\Big|\int_0^1 v_t^{(2)}(\psi^{(2)}_t(Z))\; \mathrm{d}t\Big|\Big]\\
    		& \leq  \int_0^1 \mathbb{E}_{Z \sim \mathcal{N}(0,1)}\Big[\Big|\varepsilon \sin\Big(\frac{\psi^{(2)}_t(Z)}{\varepsilon} \Big)\Big|\Big] \; \mathrm{d}t\\
    		& \leq \varepsilon.
    	\end{align*}
        Analogously, we obtain for $x \in \mathbb{R}$ and $t \in [0,1]$
        \begin{equation*}
            |\psi_t^{(2)}(x) - x|\leq t \varepsilon.
        \end{equation*}
    	For the total variation distance, we first note that by the change-of-variables theorem, we have for every element of a sequence of diffeomorphisms $(f_k)_{k \in \mathbb{N}}$
    	\begin{equation*}
    		f_k(Z) \sim \frac{\varphi(f_k^{-1}(\cdot))}{|f_k^{\prime}(f_k^{-1}(\cdot))|}, 
    	\end{equation*}
    	where $\varphi$ is the density of $\mathcal{N}(0,1)$ and where we identified the density with the distribution. By definition of $v^{(1)}$, we have that $\psi^{(1)}_t(Z) \sim \varphi$.
    	In case $f_k^{\prime}(x)> 0$ for all $x \in \mathbb{R}$, we obtain for the total variation distance
    	\begin{align}
    		\operatorname{TV}(\mathcal{N}(0,1), \mathbb{P}^{f_k(Z)}) & = \frac{1}{2} \int \Big|\varphi(x) - \frac{\varphi(f_k^{-1}(x))}{|f_k^{\prime}(f_k^{-1}(x))|}\Big|\; \mathrm{d}x \notag \\
    		& = \frac{1}{2} \int |\varphi(f_k(x)) f_k^{\prime}(x)- \varphi(x)| \; \mathrm{d}x \notag\\
    		& = \frac{1}{2} \int |\varphi(x)(f_k^{\prime}(x) -1) + f_k^{\prime}(x)(\varphi(f_k(x)) -\varphi(x)) |\; \mathrm{d}x \notag\\
    		& \geq \frac{1}{2} \int |\varphi(x)(f_k^{\prime}(x) -1)|\; \mathrm{d}x - \frac{1}{2} \int | f_k^{\prime}(x)(\varphi(f_k(x)) -\varphi(x)) |\; \mathrm{d}x. \label{TV_bound}
    	\end{align}
    	In case $\sup_{k \in \mathbb{N}}\|f_k^{\prime}\|_{\infty} \leq C < \infty$ and for $k$ large enough such that $\|f_k - \operatorname{Id}\|_{\infty}\leq 1$, we have
        \begin{equation*}
            |f_k(x)-x|\leq 1 \quad \Longrightarrow \quad |f_k(x)|\geq |x|-1,
        \end{equation*}
        which implies
        \begin{equation*}
           | f_k^{\prime}(x)(\varphi(f_k(x)) -\varphi(x)) | \leq C \Big(\frac{1}{\sqrt{2 \pi}} e^{-\frac{((|x|-1)_{+})^2}{2}} + \varphi(x)\Big),
        \end{equation*}
        where $a_{+} \coloneqq \max \{0,a\}$ which accounts for a simpler bound if $|x|< 1$.
        Since the right-hand side of the above equation is integrable, we can use dominated convergence and the continuity of $\varphi$ to conclude that the second term in \eqref{TV_bound} converges towards $0$ for $k \rightarrow \infty $. 
        
        However, for the first term, we also need that 
        \begin{equation*}
            \int \varphi(x)|f^{\prime}_k (x) -1 |\; \mathrm{d}x \overset{k \rightarrow \infty}{\longrightarrow} 0.
        \end{equation*}
        
        In the Flow Matching setting, this corresponds to the derivative of $\psi^{(2)}_1$ with respect to $x$ converging to $1$ in the weighted $L^1$ norm, which is not controlled via the Flow Matching objective.
    	
    	In the given example, we can solve the corresponding ODEs and obtain
    	\begin{align*}
    		\frac{\partial \psi^{(1)}_t(x)}{\partial t} = 0, \quad  &\psi^{(1)}_0(x) = x,  &\Longrightarrow \quad  &\psi^{(1)}_t(x) = x\\
    		\frac{\partial \psi^{(2)}_t(x)}{\partial t} = \varepsilon\sin\Big(\frac{ \psi^{(2)}_t(x)}{\varepsilon}\Big), \quad  &\psi^{(2)}_0(x) = x, &\Longrightarrow \quad  &\psi^{(2)}_t(x) = 2 \varepsilon \arctan \Big(e^{t} \tan \Big(\frac{x}{2\varepsilon}\Big)\Big) + 2 \pi \varepsilon \Big\lfloor \frac{x}{2 \pi \varepsilon} + \frac{1}{2} \Big\rfloor.
    	\end{align*}
    	where $\psi^{(2)}_t$ is defined for almost all $x\in \mathbb{R}$. Note that $\big\lfloor \frac{x}{2 \pi \varepsilon} + \frac{1}{2} \big\rfloor$ is chosen such that we obtain a continuous, differentiable and bijective function on $\mathbb{R}$, if the points where the tangent term in $\psi^{(2)}_t$ is not well defined are set to the appropriate value. 
    	The solution of $\psi^{(2)}$ can be checked by calculating the derivative with respect to $t$ and using $\sin(2 \theta) = \frac{2 \tan(\theta)}{1+\tan^2(\theta)}$.
    	The derivative with respect to $x$ of $\psi_t^{(2)}$ is given by
    	\begin{equation*}
    		\frac{\partial\psi^{(2)}_t}{\partial x} = \frac{e^t \sec ^2\big(\frac{x}{2 \varepsilon}\big)}{e^{2 t} \tan ^2\big(\frac{x}{2 \varepsilon}\big)+1} = \frac{e^t}{\big(e^{2 t}-1\big) \sin ^2\big(\frac{x}{2 \varepsilon}\big)+1} > 0.
    	\end{equation*}
        Since $t \in [0,1]$, the derivative is bounded.
        Therefore, we can use the lower bound from \eqref{TV_bound}. By construction $\psi_1^{(2)}(x) \rightarrow x$ for $\varepsilon \rightarrow 0$, thus the second term in \eqref{TV_bound} converges towards $0$. 
    For the first term
    	\begin{align*}
    		\int \Big|&\varphi(x)\Big(\frac{e}{\big(e^{2 }-1\big) \sin ^2\big(\frac{x}{2 \varepsilon}\big)+1} -1\Big)\Big|\; \mathrm{d}x  \\ & \geq \frac{1}{\sqrt{2\pi e}}\int_{-1}^1 \Big|\Big(\frac{e}{\big(e^{2 }-1\big) \sin ^2\big(\frac{x}{2 \varepsilon}\big)+1} -1\Big)\Big|\; \mathrm{d}x \\
    		& \geq \frac{1}{2\sqrt{2\pi e}}\int_{-1}^1 \mathds{1}\Big\{\frac{e}{\big(e^{2 }-1\big) \sin ^2\big(\frac{x}{2 \varepsilon}\big)+1} \leq \frac{1}{2}\Big\}\; \mathrm{d}x\\
    		& = \frac{1}{2\sqrt{2\pi e}}\int_{-1}^1 \mathds{1}\Big\{x \in \bigcup_{n \in \mathbb{Z}} \Big[2 \varepsilon \Big( \pi n+ \arcsin\Big( \sqrt{\frac{2e-1}{e^2-1}}\Big)\Big), 2 \varepsilon \Big( \pi (n+1)- \arcsin\Big( \sqrt{\frac{2e-1}{e^2-1}}\Big)\Big) \Big]\Big\}\; \mathrm{d}x.
    	\end{align*}
    	Denoting $\gamma \coloneqq \arcsin\Big( \sqrt{\frac{2e-1}{e^2-1}}\Big) $, the length of the intervals in the bound is
    	$
    		2 \varepsilon ((n+1)\pi -\gamma) - 2 \varepsilon (n \pi +\gamma) = 2 \varepsilon (\pi -2 \gamma).
    	$
    	Next we count how many intervals lie completely in $[-1,1]$. We have that
    	\begin{equation*}
    		2 \varepsilon (\pi n + \gamma )\geq -1 \Longleftrightarrow n \geq -\frac{1}{2\varepsilon\pi}- \frac{\gamma}{\pi}, \quad 2 \varepsilon (\pi (n+1) - \gamma )\leq 1\Longleftrightarrow n  \leq \frac{1}{2 \varepsilon \pi}+ \frac{\gamma}{\pi}-1.
    	\end{equation*}
    	Therefore, we have 
    	\begin{equation*}
    		\Big\lfloor \frac{1}{2 \varepsilon \pi}+ \frac{\gamma}{\pi}-1\Big\rfloor - \Big\lceil -\frac{1}{2\varepsilon\pi}- \frac{\gamma}{\pi} \Big\rceil +1
    	\end{equation*}
    	intervals that lie completely in $[-1,1]$. We bound this number via
    	\begin{equation*}
    		\Big\lfloor \frac{1}{2 \varepsilon \pi}+ \frac{\gamma}{\pi}-1\Big\rfloor - \Big\lceil -\frac{1}{2\varepsilon\pi}- \frac{\gamma}{\pi} \Big\rceil +1 \geq  \frac{1}{\varepsilon\pi}+ \frac{\gamma}{\pi}-2.
    	\end{equation*}
    	Thus, we can bound
    	\begin{align*}
    		\int \Big|\varphi(x)\Big(\frac{e}{\big(e^{2 }-1\big) \sin ^2\big(\frac{x}{2 \varepsilon}\big)+1} -1\Big)\Big|\; \mathrm{d}x &\geq \frac{1}{2\sqrt{2\pi e}} 2 \varepsilon (\pi -2 \gamma)\big(\frac{1}{\varepsilon\pi}+ \frac{\gamma}{\pi}-2\big)\\&=  \frac{1}{2\sqrt{2\pi e}} \Big(\frac{2}{\pi}(\pi -2 \gamma)+ 2 \varepsilon \gamma \frac{\pi - 2 \gamma}{\pi}- 4 \varepsilon (\pi - 2 \gamma) \Big)\\
    		& \overset{\varepsilon \rightarrow 0}{\longrightarrow} \frac{\pi - 2 \gamma}{\pi \sqrt{2 \pi e}}.
    	\end{align*}
    	Setting $0<c < \frac{\pi - 2 \gamma}{2\pi \sqrt{2 \pi e}} $ yields the result.
    \end{proof}
\begin{proof}[Proof of \Cref{thm:error_decomp}]

	As $\mathsf{W}_1$ and $\mathsf{W}_2$ satisfy the triangle inequality (see \citet[p. 94]{Villani2008}), 
	\begin{align*}
		\mathsf{W}_q(\mathbb{P}^*, \mathbb{P}^{\hat{\psi}_1(Z)}) \leq \mathsf{W}_q(\mathbb{P}^*, \mathbb{P}^{\psi^n_1(Z)}) + \mathsf{W}_q( \mathbb{P}^{\psi^n_1(Z)}, \mathbb{P}^{\hat{\psi}_1(Z)}).
	\end{align*}
	Using \citet[Remark 6.5]{Villani2008} we get
	\begin{align*}
		\mathsf{W}_1( \mathbb{P}^{\psi^n_1(Z)}, \mathbb{P}^{\hat{\psi}_1(Z)}) & \leq \mathsf{W}_2( \mathbb{P}^{\psi^n_1(Z)}, \mathbb{P}^{\hat{\psi}_1(Z)})\\ & \leq \sqrt{\mathbb{E}[|\psi^n_1(Z) - \hat{\psi}_1(Z)|^2]} \\ & = \Big( \int |\psi^n_1(x) - \hat{\psi}_1(x) |^2 p_0(x) \; \mathrm{d}x\Big)^{1/2}.
	\end{align*}
	Like in \cite[Proposition 3]{albergo2022} set \begin{equation*}
		Q_t = \int  |\psi^n_t(x) - \hat{\psi}_t(x) |^2 p_0(x) \; \mathrm{d}x.
	\end{equation*}
	Since $p^*, \psi,$ $\hat{\psi}$ and $v_t^n$ are integrable with respect to the Lebesgue measure, we can use the dominated convergence theorem and interchange integration and differentiation to obtain
	\begin{align*}
		\frac{d}{d t}Q_t &= 2 \int (\psi^n_t(x) - \hat{\psi}_t(x) )^{\top}(v^n_t(\psi^n_t(x))-\hat{v}_t(\hat{\psi}_t(x))) p_0(x) \; \mathrm{d}x \\
		& = 2 \int (\psi^n_t(x) - \hat{\psi}_t(x) )^{\top}(v^n_t(\psi^n_t(x))-\hat{v}_t(\psi^n_t(x))) p_0(x) \; \mathrm{d}x \\
		& \quad + 2 \int (\psi^n_t(x) - \hat{\psi}_t(x) )^{\top}(\hat{v}_t(\psi^n_t(x))-\hat{v}_t(\hat{\psi}_t(x))) p_0(x) \; \mathrm{d}x. 
	\end{align*}
	As $0 \leq | \psi^n_t(x) - \hat{\psi}_t(x)  - (v^n_t(\psi^n_t(x))-\hat{v}_t(\psi^n_t(x)))|^{2}, $ we get
	\begin{equation*}
		2(\psi^n_t(x) - \hat{\psi}_t(x) )^{\top}(v^n_t(\psi^n_t(x))-\hat{v}_t(\psi^n_t(x))) \leq |\psi^n_t(x) - \hat{\psi}_t(x)  |^2 + | v^n_t(\psi^n_t(x))-\hat{v}_t(\psi^n_t(x))|^2,
	\end{equation*}
	and since for fixed $t$ the vector field $\hat{v}_t$ is $\Gamma_t$ Lipschitz continuous
	\begin{equation*}
		2 (\psi^n_t(x) - \hat{\psi}_t(x) )^{\top}(\hat{v}_t(\psi^n_t(x))-\hat{v}_t(\hat{\psi}_t(x))) \leq 2 \Gamma_t |\psi^n_t(x) - \hat{\psi}_t(x)|^2.
	\end{equation*}
	Therefore 
	\begin{equation*}
		\frac{d}{d t}Q_t \leq (1+2\Gamma_t) Q_t + 2\int | v^n_t(\psi^n_t(x))-\hat{v}_t(\psi^n_t(x))|^2 p_0(x) \; \mathrm{d}x.
	\end{equation*}
	Now we use the formulation of Grönwall's lemma in \cite{Walter1970} and use $Q_0 = 0$ (this holds since $\psi_0 (x) = \hat{\psi}_0 (x) = x$). This leads to 
	\begin{equation}\label{Gronwall_bound}
		Q_1 \leq e^{\int_{0}^1 1+2\Gamma_t\; \mathrm{d}t}2\int_{0}^{1} \int | v^n_t(\psi^n_t(x))-\hat{v}_t(\psi^n_t(x))|^2 p_0(x) \; \mathrm{d}x\; \mathrm{d}t.
	\end{equation}
	By \Cref{thm:cond_equiv} the vector field $\hat{v}$ is chosen such that it minimizes $ \int_{0}^{1} \mathbb{E}_{X_t \sim p^n_t}\big[|\hat{v}_t(X_t)- v^n_t(X_t)|^2\big]\; \mathrm{d}t$ and $X_t = \psi^n_t(Z),$ thus we get for every $\tilde{v} \in \mathcal{N}$
	\begin{align*}
		Q_1 &\leq 2e^{\int_{0}^1 1+2\Gamma_t\; \mathrm{d}t}\int_{0}^{1} \int | v^n_t(\psi^n_t(x))-\tilde{v}_t(\psi^n_t(x))|^2 p_0(x) \; \mathrm{d}x\; \mathrm{d}t\\
		& \leq  2e^{\int_{0}^1 1+2\Gamma_t\; \mathrm{d}t}\int_{0}^{1}  \|v^n_t - \tilde{v}_t \|^2_{\infty}   \; \mathrm{d}t,\\
		& \leq 2 e^{\int_{0}^1 1+2\Gamma_t\; \mathrm{d}t} \|v^n - \tilde{v} \|^2_{\infty}.
	\end{align*}
	Taking the square root leads to the result.
\end{proof}
\begin{proof}[Proof of \Cref{thm:kde_rate_smooth}]
	 
	Set $\mu_n \coloneqq \frac{1}{n} \sum_{i = 1}^{n} \delta_{X_i}.$ We denote $K_{\sigma_{\min}}\coloneqq \frac{1}{\sigma_{\min}^d} K(\frac{\cdot}{\sigma_{\min}}).$
	By the triangle inequality of the supremum we get
	\begin{align*}
		\mathsf{W}_1(\mathbb{P}^*, \mathbb{P}^{\psi^n_1(Z)}) &= \sup_{f \in \mathrm{Lip}(1)}  \int f(z)(p^*(z) - (K_{\sigma_{\min}} \ast \mu_n)(z)) \; \mathrm{d}z \notag \\
		& \leq \sup_{f \in \mathrm{Lip}(1)}  \int f(z)(p^*(z) - (K_{\sigma_{\min}} \ast p^*)(z)) \; \mathrm{d}z\notag \\ & \quad +\sup_{f \in \mathrm{Lip}(1)}  \int f(z)((K_{\sigma_{\min}} \ast p^*)(z)- (K_{\sigma_{\min}} \ast \mu_n)(z)) \; \mathrm{d}z. 
	\end{align*}
	The first term is the bias of the kernel density estimator and the second term refers to the stochastic error. \\
	\underline{First term:}
	Using Fubini's theorem and $\int K_{\sigma_{\min}}(y) \; \mathrm{d}y = 1$, we rewrite the first term as
	\begin{align*}
		\sup_{f \in \mathrm{Lip}(1)}  \int f(z)(p^*(z) &- (K_{\sigma_{\min}} \ast p^*)(z)) \; \mathrm{d}z \\& = \sup_{f \in \mathrm{Lip}(1)}  \int f(z)\Big(\int K_{\sigma_{\min}}(y) p^*(z) \; \mathrm{d}y- \int K_{\sigma_{\min}}(y)p^*(z-y)\; \mathrm{d}y\Big) \; \mathrm{d}z \\
		& =  \sup_{f \in \mathrm{Lip}(1)}  \int  K_{\sigma_{\min}}(y) \Big( \int f(z) p^*(z)\; \mathrm{d}z  - \int f(z) p^*(z-y)\; \mathrm{d}z \Big)\; \mathrm{d}y \\
		& = \sup_{f \in \mathrm{Lip}(1)}  \int  K_{\sigma_{\min}}(y) \big((f\ast p^*(-\cdot))(0)  - (f\ast p^*(-\cdot))(y) \big)\; \mathrm{d}y. 
	\end{align*}
	Then, since $|p^*|_{B^{\alpha}_{1, \infty}} \leq \Lambda,$ we have that
	\begin{align*}
		|\nabla (f \ast p^*(-\cdot ))(x) - \nabla (f \ast p^*(-\cdot ))(y)|& =  \Big|\int \nabla f(z) p^*(x-z) \; \mathrm{d}z - \int \nabla f(z) p^*(y-z) \; \mathrm{d}z\Big|\\
		& \leq \sqrt{d}\Lambda |x-y|^{\alpha},
	\end{align*}
	where $\nabla$ denotes the vector of weak partial derivatives and we use that $|\nabla f| \leq \sqrt{d}$ for $f \in \operatorname{Lip}(1)$. \\
	Therefore by the generalized mean value theorem, for every $y$ and every $f$ there exists a $\xi_{y,f}$ and a $C_1^{\prime}>0$ such that
	\begin{align*}
		\sup_{f \in \mathrm{Lip}(1)}  \Big| \int & K_{\sigma_{\min}}(y) \Big( \int f(z) p^*(z)\; \mathrm{d}z  - \int f(z) p^*(z-y)\; \mathrm{d}z \Big)\; \mathrm{d}y \Big| \\& =   \sup_{f \in \mathrm{Lip}(1)} \Big| \int  K_{\sigma_{\min}}(y) (\nabla(f \ast p^*(-\cdot))(\xi_{y,f}))^{\top}y\; \mathrm{d}y\Big|\\
		& =  \sup_{f \in \mathrm{Lip}(1)} \Big| \int  K_{\sigma_{\min}}(y) (\nabla(f \ast p^*(-\cdot))(\xi_{y,f}) -\nabla(f \ast p^*(-\cdot))(0) )^{\top} y\; \mathrm{d}y  \\ & \quad + \int  K_{\sigma_{\min}}(y) (\nabla(f \ast p^*(-\cdot))(0))^{\top} y \; \mathrm{d}y   \Big|\\
		& \leq   \sup_{f \in \mathrm{Lip}(1)}  \int  |K_{\sigma_{\min}}(y)| |(\nabla(f \ast p^*(-\cdot))(\xi_{y,f}) -\nabla(f \ast p^*(-\cdot))(0) | |y|\; \mathrm{d}y \\ & \quad  + \Big|\int  K_{\sigma_{\min}}(y) \nabla(f \ast p^*(-\cdot))(0))^{\top} y \; \mathrm{d}y  \Big| \\
		& \leq  C_1^{\prime} \int  |K_{\sigma_{\min}}(y)||y|^{1+\alpha}  \; \mathrm{d}y  +  \sup_{f \in \mathrm{Lip}(1)}  \Big|\int  K_{\sigma_{\min}}(y) (\nabla(f \ast p^*(-\cdot))(0))^{\top} y \; \mathrm{d}y\Big| . 
	\end{align*}
	For the first term we have that
	\begin{align*}
		\int  K_{\sigma_{\min}}(y)|y|^{1+\alpha}  \; \mathrm{d}y & = \frac{1}{\sigma_{\min}^d} \int  K(\frac{y}{\sigma_{\min}})|y|^{1+\alpha}  \; \mathrm{d}y\\
		& = \sigma_{\min}^{1+\alpha} \int  K(u)|u|^{1+\alpha}  \; \mathrm{d}u.
	\end{align*}
	For the last term we use the assumption $\int  K_{\sigma_{\min}}(y)y \; \mathrm{d}y = 0 .$ \\
	Hence we get
	\begin{align*}
		\sup_{f \in \mathrm{Lip}(1)}  \int f(z)(p^*(z) - (K_{\sigma_{\min}} \ast p^*)(z)) \; \mathrm{d}z \leq   \sigma_{\min}^{1+\alpha}C_1^{\prime} \int  K(u)|u|^{1+\alpha}  \; \mathrm{d}u.
	\end{align*}
	By assumption the last term is finite. Setting $C_1 \coloneqq C_1^{\prime} \int  K(u)|u|^{1+\alpha}  \; \mathrm{d}u$ yields the first term in the final bound. \\
	\underline{Second term:}	
	We rewrite the second term as
	\begin{align*}
		\sup_{f \in \mathrm{Lip}(1)}  \int f(z)((K_{\sigma_{\min}} \ast p^*)(z)&- (K_{\sigma_{\min}} \ast \mu_n)(z)) \; \mathrm{d}z \\
		& = 	\sup_{\substack{f \in \mathrm{Lip}(1)\\ f(0) = 0}}\int (K_{\sigma_{\min}}(- \cdot) \ast f) (x)   p^*(x)\; \mathrm{d}x - \frac{1}{n} \sum_{i = 1}^n (K_{\sigma_{\min}}(- \cdot) \ast f) (X^*_i)  \\
		& = 	\sup_{\substack{g = K_{\sigma_{\min}}(- \cdot) \ast f\\f \in \mathrm{Lip}(1) \\ f(0) = 0}} \int g(x) p^* (x) \; \mathrm{d}x - \frac{1}{n} \sum_{i = 1}^n g(X_i^*).
	\end{align*}

	For any multiindex $k \in \mathbb{N}^d_0$ with $|k|\geq 1, |k| \leq\lceil \frac{d+2}{2}\rceil $, choose an index $j$ with $k_j \geq 1$ and set $k = \ell + e_j$ where $|k| = |\ell|+1$. 
Furthermore, we have for any $x \in \Omega$, where $\Omega$ is a fixed compact set containing $\operatorname{supp}(p^*)$ since $f(0) = 0$
\begin{equation*}
    |(K_{\sigma_{\min}}(-\cdot) \ast f)(x)| \leq \int K_{\sigma_{\min}}(-y) (|x|+ |y|) \; \mathrm{d}y\lesssim 1.
\end{equation*}
    
    Then using Young's inequality and the properties of convolution, we obtain

	\begin{equation*}
		\|D^{k}	(K_{\sigma_{\min}}(-\cdot) \ast f)\|_{\infty} \leq \|D^{e_j} f\|_{\infty} \|D^{\ell}K_{\sigma_{\min}(- \cdot)} \|_1 \leq 1 \cdot \|D^{\ell}K_{\sigma_{\min}} (-\cdot)\|_1.
	\end{equation*}
	By assumption $\|D^{\ell}K \|_1 \leq C $ and hence  $\|D^{\ell}K_{\sigma_{\min}} \|_1 \leq C \sigma_{\min}^{-|k|+1}. $
    Thus
    \begin{equation}\label{eq:bound_convolution}
        \|D^{k}	(K_{\sigma_{\min}} (-\cdot)\ast f)\|_{\infty} \lesssim \sigma_{\min}^{-|k|+1}.
    \end{equation}
    Now we can use \citet[Theorem 4]{Schreuder2020_2}. For even $d,$ set the smoothness parameter to $ \frac{d}{2}.$ Then there exists a constant $C^{\prime}_2 > 0$ such that
	\begin{align*}
		\mathbb{E} \Big[\sup_{\substack{g = K_{\sigma_{\min}}(- \cdot) \ast f\\f \in \mathrm{Lip}(1) \\ f(0) = 0}} \int g(x) p^* (x) \; \mathrm{d}x - \frac{1}{n} \sum_{i = 1}^n g(X_i^*)\Big]  \leq \frac{C^{\prime}_2}{\sigma_{\min}^{\frac{d}{2}-1}} n^{-1/2}\log(n).
	\end{align*}
	For odd $d>1,$ we know that
	\begin{align*}
		\|D^{\frac{d-1}{2}}(f \ast K_{\sigma_{\min}})\|_{\infty} \leq \| D^1 f\|_{\infty}\|D^{\frac{d-3}{2}}K_{\sigma_{\min}} \|_1\leq C^{\prime}_2 \sigma_{\min}^{-\frac{d-3}{2}},
	\end{align*} 
    and analogously
    \begin{equation*}
        \|D^{\frac{d+1}{2}}(f \ast K_{\sigma_{\min}})\|_{\infty} \lesssim \sigma_{\min}^{-\frac{d-1}{2}}.
    \end{equation*}
    Thus
    \begin{equation*}
        |D^{\frac{d-1}{2}}g(x) - D^{\frac{d-1}{2}}g(y)|\lesssim \min \Big\{\sigma_{\min}^{- \frac{d-3}{2}} , \sigma_{\min}^{- \frac{d-1}{2}}|x-y|\Big\}.
    \end{equation*}
    Applying the first bound in case $|x-y|> \sigma_{\min}$ and the second bound in case $|x-y| \leq \sigma_{\min}$ leads to
    
	\begin{align*}
		\frac{D^{\frac{d-1}{2}}(f \ast K_{\sigma_{\min}})(x) - D^{\frac{d-1}{2}}(f \ast K_{\sigma_{\min}})(y) }{|x-y|^{1/2}} \lesssim \sigma_{\min}^{-\frac{d}{2}+1}.
	\end{align*}
	Then we can use \citet[Theorem 4]{Schreuder2020_2} with $\alpha = \frac{d}{2}$ again. Therefore, there is a $C^{\prime\prime}_{2}>0$ such that
	\begin{align*}
		\mathbb{E} \Big[\sup_{\substack{g = K_{\sigma_{\min}}(- \cdot) \ast f\\f \in \mathrm{Lip}(1) \\ f(0) = 0}} \int g(x) p^* (x) \; \mathrm{d}x - \frac{1}{n} \sum_{i = 1}^n g(X_i^*)\Big] \leq \frac{C^{\prime\prime}_{2}}{\sigma_{\min}^{\frac{d}{2}-1}} n^{-1/2}\log(n).
	\end{align*}
	Combining both terms we conclude that there are constants $C_1$ and $C_2$ with
	\begin{equation*}
		\mathbb{E} \Big[\mathsf{W}_1(\mathbb{P}^*, \mathbb{P}^{\psi^n_1(Z)})\Big]\leq C_1\sigma_{\min}^{1+\alpha} + \frac{C_2}{\sigma_{\min}^{\frac{d}{2}-1}} n^{-1/2}\log(n). \qedhere
	\end{equation*}
\end{proof}

\paragraph{Proof of \Cref{thm:gaussian_kernel_optimal_rate}}
The proof of \Cref{thm:gaussian_kernel_optimal_rate} requires some additional results. The proof of the subsidiary lemmata can be found in \Cref{sec:helper}.
In order to control the error $e^{\int_0^1 \Gamma_t \; \mathrm{d}t}\|v^n-\hat{v}\|_{\infty}$ (where we omitted constants) we need to approximate the function $v^n$ and 
	control the Lipschitz constant of the approximation $\hat{v}_t$ for fixed $t.$ 
	However, we cannot expect to approximate a nontrivial function on $\mathbb{R}^d$ with respect to the supremum norm. Additionally, $v^n_t$ as constructed by \cite{lipman2023} admits a Lipschitz constant that diverges for $\sigma_{\min} \rightarrow 0$ (see \Cref{lem:v_tv_tn_supnorm_bound}). Hence we cannot expect to construct a network $\hat{v}$ that approximates $v^n$ and its Lipschitz constant on $\mathbb{R}^d$ with a small absolute error for every $x \in \mathbb{R}^d.$ 
    
    In the given setting, a rapid decay of the kernel function results in less significant approximation errors outside of the support of $p^*.$ We first show that in case of rapidly decreasing kernels the error $\mathsf{W}_1(\mathbb{P}^{\psi(Z)}, \mathbb{P}^{\hat{\psi}(Z)})$ can  be decomposed into two approximation errors, one depending on the approximation of $v$ on a compact domain and one corresponding to the area outside with little probability mass.

		\begin{lemma} \label{thm:tails}Assume that $\operatorname{supp}(p^*) \subset (-1,1)^d$ and $\sigma_{\min} \leq 1.$ 
        Additionally, assume \Cref{ass:sigma+mu}.
		Then for every $a > 1$ there is a network $\tilde{v},$ which is cutoff outside $(-a,a)^d$ such that
		\begin{equation*} 
			\mathsf{W}_1(\mathbb{P}^{\psi_1^n(Z)}, \mathbb{P}^{\hat{\psi}_1(Z)}) \leq \sqrt{2 e} e^{\int_{0}^1 \Gamma_t\; \mathrm{d}t} \Big(
			\int_{0}^{1} \int_{(-a,a)^d} | v^n_t(\psi^n_t(x))-\tilde{v}_t(\psi^n_t(x))|^2 p_0(x) \; \mathrm{d}x\; \mathrm{d}t +  c \frac{a e^{-\frac{a^2}{2}}}{\sigma_{\min}^2}\Big)^{1/2}.
		\end{equation*}
	\end{lemma}

Next we want to choose $a$ in \Cref{thm:tails} and a neural network of finite size (depending on $n$) such that the bound of $	\mathsf{W}_1(\mathbb{P}^{\psi^n(Z)}, \mathbb{P}^{\hat{\psi}(Z)}) $ decays at a rate of $n^{-\frac{1+\alpha}{2\alpha + d}}.$ Hence we need to approximate 
			\begin{equation*}
		v^n_t(x) = \sum_{i = 1}^{n} v_t(x|X_i^*)\frac{p_t(x|X_i^*)}{ \sum_{i = 1}^{n} p_t(x|X_i^*)	}
	\end{equation*} and simultaneously control the Lipschitz constant of the approximation $\hat{v}$ on $(-a,a)^d.$ This is indeed sufficient, since using the same reasoning as in the proof of \Cref{thm:tails}, it holds that if $x \in (-a,a)^d$, then $\psi_t^n(x) \in (-a,a)^d$ for every $t$ and a choice of $a$ big enough. 
As $\hat{v}$ approximates $v^n$ in the supremum norm, the Lipschitz constant of $\hat{v}$ will be at least of the order of the Lipschitz constant of $v^n$. The same applies to the supremum norm. Hence we need bounds for both, the Lipschitz constant and the supremum norm of $v^n.$

	 	\begin{lemma}\label{lem:v_tv_tn_supnorm_bound} Assume that $\operatorname{supp}(p^*) \subset [-1,1]^d.$  For every $x \in (-a,a)^d$ we have that
	 	\begin{equation*}
	 		|v^n_t(x)| \leq \frac{\sqrt{d}(1+a)}{1-(1-\sigma_{\min})t} \leq \frac{\sqrt{d}(1+a)}{\sigma_{\min}}.
	 	\end{equation*}
       Further for the Gaussian kernel and $\operatorname{supp}(p^*) \subset [-1,1]^d$, for any $x \in \mathbb{R}^d$ and every $t \in [0,1]$
       \begin{equation*}
           \operatorname{Lip}(v^n_t) \leq \frac{1}{\sigma_t}+ \frac{2d}{\sigma_t^3}.
       \end{equation*}
	 \end{lemma}

From \citet[Corollary 4.2]{Guehring2020} we know that for any $f \in C^2$ with bounded second partial derivatives we can approximate $f$ and its Lipschitz constant simultaneously using a ReLU network as defined in \eqref{def:NN}. 
To be more precise, for every $\varepsilon \in (0, 1/2)$ there is a ReLU network $f_{\mathsf{NN}}$ such that
\begin{equation}\label{eq:approx_guehring}
\left\|f_{\mathsf{NN}}-f\right\|_{\mathcal{H}^1\left((0,1)^d\right)} \leq \varepsilon.
\end{equation}
The number of hidden layers $L$ of $f_{\mathsf{NN}}$ is bounded by $ c \cdot \log \left(\varepsilon^{-2 /(2-1)}\right)$ and the number of nonzero weights $S$ is bounded by $ c \cdot \varepsilon^{-d /(2-1)} \cdot \log^2\left(\varepsilon^{-2 /(2-1)}\right)$, where $c$ is a constant independent of $n$.
Note that \cite{Guehring2020} refer to ReLU networks as defined in \eqref{def:NN} as standard neural networks, since they also consider networks with skip connections.\\

Looking at the proof of \Cref{lem:v_tv_tn_supnorm_bound} and using standard analysis, we can bound the second partial derivatives of $v^n$ on $(-a,a)^d$ with respect to $t$ and $x$ by $\frac{C}{\sigma_{\min}^7}(1+a)^2$ for some $C >0$.
Now we are ready to prove \Cref{thm:gaussian_kernel_optimal_rate}.
\begin{proof}[Proof of \Cref{thm:gaussian_kernel_optimal_rate}]

From \Cref{thm:tails}, we know that
\begin{equation*} 
	\mathsf{W}_1(\mathbb{P}^{\psi_1^n(Z)}, \mathbb{P}^{\hat{\psi}_1(Z)}) \leq \sqrt{2 e} e^{\int_{0}^1 \tilde{L}_t\; \mathrm{d}t}\inf_{\tilde{v} \in \mathcal{N}} \Big(
	\int_{0}^{1} \int_{(-a,a)^d} | v^n_t(\psi^n_t(x))-\tilde{v}_t(\psi^n_t(x))|^2 p_0(x) \; \mathrm{d}x\; \mathrm{d}t +  c \frac{a e^{-\frac{a^2}{2}}}{\sigma_{\min}^2}\Big)^{1/2}, 
\end{equation*}
where we used the choice of \eqref{ODEhat} before dividing the area in \eqref{einsetzen_erster_term} in the proof of \Cref{thm:tails} and $\tilde{L}$ is the Lipschitz constant of $\tilde{v}$ at time $t$. We neglect an arbitrarily small additional error on the right hand side to account for the use of the infimum, as the error bound corresponding to the kernel density estimator in \eqref{basic_orakel} is nonzero anyway.
Note that due to the restriction of the network class,
\begin{equation*}
	e^{\int_{0}^1 \tilde{L}_t\; \mathrm{d}t} \leq  e\cdot e^{\int_{0}^1 \frac{2d+1}{\sigma_{\min}^3}\; \mathrm{d}t}. 
\end{equation*}
To bound the tail, we need to assure $a$ is chosen such that
\begin{equation*}
	e^{\frac{2d+1}{\sigma_{\min}^3}} \frac{\sqrt {a}e^{-\frac{a^2}{4}}}{\sigma_{\min}} \leq n^{- \frac{1+\alpha}{2\alpha +d}}.
\end{equation*}
We can use that for $a \geq 1$
\begin{equation*}
	\sqrt{a}e^{- \frac{a^2}{4}} \leq e^{- \frac{a^2}{8}}.
\end{equation*}
For $\sigma_{\min} = n^{- \frac{1}{2\alpha +d}}$, the above inequality is satisfied if
\begin{equation*}
	a \geq \max \Big\{\Big(8\Big(\frac{2+ \alpha}{2 \alpha +d}\log(n) + (2d+1)n^{\frac{3}{2\alpha +d}}\Big)\Big)^{1/2}, 1 \Big\}.
\end{equation*}  

To cover $(-a,a)^d$ with cubes of size $(0,1)^d$, we need $\lceil 2a \rceil^d$ little cubes. Now we want to bound the approximation term. 
Since the vector field $v_t^n$ cannot point outwards of the cube $(-a,a)^d$ for $a > \frac{1}{1- \sigma_{\min}}$, which holds for $n$ big enough, we know that $\psi_t^n(x) \in (-a,a)^d$ for $x \in (-a,a)^d$. Thus we can bound
\begin{align*}
	\Big(\int_{0}^{1} \int_{(-a,a)^d} | v^n_t(\psi^n_t(x))-\tilde{v}_t(\psi^n_t(x))|^2 p_0(x) \; \mathrm{d}x\; \mathrm{d}t \Big)^{1/2} & \leq \| v^n- \tilde{v}\|_{\infty, (-a,a)^d \times (0,1)}.
\end{align*}
Note that $v_t^n \in C^{\infty}$ in case of the Gaussian kernel. Hence we can apply \citet[Corollary 4.2]{Guehring2020}.
The bound of the second partial derivatives as well as the supremum norm bound of $v^n$ influence the coefficients in \citet[Corollary 4.2]{Guehring2020}, therefore we construct the network such that $n^{-\frac{10}{2\alpha +d}}v^n$ is approximated and choose the last weight matrix such that it scales the result up. The factor stems from inserting the choice of $\sigma_{\min}$ into the bound of the second partial derivatives. This addition of one layer does not influence the order of the layers and nonzero weights of the network. However, the approximation error gets scaled up too, hence we need to ensure a smaller error. We note that the scaling was chosen such that it covers the worst case bound on the second derivatives over all $t \in [0,1]$.
We need to choose $\varepsilon$ in \eqref{eq:approx_guehring} such that
\begin{equation*}
	e^{(2d+1)n^{\frac{3}{2\alpha +d}}} \varepsilon \leq n^{- \frac{11+\alpha}{2\alpha +d}}  \quad \Longleftrightarrow \quad \varepsilon \leq n^{- \frac{11+\alpha}{2\alpha +d}}e^{-(2d+1)n^{\frac{3}{2\alpha +d}}}.
\end{equation*} The addition of one cutting layer to ensure the assumptions of \Cref{thm:tails} does not influence the order of the layers and nonzero weights of the network either.
Hence choosing 
\begin{align*}
	L_n &\gtrsim   \Big(\max\Big(\log(n) +n^{\frac{3}{2\alpha+d}}, 1\Big) \Big)^{d/2} \cdot  \big(\log(n) + n^{\frac{3}{2\alpha +d}} \big) ,\\
	S_n &\gtrsim \Big(\max\Big(\log(n) +n^{\frac{3}{2\alpha+d}}, 1\Big) \Big)^{d/2}\cdot  \big( n^{\frac{11+\alpha}{2\alpha +d}} e^{(2d+1)n^{\frac{3}{2\alpha +d}}}\big)^{d+1} \cdot \big(\log^2(n) + n^{\frac{6}{2\alpha +d}} \big),
\end{align*} yields the desired rate. We note that the $d+1$ in the lower bound on $S_n$ is due to the fact that we also need to approximate the time as an additional input.
\end{proof}

\begin{proof}[Proof of \Cref{thm:kde_rate_dimred}]
	For this proof, we need some auxiliary lemmata, whose proofs are postponed to \Cref{sec:helper}.
	First, we need to bound the Lipschitz constant of $\pi$ on $B(\mathcal{M}, r)$. 
    
	\begin{lemma}{\cite[Theorem 4.8 (8)]{federer1959curvature}}\label{lem:Lipschitz-constant}
The mapping $\pi\big|_{B( \mathcal{M}, r)}$ is Lipschitz continuous with Lipschitz constant $\frac{1}{1-\frac{r}{\tau}}$.
	\end{lemma}

	Then we decompose as in \Cref{thm:kde_rate_smooth}
	\begin{equation}\label{kde_dimred_decomp}
		\mathsf{W}_1(\mathbb{P}^*, \pi_\#\mathbb{P}^{\psi^n_1(Z)}) \leq \mathsf{W}_1(\mathbb{P}^*, \pi_\#\mathbb{P}^{\psi_1(Z)}) + \mathsf{W}_1(\pi_\#\mathbb{P}^{\psi_1(Z)}, \pi_\#\mathbb{P}^{\psi^n_1(Z)}).
	\end{equation}
	\underline{First term:}
	We first note that
	\begin{align}
		\mathsf{W}_1(\mathbb{P}^*, \pi_\#\mathbb{P}^{\psi_1(Z)}) & = \sup_{\substack{f \in \operatorname{Lip}(1)\\f(0)=0}} \mathbb{E}_{\substack{X \sim \mathbb{P}^*\\ U \sim K}}\big[ f(X)-f(\pi(X+\sigma_{\min} U)) \big] \notag\\
		& = \sup_{\substack{f \in \operatorname{Lip}(1)\\f(0)=0}} \mathbb{E}_{\substack{Y \sim \mathbb{P}^*_{d^{\prime}}\\ U \sim K}}\big[ f(g(Y))-f(\pi(g(Y)+\sigma_{\min} U)) \big] \notag \\
		& = \sup_{\substack{f \circ g = \varphi \\ f \in \operatorname{Lip}(1)\\ f(0)=0}} \mathbb{E}_{\substack{Y \sim \mathbb{P}^*_{d^{\prime}}\\ U \sim K}}\big[\varphi(Y)- \varphi(\rho_{\sigma_{\min}}(Y, U)) \big],\label{first_term_reformulated}
	\end{align}
	where we define the $\mathbb{R}^{d^{\prime}}$ counterpart of the projection as \begin{equation*}
		\rho_{\sigma_{\min}}(y,u) \coloneqq g^{-1}(\pi(g(y) + \sigma_{\min} u)).
	\end{equation*}
    Since $0 \in \mathcal{M}$, $g^{-1}$ is defined for all possible inputs in this case. 
	In the following, we derive a second order Taylor expansion of $\rho$, that linearizes along the manifold. This step of the proof is sketched in \Cref{fig:tangent_space}.
	\begin{figure}
		\centering
		\includegraphics[width=0.4\linewidth]{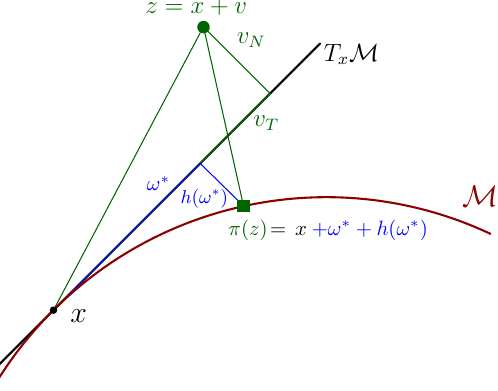}
		\caption{Setting for the Taylor expansion.}
		\label{fig:tangent_space}
	\end{figure}
	To this end, we start by bounding the projection of a disturbed point on the manifold by the magnitude of the perturbation and a vector in the tangent space of the point.
	Let $x \in \mathcal{M}$. Denote the tangent space of the point $x$ as $T_x\mathcal{M}$ and the normal space as $N_x\mathcal{M}$. Define the $\tau/2$ surrounding on $\mathcal{M}$ as
	\begin{equation*}
		\mathcal{M}_{x, \tau/2} \coloneqq \{ x+ \omega +h_x(\omega) \colon \omega \in T_x\mathcal{M}, |\omega| \leq \tau/2 \},
	\end{equation*}
	where $h_x \colon T_x\mathcal{M} \cap B( 0, \tau/2) \rightarrow N_x\mathcal{M}, \tau/2 >0$ is the function that maps a vector in the tangent space to the vector in the normal space such that $x+\omega + h_x(\omega) \in \mathcal{M}$. Note that this is a slightly larger set than $\mathcal{M}\cap B( \{x\}, \tau/2)$. Since $x \in \mathcal{M}$, we have that $h(0) = 0$. Since $T_x\mathcal{M}$ is a tangent space, $Dh(0) = 0$. Additionally $\|D^2 h (\omega)\| \lesssim \frac{1}{\tau}$, see \citet[Proposition 2.3]{amari_estimatingreach} referring to \citet[Proposition 6.1]{Niyogi2008-de}. A second order Taylor expansion, like in \cite{tang24,divol2022} yields
	\begin{equation}\label{eq:ungleichung_h}
		|h(\omega)| \lesssim \frac{1}{\tau} |\omega|^2.
	\end{equation}
    Let $0 < b < \min \{r, \tau/4, 1\}$.
	For $z \in B(\{x\}, b)$, let  $z = x + v$, with $|v|\leq b <1$. There is a unique decomposition $v = v_T + v_N$ with $v_T \in T_x\mathcal{M}$ and $v_N \in N_x\mathcal{M}$. Then $\pi(z) = x+\omega^*+h(\omega^*)$, where $\omega^*$ is the minimizer of
	\begin{equation}\label{minimizer}
		| z - (x+\omega+h(\omega))|^2 = | v- \omega -h(\omega)|^2 = |v_T - \omega|^2 + |v_N - h(\omega)|^2
	\end{equation}
	and $\pi(z)\in \mathcal{M}_{x,\tau/2}$. Since
    \begin{equation*}
        |\pi(x+v) - x|\leq |\pi(x+v) - (x+v)|+ |v| \leq 2b < \tau/2, 
    \end{equation*}
    there must be a unique minimizer $\omega^*$ in the interior of $\mathcal{M}_{x, \tau/2}$, which satisfies
	\begin{equation*}
		\omega^* - v_T = -Dh(\omega^*)^{\top}(h(\omega^*)-v_N).
	\end{equation*}
	Taking norms yields
	\begin{equation}\label{eq:bound_omega_vt}
		|\omega^* - v_T | \leq |Dh(\omega^*)| |h(\omega^*)-v_N|\leq \|Dh(\omega^*)\|  (| h(\omega^*)|+ | v_N| ).
	\end{equation}
	Due to the Pythagorean theorem, we can bound $|v_N |\leq |v|$ and $|v_T |\leq |v|$. Since $\omega^*$ is the minimizer of \eqref{minimizer}, we know that
	\begin{equation*}
		|\omega^* - v_T |^2 \leq  |\omega^* - v_T |^2 +  |v_N -h(\omega^*)|^2 \leq  |v_N -h(v_T)|^2  \leq 2 (|v|^2 + |h(v_T)|^2).
	\end{equation*}
	Using \eqref{eq:ungleichung_h}, we obtain using $|v|<1$
	\begin{equation*}
		|\omega^*| \leq |v_T| + |\omega^* - v_T| \lesssim |v| + \sqrt{ |v|^2 + \tau^{-2}|v|^4}\lesssim \frac{1}{\min(\tau, 1)}|v|.
	\end{equation*}
	Inserting this bound into \eqref{eq:bound_omega_vt} and using that by Taylor expansion $\|Dh(\omega^*)\|\lesssim \tau^{-1}|\omega^*|$ leads to 
	\begin{equation*}
		|\omega^* - v_T | \lesssim \frac{1}{\min(\tau, \tau^{3})} |v|^2.
	\end{equation*}
	Thus
	\begin{align*}
		&\pi(x+v)= x + \omega^* + h(\omega^*)\\
		\Longleftrightarrow \quad & \pi(x+v) - (x+v_T)=  \omega^* - v_T+ h(\omega^*)\\
		\Longrightarrow \quad & | \pi(x+v) - (x+v_T)|\leq |\omega^* - v_T| + |h(\omega^*)| \\
		\Longrightarrow \quad & | \pi(x+v) - (x+v_T)|\lesssim |v|^2.
	\end{align*}
	Therefore 
	\begin{equation*}
		\pi(x+v) = x + v_T + R(x,v), \quad \text{where}\quad | R(x,v)|\lesssim |v|^2.
	\end{equation*}
	
	Applying this to the projection of $g(y)+\sigma_{\min} u$ leads to
	\begin{equation*}
		\pi(g(y) + \sigma_{\min} u) = g(y) + \sigma_{\min} u_T + R(g(y), \sigma_{\min} u),
	\end{equation*} where $\|R(g(y), \sigma_{\min} u) \|\lesssim \sigma_{\min}^2 |u|^2$. Thus
	\begin{equation*}
		\pi(g(y)+\sigma_{\min} u)-g(y) = \sigma_{\min} u_T + R(g(y), \sigma_{\min} u) =: \eta.
	\end{equation*}

	A Taylor expansion of $g^{-1}$ around $g(y)$ combined with a decomposition of $\eta$ into tangential and normal part yields
	\begin{equation*}
		g^{-1}( \pi(g(y)+\sigma_{\min} u)) = g^{-1}(g(y)+\eta) = y + Dg^{-1}(g(y))P_{T_{g(y)}\mathcal{M}}\eta + O(| \eta|^2), 
	\end{equation*}
	where $O(\| \eta\|^2)$ represents a vector whose norm is of order $\| \eta\|^2$. Inserting the definition of $\eta$ leads to
	\begin{align*}
		\rho_{\sigma_{\min}}(y, u) =  g^{-1}( \pi(g(y)+\sigma_{\min} u)) &= y + Dg^{-1}(g(y))(\sigma_{\min} u_T + R(g(y), \sigma_{\min} u))) \\ & \quad + O(| \sigma_{\min} u_T + R(g(y), \sigma_{\min} u)|^2)\\
		& = y + \sigma_{\min} A(y)u + R^{\prime}(g(y), \sigma_{\min} u),
	\end{align*}
	where $| R^{\prime}(g(y), \sigma_{\min} u)|\lesssim \sigma_{\min}^2 |u|^2$ and $A(y) \coloneqq  Dg^{-1}(g(y)) \circ P_{T_{g(y)}\mathcal{M}}$  with $P_{T_{g(y)}\mathcal{M}}$ being the orthogonal projection onto the tangent space of $\mathcal{M}$ at the point $g(y)$.

	This leads to, using that $\pi$ sets everything to $0$ outside of the tube
    \begin{align*}
        \int \int \varphi (\rho_{\sigma_{\min}}(y, u) ) p^*_{d^{\prime}}(y)K(u)\; \mathrm{d}y \; \mathrm{d}u & = \int_{B(0, b)}\int \varphi (\rho_1(y, u) ) p^*_{d^{\prime}}(y)K_{\sigma_{\min}}(u)\; \mathrm{d}y \; \mathrm{d}u \\
        & \quad + \int_{B(0, b)^{c}}\int \varphi (\rho_1(y, u) ) p^*_{d^{\prime}}(y)K_{\sigma_{\min}}(u)\; \mathrm{d}y \; \mathrm{d}u. 
    \end{align*}
    For the first term, we obtain, since $f \in \operatorname{Lip}(1)$ and $\|D g\|_{\infty}\leq M_1$, which implies $\operatorname{Lip}(\varphi)\leq M_1$
    \begin{align}
\begin{aligned}
& \Big|\int_{B(0, b)} \int \varphi(\rho_1(y, u)) p_{d^{\prime}}^*(y) K_{\sigma_{\min }}(u) \; \mathrm{d} y \; \mathrm{d} u-\int_{B(0, b)} \int \varphi(y+A(y) u) p_{d^{\prime}}^*(y) K_{\sigma_{\min }}(u) \; \mathrm{d} y \; \mathrm{d} u \Big|\\
& \quad \lesssim \int_{B(0, b)}|u|^2 K_{\sigma_{\min }}(u) \; \mathrm{d} u \leq b^{1-\alpha} \int |u|^{1+ \alpha} K_{\sigma_{\min }}(u)\; \mathrm{d} u \lesssim \sigma_{\min}^{1+\alpha}.
\end{aligned}\label{eq:near_area}
    \end{align}

    For the integral over $B(0, b)^{c}$, we change the order of integration. For $y$ fixed, we define two sets
    \begin{equation*}
        A_{0}(y) \coloneqq \{u \in B(0, b)^{c}\mid g(y)+u \notin B( \mathcal{M}, r) \},\quad 
        A_{\text{far}} (y)\coloneqq \{ u \in B(0, b)^{c}\mid g(y)+u \in B( \mathcal{M}, r)\}.
    \end{equation*}
   On $A_{0}(y)$ it holds that $\pi(g(y)+u )= 0$. $A_{\text{far}}(y)$ consists of all $u$ such that $g(y)+u$ is close to the manifold, either within the tube around $g(y)$ or far away.
	Inserting this integral split and \eqref{eq:near_area} into \eqref{first_term_reformulated} leads to 
	\begin{align}
		\sup_{\substack{f \circ g = \varphi \\ f \in \operatorname{Lip}(1)\\ f(0)=0}} \mathbb{E}\big[\varphi(Y)- \varphi(\rho_{\sigma_{\min}}(Y, U)) \big]& \lesssim \sup_{\substack{f \circ g = \varphi \\ f \in \operatorname{Lip}(1)\\ f(0) = 0}} \int_{B(0, b)}\int (\varphi(y) - \varphi (y +  A(y)u )) p^*_{d^{\prime}}(y)K_{\sigma_{\min}}(u)\; \mathrm{d}y \; \mathrm{d}u \label{eq:innen}\\& \quad + \sigma_{\min}^{1+ \alpha} +\int \int_{ A_{0}(y) } (\varphi(y)-\varphi(g^{-1}(0)))  p^*_{d^{\prime}}(y)K_{\sigma_{\min}}(u)\; \mathrm{d}u \; \mathrm{d}y\label{eq:null}\\
        & \quad + \int \int_{ A_{\text{far}}(y)}(\varphi(y)-\varphi(g^{-1}(\pi(g(y)+u))))  p^*_{d^{\prime}}(y)K_{\sigma_{\min}}(u)\; \mathrm{d}u \; \mathrm{d}y.\label{eq:aussen}
	\end{align}
	For $\sigma_{\min}$ small enough, which will translate into $n$ large enough, we can bound \eqref{eq:null} using $\varphi\left(g^{-1}(0)\right)=f(0)=0$ and $|\varphi(y)|=|f(g(y))| \leq M$ on $[-1,1]^{d^{\prime}}$
    \begin{equation}\label{eq:A_0bound}
       \Big| \int \int_{ A_{0}(y) } (\varphi(y)-\varphi(g^{-1}(0)))  p^*_{d^{\prime}}(y)K_{\sigma_{\min}}(u)\; \mathrm{d}u \; \mathrm{d}y \Big| \leq M \int_{B(0, b)^{c}} K_{\sigma_{\min}}(u) d u \lesssim \sigma_{\min}^{1+\alpha}.
    \end{equation}

    For \eqref{eq:aussen}, we note that
    \begin{equation*}
        |\varphi(g^{-1}(g(y))) -\varphi(g^{-1}(\pi(g(y)+u))) |=|f(g(y))-f(\pi(g(y)+u))| \leq|g(y)-\pi(g(y)+u)|\leq |u| + r.
    \end{equation*}
    Then for $\sigma_{\min}$ small enough, with similar arguments as in \eqref{eq:A_0bound}
    \begin{align*}
        \int \int_{ A_{\text{far}}(y)}(\varphi(y)-\varphi(g^{-1}(\pi(g(y)+u))))  p^*_{d^{\prime}}(y)K_{\sigma_{\min}}(u)\; \mathrm{d}u \; \mathrm{d}y &\leq \int \int_{B(0, b)^{c}}  (|u| + r) p^*_{d^{\prime}}(y)K_{\sigma_{\min}}(u)\; \mathrm{d}u \; \mathrm{d}y \\
        & \lesssim \sigma_{\min}^{1+ \alpha}.
    \end{align*}
    To bound \eqref{eq:innen}, we first define
	\begin{equation*}
		\ell(u) \coloneqq \int \varphi(y+A(y)u)p^*_{d^{\prime}}(y)\; \mathrm{d}y.
	\end{equation*}
	Then, we can rewrite 
	\begin{equation}\label{eq:ell_formulation}
		\int_{B(0, b)}\int (\varphi(y) - \varphi (y +  A(y)u )) p^*_{d^{\prime}}(y)K_{\sigma_{\min}}(u)\; \mathrm{d}y \; \mathrm{d}u =   \int_{B(0, b)} K_{\sigma_{\min}}(u)(\ell(0)-\ell(u))\; \mathrm{d}u.
	\end{equation}
	Analogous to \Cref{thm:kde_rate_smooth}, we can use the multivariate mean value theorem to conclude that
	\begin{align*}
		\Big|\int_{B(0, b)} K_{\sigma_{\min}}(u)(\ell(0)-\ell(u))\; \mathrm{d}u \Big|& \leq \int_{B(0, b)} K_{\sigma_{\min}}(u)|u| \sup _{0 \leq t \leq 1}|\nabla \ell(t u)-\nabla \ell(0)| \; \mathrm{d} u \\ & \quad + |\nabla \ell(0)||\int_{B(0, b)} u K_{\sigma_{\min}}(u) \; \mathrm{d} u|.
	\end{align*}  
    Since by assumption $\int u K(u)\; \mathrm{d}u = 0$, we can bound
    \begin{equation*}
        \Big|\int_{B(0, b)} u K_{\sigma_{\min}}(u) \; \mathrm{d} u \Big| \lesssim b^{- \alpha} \sigma_{\min}^{1 + \alpha} \lesssim \sigma_{\min}^{1 + \alpha} .
    \end{equation*}
 Suppose that the mapping $F_u(y) \coloneqq y +A(y)u$ is a bijection for all $u\in B(0,r)$, and denote the inverse by $S_u(z)$. If that is the case, then
	\begin{equation}
		\nabla \ell(u) = \int A(y)^{\top}(\nabla \varphi)(y+A(y)u)p^*_{d^{\prime}}(y)\; \mathrm{d}y = \int A(S_u(z))^{\top} (\nabla \varphi)(z)|\operatorname{det}(D S_u(z)) |p^*_{d^{\prime}}(S_u(z))\; \mathrm{d}z.
	\end{equation}
Particularly, inserting $u = 0$ into the expression after the first equality yields
\begin{equation*}
    |\nabla \ell (0)| =\Big| \int A(y)^{\top}(\nabla \varphi)(y)p^*_{d^{\prime}}(y)\; \mathrm{d}y \Big| \leq \int \| A(y)\|  \|\nabla \varphi(y)\| p^*_{d^{\prime}}(y)\; \mathrm{d}y
\end{equation*}
and $|\nabla \varphi(y)| \leq\|D g(y)\||\nabla f(g(y))| \leq M_1$.
    
	The following lemma shows that the above derivation of $\nabla \ell$ is justified as well as several properties of the mapping $F_u$ and its inverse, that will be used in the following. It also includes a claim that was used in the proof before. Intuitively, the properties follow from the fact that $F_u$ is a small perturbation of the identity.
	\begin{lemma}\label{lem:helper_lemma}
	For any $u\in \mathbb{R}^{d}$ such that $|u|\leq b < r$, the mapping $F_u \colon \mathbb{R}^{d^{\prime}} \rightarrow \mathbb{R}^{d^{\prime}}, \ F_u(y)\coloneqq y +A(y)u$ is a diffeomorphism. Further
		\begin{align*}
			\mathrm{(I)} \quad  &\sup_{y \in \mathbb{R}^{d^{\prime}}} \|  A(y)\|  \leq   m^{-1}, \\
			\mathrm{(II)} \quad  &\sup_{y \in \mathbb{R}^{d^{\prime}}} \| D A(y)\| \leq c_A < \frac{1}{r}, \\
			\mathrm{(III)} \quad  & F_u \text{ is injective}, \\
			\mathrm{(IV)} \quad & S_u(z) = z - A(S_u(z))u,\\
			\mathrm{(V)} \quad & \sup_{y \in  \mathbb{R}^{d^{\prime}}} \|D S_u(y)\| \lesssim \frac{1}{1-c_A r},\\
			\mathrm{(VI)} \quad &\sup_{y \in  \mathbb{R}^{d^{\prime}}}  ||\operatorname{det}(D S_u(y))|-1| = O(|u|).
		\end{align*}
		
	\end{lemma}
First, we note that 
\begin{equation*}
     |\nabla \ell(0)||\int_{B(0, b)} u K_{\sigma_{\min}}(u) \; \mathrm{d} u| \leq \frac{M_1}{m} \sigma_{\min}^{1+\alpha}.
\end{equation*}
	Using \Cref{lem:helper_lemma}, the fact that $\varphi$ is Lipschitz and the Besov-smoothness of $p_{d^{\prime}}^*$, we obtain
	\begin{align*}
		|\nabla \ell(\xi_u)- \nabla \ell(0) |& = \Big|\int (\nabla\varphi(z))^{\top} \Big(A(S_{\xi_u}(z))p^*_{d^{\prime}}(S_{\xi_u}(z))|\operatorname{det}(DS_{\xi_u}(z))| -  A(z)p^*_{d^{\prime}}(z)\Big) \; \mathrm{d}z\Big|\\
		& \leq \int |\nabla\varphi(z)|  \| A(S_{\xi_u}(z)) - A(z)\|
		p^*_{d^{\prime}}(S_{\xi_u}(z))|\operatorname{det}(DS_{\xi_u}(z))|  \; \mathrm{d}z\\
         & \quad + \int |\nabla\varphi(z)| \|A(z)\|p^*_{d^{\prime}}(S_{\xi_u}(z)) |\operatorname{det}(D S_{\xi_u}) -1| \; \mathrm{d}z\\
		& \quad + \int |\nabla\varphi(z)| \|A(z)\| |p^*_{d^{\prime}}(S_{\xi_u}(z))- p^*_{d^{\prime}}(z)| \; \mathrm{d}z\\
		& \lesssim |u|+|u|^{ \alpha}.
	\end{align*}
	Plugging everything into \eqref{eq:ell_formulation}, we obtain
	\begin{equation*}
		\int_{B(0, r)} K_{\sigma_{\min}}(u)(\ell(0)-\ell(u))\; \mathrm{d}u \leq  \int_{B(0, r)} K_{\sigma_{\min}}(u) (|u|+|u|^{\alpha})|u|\; \mathrm{d}u + \sigma_{\min}^{1+ \alpha} \lesssim  \sigma_{\min}^{1+\alpha}.
	\end{equation*}
	Since $\sigma_{\min}\leq 1$, we have shown
	\begin{equation*}
		\mathsf{W}_1(\mathbb{P}^*, \pi_\#\mathbb{P}^{\psi_1(Z)}) \lesssim \sigma_{\min}^{1+\alpha}.
	\end{equation*}

	\underline{Second term:} First, we define the set
	\begin{equation*}
		\mathcal{F}\coloneqq \Big\{h \colon \mathbb{R}^{d} \rightarrow \mathbb{R} \mid h = f \circ \varphi, \;f \in \operatorname{Lip}(1), f(0)=0, \;\varphi \in \operatorname{Lip}\Big(\frac{\tau}{\tau -r}\Big), \; \varphi(0)=0 \Big\}.
	\end{equation*}
	For the second term in \eqref{kde_dimred_decomp}, we can use the Lipschitz bound from \Cref{lem:Lipschitz-constant} to bound
	\begin{align}
		\mathsf{W}_1(\pi_\#\mathbb{P}^{\psi_1(Z)}, \pi_\#\mathbb{P}^{\psi^n_1(Z)}) & = \sup_{\substack{f \in \operatorname{Lip}(1)\\ f(0)=0}}\int f(\pi(x)) \; \mathrm{d}(K_{\sigma_{\min}} \ast \mathbb{P}^* - K_{\sigma_{\min}}  \ast \mathbb{P}_n)(x)\notag \\ 
		& = \sup_{\substack{f \in \operatorname{Lip}(1)\\ f(0)=0}}\int_{B(\mathcal{M},r)} f(\pi(x)) \; \mathrm{d}(K_{\sigma_{\min}} \ast \mathbb{P}^* - K_{\sigma_{\min}}  \ast \mathbb{P}_n)(x) \notag \\
		& \leq \sup_{h \in \mathcal{F}}\int_{B(\mathcal{M},r)} h(x) \; \mathrm{d}(K_{\sigma_{\min}} \ast \mathbb{P}^* - K_{\sigma_{\min}}  \ast \mathbb{P}_n)(x)\notag \\
		& \leq \sup_{h \in \mathcal{F}}\int h(x) \; \mathrm{d}(K_{\sigma_{\min}} \ast \mathbb{P}^* - K_{\sigma_{\min}}  \ast \mathbb{P}_n)(x)\label{eq:inner_integral}\\
		& \quad + \sup_{h \in \mathcal{F}}\int_{B(\mathcal{M},r)^c} h(x) \; \mathrm{d}(K_{\sigma_{\min}} \ast \mathbb{P}^* - K_{\sigma_{\min}}  \ast \mathbb{P}_n)(x).\notag
	\end{align}
    For the first inequality, we used that we can extend $\pi$ outside of $B(\mathcal{M}, r)$ such that the extension $\tilde{\pi} \in \operatorname{Lip}\big(\frac{\tau}{\tau - r} \big)$ by Kirszbraun's theorem.
    Since we assumed that $g$ is bounded on the support of $\mathbb{P}^*$, we can again use similar arguments to \eqref{eq:A_0bound} 
	\begin{align*}
		\sup_{h \in \mathcal{F}}\int_{B(\mathcal{M},r)^c} h(x) \; \mathrm{d}(K_{\sigma_{\min}} \ast \mathbb{P}^* - K_{\sigma_{\min}}  \ast \mathbb{P}_n)(x) &\lesssim \int_{B(\mathcal{M},r)^c} |x| \; \mathrm{d}(K_{\sigma_{\min}} \ast \mathbb{P}^* + K_{\sigma_{\min}}  \ast \mathbb{P}_n)(x) \\
        & \lesssim \int_{|u|\geq r} (M + |u|) K_{\sigma_{\min}}(u) \; \mathrm{d}u \lesssim \sigma_{\min}^{1+ \alpha}.
	\end{align*}
	For \eqref{eq:inner_integral}, we obtain
	
	\begin{align*}
		\sup_{h \in \mathcal{F}}\int h(x) \; \mathrm{d}(K_{\sigma_{\min}} \ast \mathbb{P}^* - K_{\sigma_{\min}}  \ast \mathbb{P}_n)(x) &\lesssim \mathsf{W}_1(\mathbb{P}^{\psi_1(Z)}, \mathbb{P}^{\psi^n_1(Z)})\\
		&  = \sup_{f \in \operatorname{Lip}(1)} \int f(x) \; \mathrm{d}((K_{\sigma_{\min}}\ast\mathbb{P}^*)-(K_{\sigma_{\min}}\ast\mathbb{P}_n))(x)\\
		& = \sup_{f \in \operatorname{Lip}(1)} \int (K_{\sigma_{\min}}(-\cdot) \ast f)(x) \; \mathrm{d}(\mathbb{P}^*-\mathbb{P}_n)(x)\\
		& = \sup_{f \in \operatorname{Lip}(1)} \int (K_{\sigma_{\min}}(-\cdot) \ast f)(g(y)) \; \mathrm{d}(\mathbb{P}^*_{d^{\prime}}-\mathbb{P}_{n, d^{\prime}})(y)\\
		& = \sup_{\substack{h = (K_{\sigma_{\min}}(-\cdot)\ast f)(g)\\f \in \operatorname{Lip}(1) \\ f(0) = 0}} \int h(y) \; \mathrm{d}(\mathbb{P}^*_{d^{\prime}}-\mathbb{P}_{n, d^{\prime}})(y),
	\end{align*}
	where $\mathbb{P}_{n, d^{\prime}}$ is the empirical measure resulting from $Y_1,...,Y_n$ such that $g(Y_i)=X_i$ for every $i\in \{1,...,n\}$. 
	From \eqref{eq:bound_convolution}, we know that for any multi-index $k \in \mathbb{N}_0^d$ with $|k|\geq 1, |k| \leq\lceil \frac{d+2}{2}\rceil $,
	\begin{equation*}
		\|D^{k}	(K_{\sigma_{\min}} \ast f)\|_{\infty} \lesssim \sigma_{\min}^{-|k|+1}.
	\end{equation*}
	
	Next, we use Faà di Bruno's formula and the boundedness of all derivatives of $g$ to conclude
	\begin{equation*}
		\|D^{k}(K_{\sigma_{\min}} \ast f)(g)\|_{\infty} \lesssim \sigma_{\min}^{-|k|+1},
	\end{equation*} where the constant depends on $k, d$ and the properties of $g$.
	Now we can proceed as in \Cref{thm:kde_rate_smooth} and obtain
	\begin{equation*}
		\mathbb{E}\big[\mathsf{W}_1(\pi_\#\mathbb{P}^{\psi_1(Z)}, \pi_\#\mathbb{P}^{\psi^n_1(Z)})\big]\lesssim \frac{1}{\sqrt{\sigma_{\min}^{d^{\prime} -2}}}n^{-\frac{1}{2}}.
	\end{equation*}
\end{proof}

\begin{proof}[Proof of \Cref{thm:gaussian_kernel_optimal_rate_dimred}]
	 
	We proceed completely analogously to \Cref{thm:gaussian_kernel_optimal_rate} replacing $d$ with $d^{\prime}$. However, since the vector field we want to approximate is still $d$ dimensional, we still need $\lceil 2a\rceil^d$ little cubes. Hence we choose
	\begin{align*}
		L_n &\gtrsim   \Big(\max\Big(\log(n) +n^{\frac{3}{2\alpha+d^{\prime}}}, 1\Big) \Big)^{d/2} \cdot  \big(\log(n) + n^{\frac{3}{2\alpha +d^{\prime}}} \big) ,\\
		S_n &\gtrsim \Big(\max\Big(\log(n) +n^{\frac{3}{2\alpha+d^{\prime}}}, 1\Big) \Big)^{d/2}\cdot  \big( n^{\frac{11+\alpha}{2\alpha +d^{\prime}}} e^{(2d+1)n^{\frac{3}{2\alpha +d^{\prime}}}}\big)^{d+1} \cdot \big(\log^2(n) + n^{\frac{6}{2\alpha +d^{\prime}}} \big),
	\end{align*} to obtain the desired rate.     
\end{proof}

\subsection{Auxiliary results}\label{sec:helper}

    \begin{proof}[Proof of \Cref{thm:tails}]

    	First, recall the proof from \Cref{thm:error_decomp}.      
    	For any $a > 0$ we can bound the term from \eqref{Gronwall_bound}
    	\begin{align}
    		Q_1 &\leq  2 e^{1+2\int_{0}^1 \Gamma_t\; \mathrm{d}t} \int_{0}^{1} \int | v^n_t(\psi^n_t(x))-\tilde{v}_t(\psi^n_t(x))|^2 p_0(x) \; \mathrm{d}x\; \mathrm{d}t \notag \\
    		& =   2 e^{1+2\int_{0}^1 \Gamma_t\; \mathrm{d}t} \Big(
    		\int_{0}^{1} \int_{(-a,a)^d} | v^n_t(\psi^n_t(x))-\tilde{v}_t(\psi^n_t(x))|^2 p_0(x) \; \mathrm{d}x\; \mathrm{d}t\label{einsetzen_erster_term}\\ & \quad+ \int_{0}^{1} \int_{\mathbb{R}^d\setminus (-a,a)^d} | v^n_t(\psi^n_t(x))-\tilde{v}_t(\psi^n_t(x))|^2 p_0(x) \; \mathrm{d}x\; \mathrm{d}t\Big).\notag
    	\end{align}
    	The second term can be bounded by
    	\begin{align}
    		\int_{0}^{1} \int_{\mathbb{R}^d\setminus (-a,a)^d} | v^n_t(\psi^n_t(x))&-\tilde{v}_t(\psi^n_t(x))|^2 p_0(x) \; \mathrm{d}x\; \mathrm{d}t \notag \\&\leq  2 \int_{0}^{1} \int_{\mathbb{R}^d\setminus (-a,a)^d} \Big( | v^n_t(\psi^n_t(x))|^2+ |\tilde{v}_t(\psi^n_t(x))|^2\Big) p_0(x) \; \mathrm{d}x\; \mathrm{d}t.\label{eq:bound_outer} 
    	\end{align}
    To bound $\psi_t^n(x)$ for $x \in \mathbb{R}^d\setminus (-a,a)^d$, we use that 
    	\begin{equation} 
        \label{eq:psi_pushes_inwards}|\psi^n_t(x) | \leq |x | + \sqrt{d}.
        \end{equation} 

        To see why this is true, first note that under \Cref{ass:sigma+mu}, we can rewrite
        \begin{equation*}
            v_t(x|X_i^*)= \frac{X_i^* - (1- \sigma_{\min})x}{\sigma_t}
        \end{equation*}
        and thus
        \begin{equation}\label{eq:vundm}
            v_t^n(x) = \frac{m_t(x)- (1- \sigma_{\min})x}{\sigma_t}, \quad \text{where} \quad m_t(x) = \sum_{i = 1}^n X_i^* \frac{p_t(x|X_i^*)}{\sum_{j = 1}^n p_t(x|X_j^*)}.
        \end{equation}
       A rescaled version of the ODE \eqref{ODE} leads, after some cancellations, to
        \begin{equation*}
             \frac{d}{dt} \Big( \frac{\psi_t^n(x)}{\sigma_t}\Big) = \frac{m_t(\psi_t^n(x))}{\sigma_t^2}.
        \end{equation*}
        Integration yields
        \begin{equation*}
          \frac{\psi_t^n(x)}{\sigma_t} - \frac{\psi_0^n(x)}{\sigma_0} = \int_0^t \frac{m_s(\psi_s^n(x))}{\sigma_s^2} \, \mathrm{d}s. 
        \end{equation*}
        Taking norms, using that $|m_t(x)|\leq \sqrt{d}$ and integrating, we obtain
        \begin{equation*}
            |\psi_t^n|\leq \sigma_t |x| + \sqrt{d} \sigma_t \int_0^t \sigma_s^{-2} \, \mathrm{ds} = \sigma_t |x| + \sqrt{d} t \leq |x| + \sqrt{d}.
        \end{equation*}
      We cut the network at the maximum possible value of $v^n$ in $(-a,a)^d$, 
        \begin{equation}\label{eq:network_cut}
            \sup_{t \in [0,1], x \in \mathbb{R}^d}|\tilde{v}(x)| \leq \frac{\sqrt{d}(1+a)}{\sigma_{\min}}.
        \end{equation}

    	Therefore we get using \eqref{eq:psi_pushes_inwards}, \eqref{eq:vundm} and \eqref{eq:network_cut} for \eqref{eq:bound_outer} for $\sigma_{\min} \leq 1$, using crude bounds in prefactors
    	\begin{align*}
    		\int_{0}^{1} \int_{\mathbb{R}^d\setminus (-a,a)^d} &| v^n_t(\psi^n_t(x))-\tilde{v}_t(\psi^n_t(x))|^2 p_0(x) \; \mathrm{d}x\; \mathrm{d}t \leq \int_{0}^{1} \int_{\mathbb{R}^d\setminus (-a,a)^d} \frac{16 da^2 + 4 |x|^2}{\sigma_{\min}^2}p_0(x)\; \mathrm{d}x \; \mathrm{d}t\\
            & =\frac{16 da^2 }{\sigma_{\min}^2}  \int_{\mathbb{R}^d\setminus (-a,a)^d} p_0(x)\; \mathrm{d}x + \frac{4}{\sigma_{\min}^2}\int_{\mathbb{R}^d\setminus (-a,a)^d} |x|^2 p_0(x)\; \mathrm{d}x  .
    	\end{align*}
        Note that the integrand is not $t$ dependent, thus the outer integral contributes a factor of one. 
        Let $Z \sim \mathcal{N}(0, I_d)$, then for the first integral, 
        \begin{align*}
            \int_{\mathbb{R}^d\setminus (-a,a)^d} p_0(x)\; \mathrm{d}x &= \mathbb{P}\Big(\bigcup_{j=1}^d\left\{\left|Z_j\right| \geq a\right\}\Big) \leq d \mathbb{P}(|Z_1|>a) = 2d \int_{a}^{\infty} p_0(x)\; \mathrm{d}x\\& = \frac{2d}{\sqrt{2 \pi}} \int_{a}^{\infty} e^{- \frac{x^2}{2}}\; \mathrm{d}x \leq \frac{2d}{\sqrt{2 \pi}} \frac{1}{a} \int_a^{\infty} x e^{- \frac{x^2}{2}}\; \mathrm{d}x =\frac{2d}{\sqrt{2 \pi}} \frac{e^{- \frac{a^2}{2}}}{a}
        \end{align*}
        For the second integral, since the components of $Z$ are independent, 
        \begin{align*}
            \int_{\mathbb{R}^d\setminus (-a,a)^d} |x|^2 p_0(x)\; \mathrm{d}x &= \mathbb{E}\big[|Z|^2 \mathds{1}_{\{ \max_j |Z_j|>a\}} \big] \leq \sum_{j=1}^d \mathbb{E}\big[|Z|^2 \mathbf{1}_{\{|Z_j| \geq a\}}\big] \\&= d \big(\mathbb{E}\big[Z_1^2 \mathbf{1}_{\{|Z_1| \geq a\}}\big]+(d-1) \mathbb{P}(|Z_1| \geq a) \big).
        \end{align*}
        Since
        \begin{equation*}
            \mathbb{E}\big[Z_1^2 \mathbf{1}_{\{|Z_1| \geq a\}}\big] = \frac{2}{\sqrt{2\pi}}\int_a^{\infty} x^2 e^{- \frac{x^2}{2}}\; \mathrm{d}x = \frac{2}{\sqrt{2\pi}} a e^{- \frac{a^2}{2}} + \frac{2}{\sqrt{2\pi}}\int_a^{\infty} e^{-\frac{x^2}{2}}\; \mathrm{d}x \leq  \frac{2}{\sqrt{2\pi}} \Big( a + \frac{1}{a}\Big) e^{- \frac{a^2}{2}},
        \end{equation*}
        we obtain
        \begin{equation*}
            \int_{\mathbb{R}^d\setminus (-a,a)^d} |x|^2 p_0(x)\; \mathrm{d}x \leq \frac{2d}{\sqrt{2 \pi}} \Big( a + \frac{d}{a}\Big) e^{- \frac{a^2}{2}}
        \end{equation*}

    	Hence, since $a > 1$,
    	\begin{align*}
    		\int_{0}^{1} \int_{\mathbb{R}^d\setminus (-a,a)^d} | v^n_t(\psi^n_t(x))&-\tilde{v}_t(\psi^n_t(x))|^2 p_0(x) \; \mathrm{d}x\; \mathrm{d}t  \leq c \frac{a e^{-\frac{a^2}{2}}}{\sigma_{\min}^2}. \qedhere
    	\end{align*}
    \end{proof}
    	\begin{proof}[Proof of \Cref{lem:v_tv_tn_supnorm_bound}]
    	For the supremum norm bound, observe that 
    	\begin{align*}
    		|v^n_t(x)| & = \Big| \sum_{i =1}^{n} v_t(x |X_i^*)\frac{p_t(x|X_i^*)}{\sum_{j = 1}^{n}p_t(x|X_j^*)}\Big| \\
    		& \leq  \sum_{i =1}^{n} |v_t(x |X_i^*)|\frac{p_t(x|X_i^*)}{\sum_{j = 1}^{n}p_t(x|X_j^*)} \\
    		& \leq \max_{i \in \{1,\dots,n\}} |v_t(x |X_i^*)|\\
    		& \leq  \max_{i \in \{1,\dots,n\}}\frac{|X_i^*| + |x|}{1-(1-\sigma_{\min})t}\\
    		& \leq \frac{\sqrt{d}(1+a)}{1-(1-\sigma_{\min})t}.
    	\end{align*}
    	For the Lipschitz bound, note that
    	\begin{align}
    		\nabla_x v_t^n(x) &=  \nabla_x \sum_{i = 1}^n \Big(\frac{\frac{\partial \sigma_t}{\partial t}}{\sigma_t} x-\mu_t(X_i)+ \frac{\partial \mu_t}{\partial t}(X_i)\Big)\frac{p_t(x|X_i)}{\sum_{j = 1}^n p_t(x|X_j)}\label{lipschitz_v_n} \\
    		& =  \nabla_x \frac{-(1-\sigma_{\min})}{\sigma_t} x - \nabla_x \frac{-(1-\sigma_{\min}) t}{\sigma_t} \frac{ \sum_{i = 1}^n X_i p_t(x|X_i) }{\sum_{j = 1}^n p_t(x|X_j)}+ \nabla_x  \frac{\sum_{i = 1}^n X_i p_t(x|X_i)}{\sum_{j = 1}^n p_t(x|X_j)} . \notag
    	\end{align}
    	Now for the partial derivative of the $\ell$-th coordinate function with respect to $x_k,$ $\ell, k \in \{1,\dots,d \},$ we get
    	\begin{align*}
    		&\frac{\partial }{\partial x_k}\frac{\sum_{i = 1}^n X_{i, \ell} p_t(x|X_i)}{\sum_{j = 1}^n p_t(x|X_j)}  \\&= \frac{\big(\sum_{i = 1}^n \big(-\frac{x_{k}-tX_{ik}}{\sigma_t^2} \big) X_{i, \ell} p_t(x|X_i)\big)(\sum_{j = 1}^n p_t(x|X_j))- \big( \sum_{i = 1}^n X_{i, \ell} p_t(x|X_i)\big)\big(\sum_{j = 1}^n \big(-\frac{x_{k}-tX_{jk}}{\sigma_t^2} \big)  p_t(x|X_j)\big)}{\big(\sum_{j = 1}^n p_t(x|X_j) \big)^2} \\
    		& = \frac{t}{\sigma_t^2} \Bigg(\frac{\sum_{i = 1}^n X_{ik} X_{i \ell} p_t(x|X_i)}{\sum_{j =1}^n p_t(x|X_j)}- \frac{\big(\sum_{i = 1}^n X_{ik}  p_t(x|X_i)\big)\big(\sum_{i = 1}^n X_{i \ell} p_t(x|X_i)\big)}{\big(\sum_{j =1}^n p_t(x|X_j)\big)^2} \Bigg).
    	\end{align*}
    	Since $\operatorname{supp}(p^*) \subset [-1,1]^d,$ we can bound 
    	\begin{align*}
    		\frac{\partial }{\partial x_k}\frac{\sum_{i = 1}^n X_{i, \ell} p_t(x|X_i)}{\sum_{j = 1}^n p_t(x|X_j)}  \leq \frac{2t}{\sigma^2_t}. 
    	\end{align*}
    	Using $t \in [0,1],$ $\sigma_{\min} \leq 1,$ we get for \eqref{lipschitz_v_n}
    	\begin{align*}
    		\nabla_x v_t^n(x) \leq \frac{1}{\sigma_t} I_d + \frac{2}{\sigma_t^3} J_d,
    	\end{align*}
    	where $I_d$ denotes the $d \times d$ identity matrix, $J_d$ denotes the  $d \times d$ matrix consisting of ones and $\leq$ denotes entrywise inequality. Using the mean value theorem, we obtain for $x,y \in \mathbb{R}^d$
    	\begin{align*}
    		|v_t^n(x) - v_t^n(y)|  & \leq \Big\| \frac{1}{\sigma_t} I_d + \frac{2}{\sigma_t^3} J_d\Big\|_2 | x-y|.
    	\end{align*}
    	As
    	\begin{align*}
    		\Big\| \frac{1}{\sigma_t} I_d + \frac{2}{\sigma_t^3} J_d\Big\|_2  &\leq \Big\| \frac{1}{\sigma_t} I_d \Big\|_2+ \Big\| \frac{2}{\sigma_t^3} J_d\Big\|_2 \\
    		& = \frac{1}{\sigma_t} + \frac{2d}{\sigma_t^3},
    	\end{align*}
    	we get the desired bound.
    \end{proof}

    	\begin{proof}[Proof of \Cref{lem:helper_lemma}]$ $
    	\begin{enumerate}
    		\item[(I)]Differentiating $g^{-1}(g(y)) = y$ leads to $Dg^{-1}(g(y)) Dg(y)= I_{d^{\prime}}$, which implies that $Dg^{-1}(g(y)) \colon T_{g(y)}\mathcal{M} \rightarrow \mathbb{R}^{d^{\prime}}$ is the left inverse of $Dg(y) \colon  \mathbb{R}^{d^{\prime}} \rightarrow T_{g(y)}\mathcal{M}$. Using the construction of the pseudoinverse via the singular value decomposition, we conclude that
    		\begin{equation*}\|Dg^{-1}(g(y)) \| = \frac{1}{\lambda_{\min}(Dg(y))} \leq \frac{1}{m}.
    		\end{equation*}
    		Hence 
    		\begin{equation*}
    			\|A(y)\|\leq \|Dg^{-1}(g(y))\|\|P_{T_{g(y)}\mathcal{M}}\|\leq \frac{1}{m}.
    		\end{equation*}
    		
    		\item[(II)]
    		For any $h \in \mathbb{R}^{d^{\prime}}$, differentiating $Dg^{-1}(g(y))Dg(y)=I_{d^{\prime}}$ in direction $h$ yields
    		\begin{equation}\label{eq:ableitung_ableitung}
    			D(Dg^{-1}(g(y)))[h]Dg(y) + Dg^{-1}(g(y))D(Dg(y))[h] = 0,
    		\end{equation}
             where $D(D g(y))[h]  =\frac{d}{d t} D g(y+t h)\big|_{t=0}$ and $D(D g^{-1}(g(y)))[h]=\frac{d}{d t} D g^{-1}(g(y+t h))\big|_{t=0}$.
            
    		Fix $y \in \mathbb{R}^{d^{\prime}}$ and set $x = g(y)$. Differentiating $g(g^{-1}(x))=x$ yields
    		\begin{equation*}
    			Dg(g^{-1}(x))Dg^{-1}(x) = I_{T_{x}\mathcal{M}} \quad \Longleftrightarrow \quad Dg(y)Dg^{-1}(g(y)) = I_{T_{g(y)}\mathcal{M}}, 
    		\end{equation*}
    		where $I_{T_{x}\mathcal{M}}$ is the identity map on the tangent space. Using this right inverse of $Dg(y)$ and recalling that $P_{T_{g(y)}\mathcal{M}}$ is the orthogonal projection onto the tangent space, we obtain for \eqref{eq:ableitung_ableitung}, 
    		\begin{equation*}
    			D(Dg^{-1}(g(y)))[h]\Big|_{T_{g(y)}\mathcal{M}} = - Dg^{-1}(g(y))D(Dg(y))[h]Dg^{-1}(g(y)).
    		\end{equation*}
    		Taking the supremum over $h$ such that $|h| = 1$ leads to
    		\begin{equation*}
    			\|D(Dg^{-1}(g(y)))\| \leq \| Dg^{-1}(g(y))\|^2 \|D(Dg(y)) \| \leq \frac{M_2}{m^2}
    		\end{equation*}
    		
    		Combining \citet[Chapter II, Theorem 4.5, Corollary 4.6]{Stewart1990-lg} with \citet[Corollary 3]{Boissonnat2019-ya} and the fact that for $x \in [0, \frac{\pi}{2}]$ $\sin(x) \leq 3\sin\big(\frac{x}{2}\big)$ yields
    		\begin{equation*}
    			\|P_{T_{g(y+th)}\mathcal{M}}-P_{T_{g(y)}\mathcal{M}}\| \leq  \frac{3}{2 \tau}|g(y+th)- g(y)|.
    		\end{equation*}
    		Furthermore,
    		\begin{equation*}
    			\| D(P_{T_{g(y)}\mathcal{M}})[h]\| = \lim_{t \rightarrow 0} \frac{\| P_{T_{g(y+th)}\mathcal{M}} - P_{T_{g(y)}\mathcal{M}} \| }{|t|} \leq \lim_{t \rightarrow 0}\frac{3}{2\tau}\frac{|g(y+th)- g(y)|}{|t| } = \frac{3}{2\tau} |Dg(y)[h]|
    		\end{equation*}
    		
    		Next, we differentiate for $h \in \mathbb{R}^{d^{\prime}}$
    		\begin{equation*}
    			DA(y)[h] = D(Dg^{-1}(g(y)))[h]P_{T_{g(y)}\mathcal{M}} + Dg^{-1}(g(y)) D(P_{T_{g(y)}\mathcal{M}})[h].
    		\end{equation*}
    		Recall that the operator norm is defined as $\sup_{|h| = 1} \| DA(y)[h]\| = \sup_{|h| = 1}\sup_{|v|=1} | DA(y)[h]v| $. Then, inserting all constants from \Cref{ass:manifold}
    		\begin{equation*}
    			\| DA(y)\| \leq \|D(Dg^{-1}(g(y)))\| \|P_{T_{g(y)}\mathcal{M}} \|+ \|Dg^{-1}(g(y)) \| \|D(P_{T_{g(y)}\mathcal{M}})\| \leq \frac{M_2}{m^2} + \frac{3 M_1}{2m\tau }=c_A < \frac{1}{b}.
    		\end{equation*}
    		
    		\item[(III)] Assume that $F_u(y_1)=F_u(y_2)$. Then \begin{equation*}
    			0 = F_u(y_1)-F_u(y_2) = y_1 -y_2 + (A(y_1)-A(y_2))u.
    		\end{equation*}
    		If $y_1 \neq y_2$
    		\begin{equation*}
    			| y_1 -y_2| = | (A(y_1)-A(y_2))u|\leq | A(y_1)-A(y_2)| |u| \leq c_A r |y_1-y_2 | < |y_1-y_2 | .
    		\end{equation*}
    		The above inequality is false, which implies $y_1 = y_2$.
    		
    		\item[(IV)] From the proof of part (III), we can also conclude that $F_u$ is surjective. Rearranging yields
    		\begin{equation*}
    			F_u(S_u(z)) = z \quad \Longleftrightarrow \quad S_u(z)+A(S_u(z))u = z \quad \Longleftrightarrow \quad S_u(z) = z - A(S_u(z))u.
    		\end{equation*}
    		\item[(V)] In order to use the inverse function theorem, we first need to ensure that $D F_u$ has a nonzero determinant. 
    		For $h \in \mathbb{R}^{d^{\prime}}$, define
    		\begin{equation*}
    			D F_u(y)[h]=h+(D A(y)[h])u.
    		\end{equation*}
    		Set the linear operator $G(y, u)\colon \mathbb{R}^{d^{\prime}} \rightarrow \mathbb{R}^{d^{\prime}}$ with $ G(y, u)[h]=(D A(y)[h]) u$, 
    		then
    		\begin{equation*}
    			D F_u(y)=I_{d^{\prime}}+G(y, u), \quad \text{and} \quad \| G(y, u)\| \leq \| DA(y)\| |u|\leq c_A r < 1.
    		\end{equation*}

    		For the $i$-th largest singular value $s_i$, using Weyl's singular value perturbation bound, for an English reference see e.g.\ \cite{stewart1990},
    		\begin{equation*}
    			|s_i(D F_u(y)) - s_i (I_{d^{\prime}})|\leq \| G(y, u)\|\leq \|DA(y)\| |u| \leq c_A r < 1.
    		\end{equation*}
    		Since the absolute value of the determinant is the product of the singular values, we obtain
    		\begin{equation*}
    			(1-c_A r)^{d^{\prime}}\leq |\operatorname{det}(DF_u(y))|\leq (1+c_Ar)^{d^{\prime}}.
    		\end{equation*}
    		Moreover, $t \mapsto \operatorname{det}\left(I_{d^{\prime}}+t G(y, u)\right)$ is continuous, unequal to $0$ and equals 1 at $t=0$. Hence, $\operatorname{det}\left(D F_u(y)\right)>0$.
    		Since $\|DA(y)\||u| \leq c_Ar <1$, a Neumann series yields
    		\begin{equation*}
    			\|(DF_u(y))^{-1}\| = \|(I_{d^{\prime }} -  (-G(y,u)))^{-1}\|
    			\leq (1 - \|G(y,u)\|)^{-1} \leq \frac{1}{1-c_Ar}
    		\end{equation*}
    		Since $F_u$ is injective, we have by the inverse function theorem that
    		\begin{equation*}
    			DS_u(y)= (DF_u(S_u(y)))^{-1}, \quad \text{ thus} \quad \| DS_u(y) \|\leq \frac{1}{1-c_Ar}.
    		\end{equation*}
    		\item[(VI)]  
    		First, we differentiate $S_u(z) = z- A(S_u(z))u$ to obtain for $h \in \mathbb{R}^{d^{\prime}}$
    		\begin{equation}\label{ableitung_umkehrfunktion}
    			DS_u(z)[h] = h - DA(S_u(z))[DS_u(z)[h]]u.
    		\end{equation}
    		Defining the operator $B(z,u) \coloneqq (DA(S_u(z)))[\cdot]u$ for $v \in \mathbb{R}^{d^{\prime}}$ and rearranging yields on a matrix level
    		\begin{equation*}
    			(I_{d^{\prime}} + B(z,u))DS_u(z) = I_{d^{\prime}} \quad \Longleftrightarrow \quad DS_u(z) = (I_{d^{\prime}} + B(z,u))^{-1},
    		\end{equation*}
    		where the invertibility follows from
    		\begin{equation*}
    			\| B(z,u)\| = \sup_{|v|=1}\|(DA(S_u(z))[v])u\|\leq c_A |u|<1.
    		\end{equation*}
    		The inequality follows from property (II). Thus
    		\begin{equation*}
    			\operatorname{det}(DS_u(z)) = (\operatorname{det}(I+B(z,u)))^{-1}.
    		\end{equation*}
    		By the same singular value perturbation bound that was used in (V), we obtain
    		\begin{equation*}
    			|s_i(I+B(z,u) ) - s_i(I)|\leq \|B(z,u) \|\leq 
    			c_A |u|. \end{equation*}
    		This implies
    		\begin{equation*}
    			(1-c_A |u|)^{d^{\prime}} \leq |\operatorname{det}(I + B(z,u))|\leq (1+c_A |u|)^{d^{\prime}}, 
    		\end{equation*}
    		which yields
    		\begin{equation*}
    			(1+c_A |u|)^{-d^{\prime}} -1\leq |\operatorname{det}(DS_u(z))|-1 \leq (1-c_A |u|)^{-d^{\prime}}-1.
    		\end{equation*}
    		We obtain
    		\begin{equation*}
    			||\operatorname{det}(DS_u(z))|-1 |\leq \max\big( 1 - (1+c_A |u|)^{-d^{\prime}}, (1-c_A|u|)^{-d^{\prime}}-1\big).
    		\end{equation*}
            Two Taylor expansions yield for $\xi \in [0, c_A |u|] \subset [0, c_A r]$
            \begin{align*}
                |(1+ \xi)^{-d^{\prime}}-1| &\leq \sup_{\zeta \in [0, c_Ar]}d^{\prime}(1+ \zeta)^{-d^{\prime}-1}\xi\leq d^{\prime} \xi,\\
                         |(1- \xi)^{-d^{\prime}}-1| &\leq \sup_{\zeta \in [0, c_Ar]}d^{\prime}(1- \zeta)^{-d^{\prime}-1}\xi\leq d^{\prime}(1-c_Ar)^{-d^{\prime} -1}\xi,
            \end{align*}
            which implies
            \begin{equation*}
                ||\operatorname{det}(DS_u(z))|-1 |\leq d^{\prime}c_A(1-c_Ar)^{-d^{\prime} -1}|u|.
            \end{equation*}
    	\end{enumerate}
        Since $\sup _y\|D A(y)\||u|<1$, $F_u$ is surjective. Additionally, we showed that $S_u$ exists and from \eqref{ableitung_umkehrfunktion} and the properties of $g$ we can conclude that the derivative of $S_u$ is continuous. Thus we obtain that $F_u$ is a diffeomorphism. 
    \end{proof}
	
	\appendix
	\section{Appendix}\label{appendix}
   \begin{theorem}\label{thm:kde_rate_nonsmooth}
	Assume $p^* \in B_{1, \infty}^{\alpha}(\Lambda, \mathbb{R}^{d}), \alpha \in (0,1], \Lambda >0$. Further assume $K$ is a nonnegative, $d$-dimensional kernel such that $\int  K_{\sigma_{\min}}(y)y \; \mathrm{d}y = 0 $ and
	\begin{align*}
		 \int  |y|^{1+\alpha }   K(y)  \;\mathrm{d}y < \infty, \quad  \int    \big(1+|u|^{d+4} \big)   K^2(u)   \;\mathrm{d}u < \infty &\quad \text{and} \quad \int  |x|^{d+4} p^*(x) \;  \mathrm{d}x < \infty.
	\end{align*}
	If $\sigma_{\min} \leq 1$, then there are $C_1$ and $C_2$ such that 
	\begin{equation*}
		\mathbb{E}\big[	\mathsf{W}_1(\mathbb{P}^*, \mathbb{P}^{\psi^n_1(Z)})  \big] \leq C_1 \sigma_{\min}^{1+\alpha} + \frac{C_2}{\sqrt{n\sigma_{\min}^d}}.
	\end{equation*}
	
\end{theorem}

\begin{proof}   
 As in \Cref{thm:kde_rate_smooth},
	\begin{align}
		\mathsf{W}_1(\mathbb{P}^*, \mathbb{P}^{\psi^n_1(Z)}) &= \sup_{f \in \mathrm{Lip}(1)}  \int f(z)(p^*(z) - (K_{\sigma_{\min}} \ast \mu_n)(z)) \; \mathrm{d}z \notag \\
		& \leq \sup_{f \in \mathrm{Lip}(1)}  \int f(z)(p^*(z) - (K_{\sigma_{\min}} \ast p^*)(z)) \; \mathrm{d}z\notag \\ & \quad +\sup_{f \in \mathrm{Lip}(1)}  \int f(z)((K_{\sigma_{\min}} \ast p^*)(z)- (K_{\sigma_{\min}} \ast \mu_n)(z)) \; \mathrm{d}z. \label{dreiecksugl_kernschaetzer_wasserstein}
	\end{align}
	For the first term, we obtain as in \Cref{thm:kde_rate_smooth}
	\begin{align*}
		\sup_{f \in \mathrm{Lip}(1)}  \int f(z)(p^*(z) - (K_{\sigma_{\min}} \ast p^*)(z)) \; \mathrm{d}z \leq   C_1\sigma_{\min}^{1+\alpha} .
	\end{align*}
	\underline{Second term:}	
	To bound the expectation of  \eqref{dreiecksugl_kernschaetzer_wasserstein} we first use that by \citet[Theorem 6.15]{Villani2008}
	\begin{equation*}
		\sup_{f \in \mathrm{Lip}(1)}  \int f(z)((K_{\sigma_{\min}} \ast p^*)(z)- (K_{\sigma_{\min}} \ast \mu_n)(z)) \; \mathrm{d}z \leq    \int |z||(K_{\sigma_{\min}} \ast p^*)(z)- (K_{\sigma_{\min}} \ast \mu_n)(z)| \; \mathrm{d}z.
	\end{equation*}
	Then for $\rho > \frac{d}{2}$ \begin{align}
		\mathbb{E}\Big[ \int &|z||(K_{\sigma_{\min}} \ast p^*)(z) - (K_{\sigma_{\min}} \ast \mu_n)(z)| \; \mathrm{d}z\Big] \notag \\ & = \int (1+ |z|)^{-\rho}|z| (1+ |z|)^{\rho}  \mathbb{E}\Big[ \Big|\frac{1}{n}\sum_{i =1}^{n}K_{\sigma_{\min}}(z-X_i) - \mathbb{E}[K_{\sigma_{\min}}(z-X_i)] \Big| \Big]\; \mathrm{d}z \notag \\
		& \leq \frac{1}{n} \sqrt{\int (1+|z|)^{-2\rho}\; \mathrm{d}z} \sqrt{\int |z|^2 (1+|z|)^{2\rho} \mathbb{E}\Big[ \Big(\sum_{i =1}^{n}K_{\sigma_{\min}}(z-X_i) - \mathbb{E}[K_{\sigma_{\min}}(z-X_i)] \Big)^2\Big] \; \mathrm{d}z},\label{eq:TV_jensen}
				\end{align}
	where the equality holds by Fubini's theorem and the inequality holds by the Jensen inequality together with the Cauchy-Schwarz inequality.\\
	Next we bound the expectation. Since $\mathbb{E}[K_{\sigma_{\min}}(z-X_i) - \mathbb{E}[K_{\sigma_{\min}}(z-X_i)]] = 0$ and $X_1,\dots,X_n$ are i.i.d. we have that
	\begin{align} 
		\mathbb{E}\Big[ \Big(\sum_{i =1}^{n}K_{\sigma_{\min}}(z-X_i) - \mathbb{E}[K_{\sigma_{\min}}(z-X_i)] \Big)^2\Big] & = n \operatorname{Var}\big(  K_{\sigma_{\min}}(z-X_1) \big)\notag \\
		& \leq n \mathbb{E}[  (K_{\sigma_{\min}}(z-X_1) )^2]. \notag
	\end{align}

Therefore
\begin{align*}
    \int |z|^2 (1+|z|)^{2\rho} &\mathbb{E}\Big[ \Big(\sum_{i =1}^{n}K_{\sigma_{\min}}(z-X_i) - \mathbb{E}[K_{\sigma_{\min}}(z-X_i)] \Big)^2\Big] \; \mathrm{d}z\\
    & \leq  \frac{n}{\sigma_{\min}^d} \int \mathbb{E}\Big[ |X_1 + \sigma_{\min}y|^2 \big(1+ |X_1 + \sigma_{\min}y|\big)^{2 \rho} K^2(y)\Big] \; \mathrm{d}y\\
    & \lesssim \frac{n}{\sigma_{\min}^d} \int \mathbb{E}\Big[ \big(1+ |X_1 + \sigma_{\min}y|\big)^{2 \rho+2} K^2(y)\Big] \; \mathrm{d}y\\
    & \lesssim  \frac{n}{\sigma_{\min}^d} \mathbb{E}[(1+|X_1|)^{2\rho +2}]  \int(1 + \sigma_{\min}|y|)^{2 \rho + 2}K^2(y)\; \mathrm{d}y.
\end{align*}
    Using that $\sigma_{\min} \leq 1$ we can bound \eqref{eq:TV_jensen} for $\rho = \frac{d}{2} + 1$ further by
\begin{align*}
    \mathbb{E}\Big[ \int |z||(K_{\sigma_{\min}} \ast p^*)(z) - (K_{\sigma_{\min}} \ast \mu_n)(z)| \; \mathrm{d}z\Big] & \lesssim \frac{1}{\sqrt{n \sigma_{\min}^d}} \int (1+ |y|^{4+d}) K^2(y) \; \mathrm{d}y.
\end{align*}
By assumption, the last integral is finite. 
 In the end, we get for 
	\eqref{dreiecksugl_kernschaetzer_wasserstein} that
	\begin{equation*}
		\mathbb{E}\big[	\mathsf{W}_1(\mathbb{P}^*, \mathbb{P}^{\psi^n_1(Z)})  \big]\leq C_1 \sigma_{\min}^{1+\alpha} + \frac{C_2\max(1,{\sigma_{\min}},\sigma_{\min}^{\rho+1})}{\sqrt{n\sigma_{\min}^d}}.\qedhere
	\end{equation*}
	
\end{proof}
For completeness, we translate the following result from German including the proof.
	\begin{lemma}\cite[Satz 4.4.3]{Gabcke1979} \label{upper_incomplete_gamma_function}
		Let $a \geq 1$ and $x>a.$ The upper incomplete gamma function 
		\begin{equation}\label{eq:incompl_gamma}
		\Gamma(a, x):=\int_x^{\infty} e^{-v} v^{a-1} \; \mathrm{d} v 
		\end{equation}
		can be bounded by
		\begin{equation*}
		\Gamma(a, x) \leq a e^{-x} x^{a-1}.
		\end{equation*}
	\end{lemma} 
	\begin{proof}
	Let $a \geq 0$. Substituting $v=1 / w$ in \eqref{eq:incompl_gamma} leads to
		\begin{equation*}\label{eq:incompl_gamma_subs}
		\Gamma(a, x)=\int_0^{1 / x} e^{-1 / w} w^{-a-1} d w .
	\end{equation*}	
		Set
		\begin{equation*}
		h(w):=e^{-1 / w} w^{-a-1},
		\end{equation*}
		the first derivative of $h$ is
		\begin{equation*}
		h^{\prime}(w)=e^{-1 / w} w^{-a-2}\left[\frac{1}{w}-(a+1)\right].
		\end{equation*}
	If $x>a+1$, then $h^{\prime}(w)>0$ for all $w\in [0, \frac{1}{x}]$. Hence $h$ is strictly increasing and 
	\begin{equation}\label{eq:incompl_gamma_bound_1}
		\Gamma(a, x) \leq \frac{1}{x} h\left(\frac{1}{x}\right)=e^{-x} x^a \quad(a \geq 0, x>a+1) .
			\end{equation}		
		For $a \geq 1$ integration by parts in \eqref{eq:incompl_gamma} leads to
		\begin{equation*}
		\Gamma(a, x)=e^{-x} x^{a-1}+(a-1) \Gamma(a-1, x) .
		\end{equation*}
		Using \eqref{eq:incompl_gamma_bound_1} with $a-1$ instead of $a,$ we get
		\begin{equation*}
		\Gamma(a, x) \leq e^{-x} x^{a-1}+(a-1) e^{-x} x^{a-1}=a e^{-x} x^{a-1} \quad(a \geq 1, x>a). \qedhere
		\end{equation*}
	\end{proof}
\newpage
\textbf{\large Numerical illustration with smaller networks}

\begin{figure}[h]
    \centering
    
  \begin{subfigure}[b]{0.49\textwidth}
         \centering
    \includegraphics[width=\linewidth]{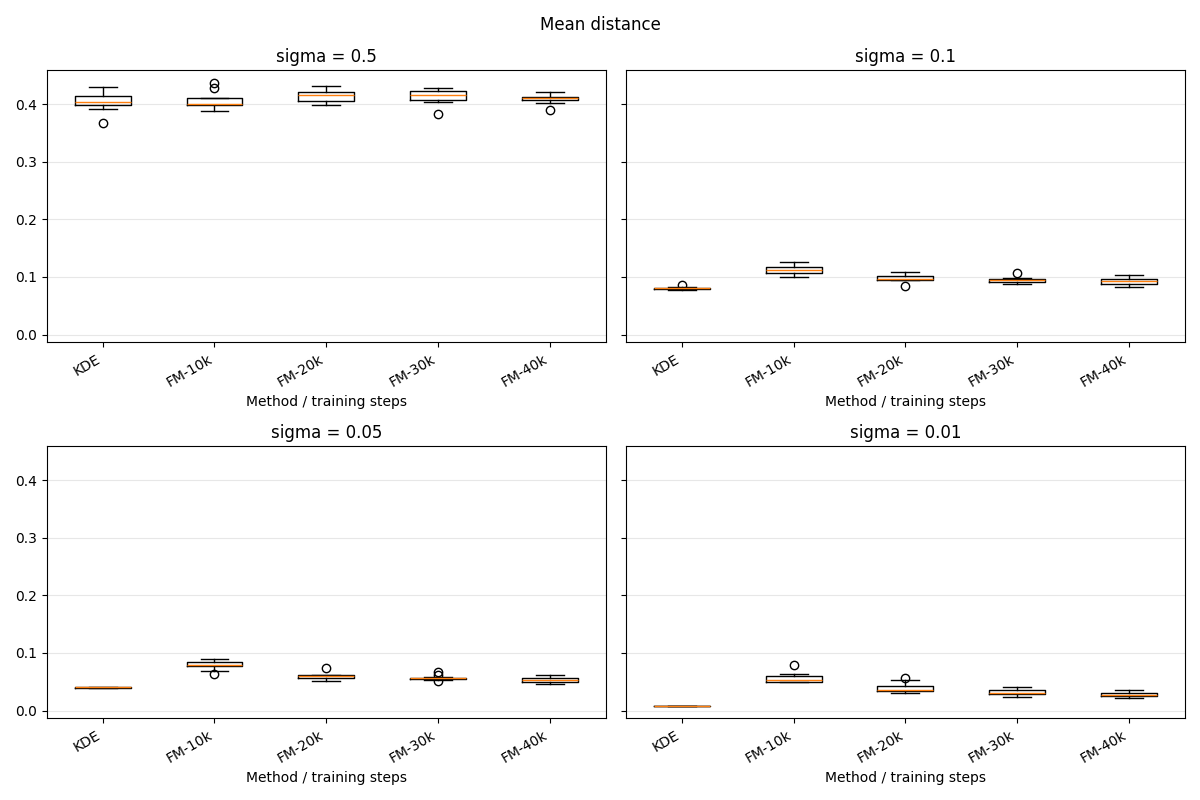}
    \caption{Mean distance of the generated samples to the manifold.}
    \end{subfigure}
    \hfill
    \begin{subfigure}[b]{0.49\textwidth}
        \centering
    \includegraphics[width=\linewidth]{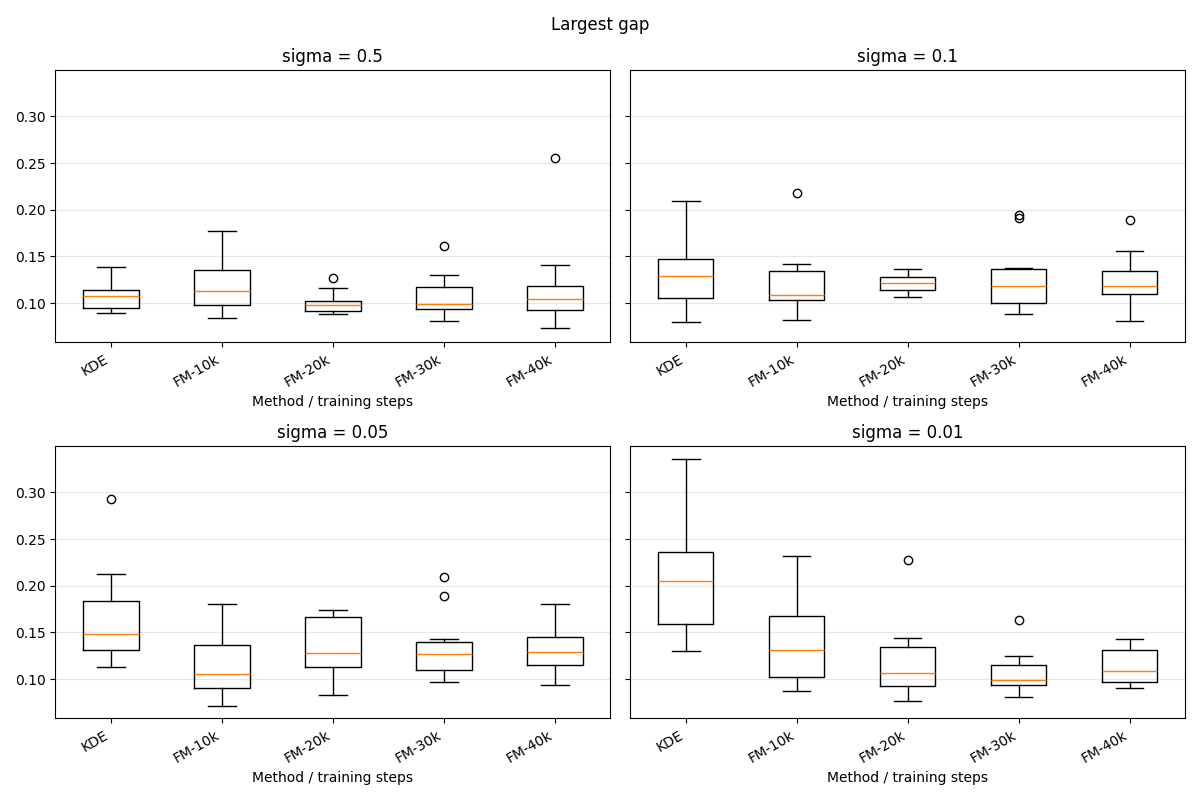}
    \caption{Largest gap on the manifold between the projections of the generated samples.}
    \end{subfigure}
    
    \vspace{0.5cm}

    \begin{subfigure}[b]{0.49\textwidth}
         \centering
    \includegraphics[width=\linewidth]{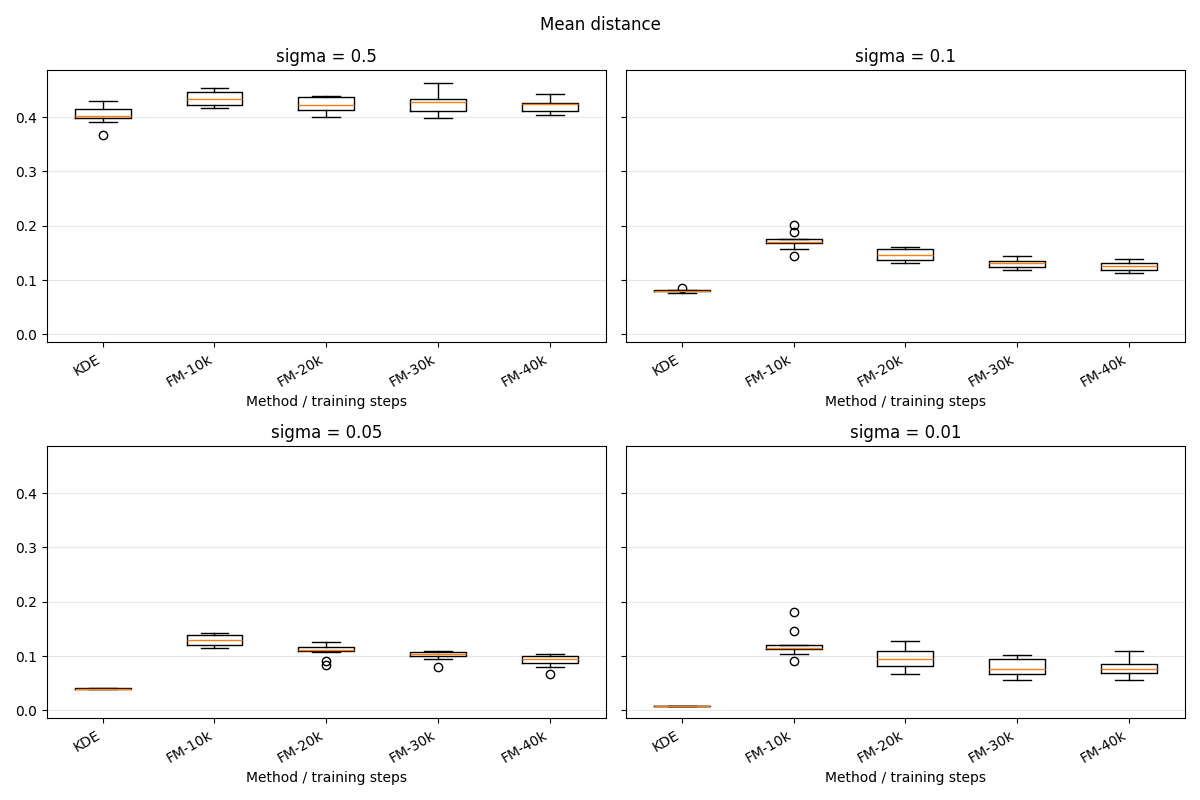}
    \caption{Mean distance of the generated samples to the manifold.}
    \end{subfigure}
    \hfill
    \begin{subfigure}[b]{0.49\textwidth}
        \centering
    \includegraphics[width=\linewidth]{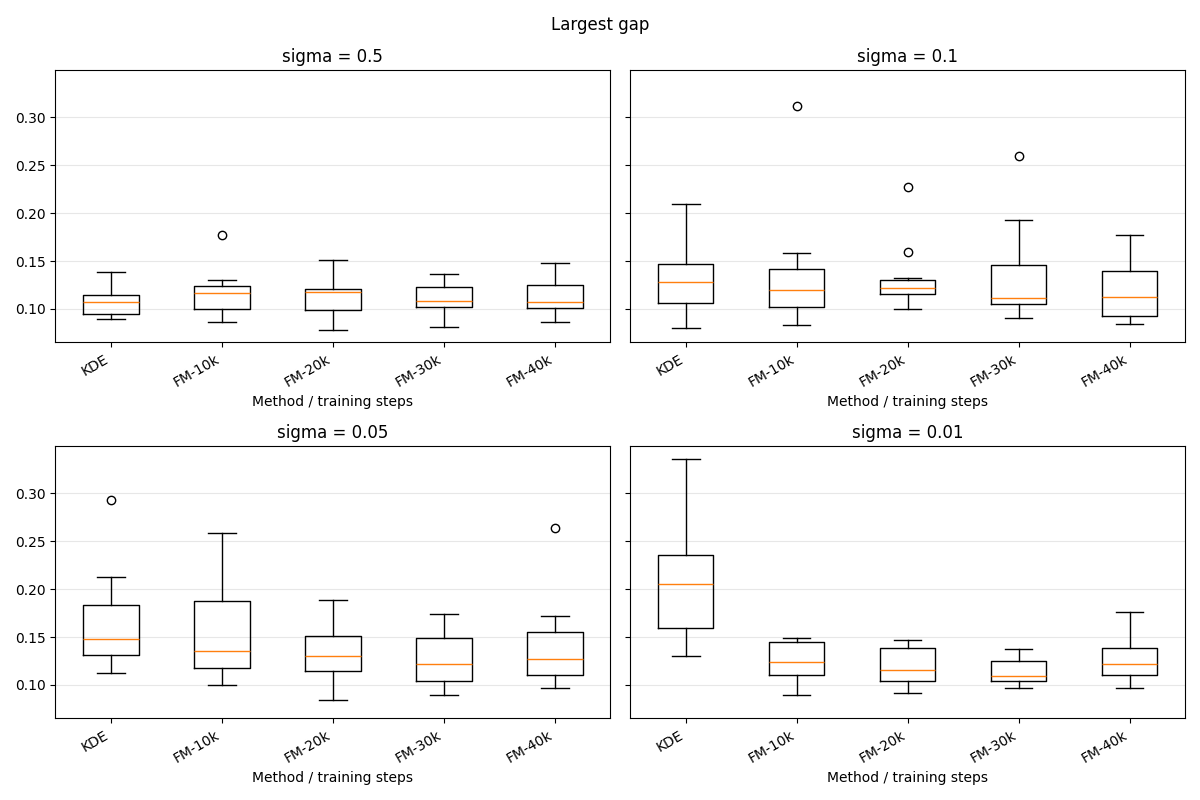}
    \caption{Largest gap on the manifold between the projections of the generated samples. }
    \end{subfigure}
    
    \caption{Comparison between the Flow Matching estimator and the KDE in terms of distance to the manifold and the largest gap of the projection estimator on the manifold. Top row uses a SeLU network with three-hidden-layers of width $32$ and bottom row uses a three-hidden-layer SeLU net of width $8$.}
    \label{fig:four_images}
\end{figure}

	\newpage
	\bibliography{literatur}
\end{document}